\documentclass{article} % For LaTeX2e
\usepackage{iclr2027_conference,times}
\usepackage{amsmath,amsfonts,bm}

\def\eqref#1{equation~\ref{#1}}
\def\1{\bm{1}}

\DeclareMathAlphabet{\mathsfit}{\encodingdefault}{\sfdefault}{m}{sl}
\SetMathAlphabet{\mathsfit}{bold}{\encodingdefault}{\sfdefault}{bx}{n}

\usepackage{amssymb,amsthm,array,booktabs}
\usepackage{graphicx}
\usepackage{hyperref}
\usepackage{url}
\usepackage{xcolor}
\definecolor{PFMBlue}{HTML}{4666E5}
\definecolor{PFMOrange}{HTML}{D85A00}
\newcommand{\pfmVelocity}[1]{{\color{PFMBlue}#1}}
\newcommand{\pfmDiffusion}[1]{{\color{PFMOrange}#1}}

\usepackage[normalem]{ulem}
\usepackage{titletoc}

\newif\ifpfmMethodApplications

\AtBeginDocument{%
  \setlength{\abovedisplayskip}{4pt plus 1pt minus 2pt}%
  \setlength{\belowdisplayskip}{4pt plus 1pt minus 2pt}%
  \setlength{\abovedisplayshortskip}{0pt plus 1pt}%
  \setlength{\belowdisplayshortskip}{3pt plus 1pt minus 1pt}%
}

\newtheorem{proposition}{Proposition}[section]
\theoremstyle{definition}
\newtheorem{pfmalgorithm}{Algorithm}
\theoremstyle{plain}
\title{Posterior Flow Matching}
\author{%
\textbf{Chenrui Ma$^1$ \quad Xi Xiao$^2$ \quad Lin Zhao$^3$ \quad Xinru Jiang$^3$}\\
\textbf{Tianyang Wang$^2$ \quad Yanning Shen$^4$ \quad Nando Fioretto$^1$}\\
{\small\normalfont $^1$\texttt{\{zzz8fa,fioretto\}@virginia.edu} \quad
 $^2$\texttt{\{xxiao,tw2\}@uab.edu}}\\
{\small\normalfont $^3$\texttt{\{zhao.lin1,jiang.xinru\}@northeastern.edu} \quad
 $^4$\texttt{yannings@uci.edu}}%
}
\iclrfinalcopy

\begin{document}
\maketitle
\ificlrfinal
  \lhead{Preprint} % The enabled switch selects preprint, not an accepted-paper header.
  \vspace{-12pt} % Share the space saving across the abstract's surrounding gaps.
\fi

\begin{abstract}
\ificlrfinal\vspace{-6pt}\fi
Flow matching learns how to transport noise toward data, but its mean prediction does not directly report the endpoint uncertainty remaining at an intermediate state. This paper introduces \emph{Posterior Flow Matching} (PFM), a general generative modeling framework that learns conditional velocity means and componentwise variances with a Gaussian posterior model. The same model evaluation supplies both a transport prediction and endpoint uncertainty, making posterior diagnostics available throughout generation. These learned moments support a mean-driven ordinary differential equation (ODE) and a posterior-driven stochastic differential equation (SDE), whose local noise scales adapt to the predicted uncertainty. Across image, point-cloud, molecular, and action generation, PFM improves generation quality and provides endpoint-posterior readouts throughout generation.
\end{abstract}

\ificlrfinal\vspace{-6pt}\fi
\vspace{-6pt}
\section{Introduction}
\label{sec:introduction}

Diffusion models~\citep{ho2020denoising,song2021score,cai2025diffusion,ma2024sit,rombach2022highresolutionimagesynthesislatent,hoogeboom2023simple,yu2024representation,zheng2025diffusion,wang2024rectified} and flow matching (FM)~\citep{lipman2023flow,albergo2023building,esser2024scaling,liu2023flow,jin2025pyramidal,klein2023equivariant,ma2026transition,ma2026drift} generate data by progressively transforming noise into structured outputs~\citep{karras2022elucidating}. Diffusion learns a score for reversing a noising process~\citep{lai2025principlesdiffusionmodels}, while FM learns a velocity field along a reference path~\citep{lipman2024flowmatchingguidecode,lee2024improving,kornilov2024optimal}. \textit{These predictions guide generation but do not directly report the endpoint uncertainty remaining at an intermediate state.}

This uncertainty arises because several source-data pairs can be compatible with the current intermediate state, each prescribing a different velocity toward its data endpoint~\citep{guo2025variational,zhang2025towards,zhang2025hierarchical,ma2025learningstraightflowsvariational}. Standard FM regression estimates only the conditional mean of these velocity targets~\citep{lipman2023flow}, which supplies the transport direction but does not report how widely the compatible endpoints are dispersed. Learning this dispersion alongside the mean would reveal which parts of a developing output remain uncertain and provide local scales for stochastic perturbations, connecting uncertainty estimation directly to generation.

We realize this idea with \emph{Posterior Flow Matching} (PFM), which extends FM regression to learn the conditional velocity mean and componentwise variance with a Gaussian posterior model. The same model evaluation supplies both the transport prediction and endpoint-posterior readouts. The velocity mean defines a mean-driven ordinary differential equation (PFM ODE), while the velocity covariance sets the local noise scales of a posterior-driven stochastic differential equation (PFM SDE). Under either sampler, the endpoint-posterior readouts describe the evolving output and its uncertainty without a separate diagnostic model.
We evaluate PFM on image generation, 3D point-cloud generation, molecular conformation generation, and action generation for robotics control. Across these tasks, PFM SDE improves generation quality over PFM ODE and other baselines. Covariance ablations further support using the learned posterior variance to guide sampling. Both samplers use the same trained posterior model and provide endpoint uncertainty estimates throughout generation.

The paper provides three key contributions. {\em First}, we extend flow matching regression to posterior learning by jointly fitting the conditional velocity mean and componentwise variance with one Gaussian negative log-likelihood (NLL) objective. An affine relation converts these moments into endpoint-posterior readouts at essentially no additional cost.
{\em Second}, we derive the mean-driven ODE and posterior-driven SDE samplers from these moments. 
Both provably preserve the reference marginals under exact dynamics, and the SDE dissipates the relative entropy of density perturbations, weighting each coordinate by its predicted velocity variance.
% Both provably preserve the reference marginals under exact dynamics, and the SDE dissipates the relative entropy of density perturbations with each coordinate's contribution to dissipation weighted by its predicted velocity variance.
{\em Third}, we show that this single formulation transfers across image, point-cloud, molecular, and action generation, where it improves generation quality while its posterior diagnostics quantitatively describe endpoint uncertainty throughout generation.
Through this work, we call for treating uncertainty as an integral part of generation rather than a post-hoc diagnostic, laying a principled foundation for reliable and interpretable generative models.

\vspace{-4pt}
\section{Related Work}
\label{sec:related-work}

Our work connects flow matching and diffusion models with posterior uncertainty estimation and variance-aware generation.\\
\textbf{Flow matching and diffusion models.}
Diffusion probabilistic models learn to reverse gradual data corruption through denoising objectives \citep{ho2020denoising}. Score-based modeling connects reverse SDEs to probability-flow ODEs, enabling stochastic and deterministic sampling within one model \citep{song2021score}. FM learns velocity fields by regressing conditional targets along reference paths \citep{lipman2023flow,tong2023improving,pooladian2023multisample,albergo2024stochastic,ma2025stochastic}. Stochastic interpolants place transport and diffusion in a common framework, where a density path admits dynamics with adjustable diffusion \citep{albergo2025stochastic,ma2024sit}. PFM builds on this correspondence: its learned posterior mean supplies transport, and its covariance selects local diffusion scales within the compatible family.\\ 
\textbf{Uncertainty quantification in generative models.}
BayesDiff propagates uncertainty from a last-layer Laplace approximation through diffusion \citep{kou2024bayesdiff}, while Generative Uncertainty combines Bayesian model uncertainty with a semantic likelihood for sample assessment \citep{jazbec2025generative}. Posterior moments can also be obtained from denoiser derivatives \citep{manor2024posterior} or velocity derivatives in flow models \citep{lv2026singlepass}. UA-Flow learns heteroscedastic velocity uncertainty with supervision corrected toward a marginal-velocity objective and propagates it for assessment and guidance \citep{han2026uncertainty}; Flowing with Confidence propagates input-dependent network noise along trajectories \citep{dekruiff2026confidence}. PFM instead directly fits the conditional moments of sampled FM velocity targets. The affine endpoint relation then gives endpoint-posterior readouts at the current state, without propagating uncertainty along the trajectory or computing model derivatives to read out the variance.\\
\textbf{Posterior modeling and variance-aware generation.}
Improved DDPM learns reverse-process variances \citep{nichol2021improved}, and Analytic-DPM estimates optimal reverse variances from a pretrained score model \citep{bao2022analytic}. These describe reverse transitions, whereas PFM models conditional velocity and endpoint dispersion. Pixel-wise aleatoric uncertainty has been used to guide diffusion sampling \citep{devita2024aleatoric}. Posterior Augmented Flow Matching reduces training-target variance through posterior-weighted candidate averages \citep{stoica2026posterior}. PFM preserves the sampled targets and fits their conditional moments with one Gaussian NLL. The same covariance supplies endpoint diagnostics and the SDE diffusion tensor, paired with the score and divergence corrections required for compatible transport.

\vspace{-4pt}
\section{Preliminaries}
\label{sec:preliminaries}

We introduce the reference targets used for learning, the score-based view of generation, and the correspondence used to construct compatible samplers. Throughout, \pfmVelocity{blue} highlights velocity targets and posterior means; \pfmDiffusion{orange} highlights posterior variance parameters and diffusion coefficients.

\vspace{-4pt}
\subsection{Reference paths and flow matching}
\label{sec:reference-path}
\vspace{-4pt}

We consider continuous states in $\mathbb R^d$, with generation time $t$ increasing from noise at $t\!=\!0$ to data at $t\!=\!1$. Uppercase letters denote random variables; Table~\ref{tab:notation} in Appendix~\ref{app:notation} collects the notation used throughout. Let $X_0\sim p_0$ and $X_1\sim p_1=p_{\mathrm{data}}$ be independent, and define the reference path and its velocity target by
\begin{equation}
 X_t=(1-t)X_0+tX_1,\qquad \pfmVelocity{V}=X_1-X_0,\qquad Z=(X_0,X_1).
 \label{eq:reference-path}
\end{equation}
Each pair $Z$ fixes an interpolation, while the marginal velocity averages compatible targets:
\begin{equation}
 \pfmVelocity{v}(x,t)=\mathbb E[\pfmVelocity{V}\mid X_t=x],\qquad
 \partial_t p_t=-\nabla_x\!\cdot(\pfmVelocity{v}p_t),
 \label{eq:marginal-velocity}
\end{equation}
where $p_t$ is the law of $X_t$. When this law has a differentiable positive density, its score is $s(x,t)=\nabla_x\log p_t(x)$. A sampler $Y_t$ with $dY_t=\pfmVelocity{v}(Y_t,t)dt$ and $Y_0\sim p_0$ follows these marginals when the flow and density evolution are well posed.

Conditional flow matching (CFM)~\citep{lipman2023flow,albergo2023building} trains a neural velocity $\pfmVelocity{v_\eta}$ using targets conditioned on $Z$:
\begin{equation}
 \mathcal L_{\mathrm{CFM}}(\eta)
 =\frac12\mathbb E_{t,X_0,X_1}\|\pfmVelocity{V}-\pfmVelocity{v_\eta}(X_t,t)\|^2,
 \qquad t\sim r,\quad t<1.
 \label{eq:cfm-loss}
\end{equation}
The time law $r$ is fixed independently of the parameters. This sample-based loss differs from regression to $\pfmVelocity{v}$ by a parameter-independent conditional-variance term, so both objectives have the same gradient~\citep{lipman2024flowmatchingguidecode} under the differentiation conditions in Appendix~\ref{app:fm}.
This regression identifies the conditional mean needed for transport, linking the sample-based objective to the marginal velocity field that drives generation; PFM extends this regression by also learning the conditional variance of the same targets. In conditional generation, $C$ specifies a class, observation, or reference structure held fixed within a trajectory, whereas $Z$ specifies the sampled reference interpolation; thus, PFM conditions on $C$ while its learned variance captures variation across paths $Z$ compatible with $(X_t,C)$. We omit $C$ from the main formulation and present the complete conditional treatment in Appendix~\ref{app:conditional}.
% , covering reference construction, posterior learning, generation, and diagnostics.

\vspace{-4pt}
\subsection{Score-based diffusion}
\label{sec:diffusion}
\vspace{-4pt}

Differently from FM, diffusion starts from a prescribed noising process~\citep{ho2020denoising,song2021score}. Writing noising time as $\tau=1-t$, let
\begin{equation}
 d\bar X_\tau=f(\bar X_\tau,\tau)d\tau+\pfmDiffusion{g}(\tau)dW_\tau,
 \qquad \text{and} \qquad 
 \bar X_0\sim p_{\mathrm{data}},
 \label{eq:forward-diffusion}
\end{equation}
where $f$ is the noising drift, $\pfmDiffusion{g}$ a scalar noise coefficient, and $W$ standard $d$-dimensional Brownian motion. Let $\bar p_\tau$ be the law of $\bar X_\tau$ and $\bar s(x,\tau)=\nabla_x\log\bar p_\tau(x)$ its score. At generation time, a network $s_\phi(x,t)$ estimates $\bar s(x,1\!-\!t)$, through denoising score matching with a known perturbation kernel~\citep{song2021score,lai2025principlesdiffusionmodels} (see~\ref{app:diffusion-training}). In increasing generation time, the reverse SDE and probability-flow ODE are~\citep{song2021score}
\begin{align}
 dY_t&=[-f(Y_t,1-t)+\pfmDiffusion{g}(1-t)^2\bar s(Y_t,1-t)]dt+\pfmDiffusion{g}(1-t)dW_t,
 \label{eq:diffusion-sde}\\
 dY_t&=[-f(Y_t,1-t)+\tfrac12\pfmDiffusion{g}(1-t)^2\bar s(Y_t,1-t)]dt.
 \label{eq:diffusion-ode}
\end{align}
They share the noising-path marginals with the exact score and initialization from the terminal noising law, while in practice, both use $s_\phi$. This ODE--SDE pairing is a special case of the correspondence below, which also permits state-dependent diffusion. The parameters $\eta$, $\phi$, and $\theta$ index FM, diffusion, and PFM, respectively.

\vspace{-4pt}
\subsection{A shared ODE--SDE correspondence}
\label{sec:correspondence}
\vspace{-4pt}

To compare deterministic and stochastic generation within a common framework, we first determine when an ODE and an SDE induce the same evolving density. This correspondence will allow PFM to use its learned posterior variance to control stochasticity while preserving the reference marginals under exact dynamics. We write the stochastic sampler as an It\^o SDE,
\begin{equation*}
 dY_t=\underbrace{b(Y_t,t)dt}_{\text{drift increment}}
 +\underbrace{\pfmDiffusion{B}(Y_t,t)dW_t}_{\text{Brownian noise increment}}.
\end{equation*}
Here, $b$ is the full drift and $\pfmDiffusion{B}$ maps a Brownian increment into state-space noise, so the It\^o form separates directed motion from a zero-mean random perturbation. To quantify this perturbation, define the diffusion tensor $\pfmDiffusion{a}=\pfmDiffusion{B}\pfmDiffusion{B}^\top$, whose matrix divergence is $(\nabla_x\!\cdot \pfmDiffusion{a})_i=\sum_j\partial_{x_j}\pfmDiffusion{a_{ij}}$. The tensor $\pfmDiffusion{a}$ is a local covariance rate, so an increment of duration $h$ has conditional covariance $h\,\pfmDiffusion{a}(x,t)+o(h)$.

The remaining question is how the drift $b$ must change when $\pfmDiffusion{a}$ is introduced so that the sampler continues to follow a prescribed density path $p_t$. Writing $s=\nabla_x\log p_t$ for its score and rewriting the Fokker--Planck equation as a continuity equation gives~\citep{song2021score,albergo2025stochastic}
\begin{equation}
 \begin{aligned}
 \pfmVelocity{v}=b-\frac12(\nabla_x\!\cdot \pfmDiffusion{a}+\pfmDiffusion{a}\,s),
 \qquad
 b=\underbrace{\pfmVelocity{v}}_{\text{ODE velocity}}
 +\underbrace{\frac12(\nabla_x\!\cdot \pfmDiffusion{a}+\pfmDiffusion{a}\,s)}_{\text{drift correction}}.
 \end{aligned}
 \label{eq:correspondence}
\end{equation}
Equation~\ref{eq:correspondence} provides the desired correspondence: after choosing a positive-semidefinite $\pfmDiffusion{a}$, the correction added to $\pfmVelocity{v}$ compensates for the probability flux induced by the stochastic perturbation. Thus, a density path transported by $\pfmVelocity{v}$ also admits an SDE with the corresponding drift; with smooth coefficients, suitable boundary behavior, unique density evolution, and matching initialization, the two dynamics have the same one-time marginals. Appendix~\ref{app:compatibility} gives the conditions and derivation.

Diffusion provides a familiar instance of this construction. For the reverse SDE in Eq.~\ref{eq:diffusion-sde}, $\pfmDiffusion{a}=\pfmDiffusion{g}(1-t)^2I$ is independent of the state, so its divergence vanishes and Eq.~\ref{eq:correspondence} recovers the probability-flow ODE in Eq.~\ref{eq:diffusion-ode}.
For FM (and PFM), using this correspondence requires a score in addition to the learned velocity. To obtain it without training another model, specialize Eq.~\ref{eq:reference-path} to $X_0\sim\mathcal N(0,I)$; the resulting Gaussian kernel gives~\citep{lipman2024flowmatchingguidecode}
\begin{equation}
 X_t\mid X_1\sim\mathcal N(tX_1,(1-t)^2I),\qquad
 s(x,t)=\frac{t\,\pfmVelocity{v}(x,t)-x}{1-t},\quad t<1.
 \label{eq:velocity-score}
\end{equation}
A velocity prediction therefore supplies the derived score $\widehat s_\eta=(t\,\pfmVelocity{v_\eta}-x)/(1-t)$ needed by the drift correction. Standard FM can then choose a prescribed isotropic tensor $\pfmDiffusion{a}=\pfmDiffusion{\kappa}(t)I$, with $\pfmDiffusion{\kappa}(t)\geq0$, which assigns the same covariance rate in every direction~\citep{albergo2025stochastic}. Its drift follows from Eq.~\ref{eq:correspondence}, its noise amplitude is $\sqrt{\pfmDiffusion{\kappa}(t)}$, and $\pfmDiffusion{\kappa}\equiv0$ recovers the ODE. This construction provides a stochastic FM sampler but does not use posterior information to adapt its noise across states and coordinates; PFM will instead choose the diffusion tensor from learned posterior moments in Section~\ref{sec:sampling}. Appendices~\ref{app:score} and~\ref{app:baseline-sampling} derive the score identity and baseline samplers.

\vspace{-4pt}
\section{Method}
\label{sec:method}

%\reviewadd{To turn FM's transport prediction into a posterior that supports both uncertainty estimation and posterior-driven sampling, }
PFM learns the conditional velocity mean and variance to describe the data endpoints compatible with an intermediate state (Figure~\ref{fig:pfm-method}, with the full overview in Figure~\ref{fig:pfm-method-full}). We first connect posterior fitting to endpoint information and define a Gaussian learning objective. The learned moments then determine ODE and SDE sampling and provide endpoint-posterior readouts for diagnostics.
\begin{figure}[t]
\centering
% Both versions reuse the original vector artwork at the same scale.
\makebox[\linewidth][c]{%
  % The bottom crop ends in the gap between panels B and C.
  \includegraphics[trim=20bp 767bp 20bp 8bp,clip,width=1.05\linewidth]{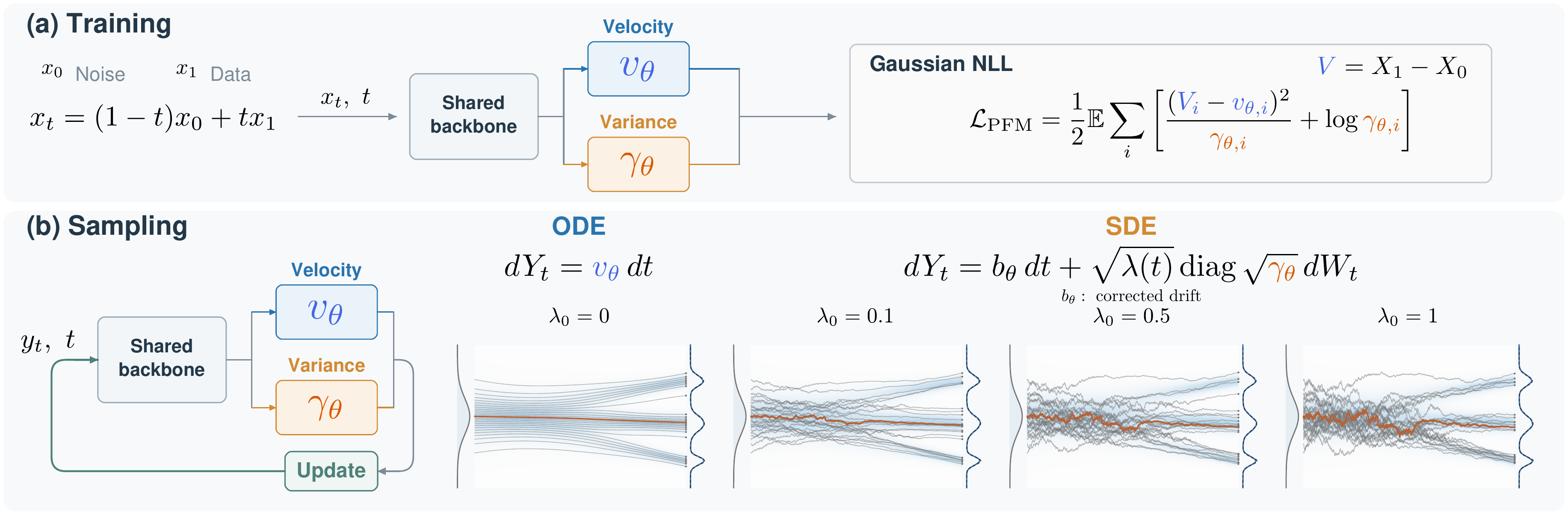}%
}
\par\vspace{-1em}
\caption[Posterior Flow Matching.]{\input{figure/method_caption} See Figure~\ref{fig:pfm-method-full} for the full overview.}
\label{fig:pfm-method}
\vspace{-1em}
\end{figure}

\vspace{-4pt}
\subsection{Learning conditional velocity posteriors}
\label{sec:learning}

The FM mean supplies a transport direction, but leaves the spread of compatible endpoints unspecified. At a fixed time, this remaining uncertainty is described by $p(X_1\mid X_t)$ and its conditional entropy. 
To formalize how much endpoint information the state retains, mutual information (MI)~\citep{xiao2026directions} relates it to the information available in the state through $I(X_t;X_1)=H(X_1)-H(X_1\mid X_t)$, where $H$ denotes differential entropy. For an endpoint-posterior model $q_\theta^x$, the expected log-likelihood gives the variational bound~\citep{barber2003variational,poole2019variational}
\begin{equation}
 I(X_t;X_1)\geq H(X_1)+\mathbb E\log q_\theta^x(X_1\mid X_t,t).
 \label{eq:mi-bound}
\end{equation}
Because the reference path fixes the MI, minimizing the posterior NLL tightens this bound by reducing its gap, the expected Kullback--Leibler (KL) divergence, also called relative entropy, from the true posterior to $q_\theta^x$. Appendix~\ref{app:information-bounds} derives this identity under finite-entropy and integrability conditions. For external generation inputs $C$, the same argument uses conditional mutual information (CMI)~\citep{ma2026cadvae}, developed in Appendix~\ref{app:conditional-learning}.

To fit this posterior while retaining the FM transport parameterization, use the reference-path identity
\begin{equation}
 X_1=X_t+(1-t)\pfmVelocity{V}.
 \label{eq:endpoint-map}
\end{equation}
At fixed $(x,t)$ with $t<1$, this invertible affine map converts a velocity posterior into an endpoint posterior. Their NLLs differ only by $d\log(1-t)$, independently of $\theta$ (Appendix~\ref{app:posterior-map}). We can therefore fit the endpoint posterior in velocity coordinates, where its mean directly supplies transport.

To estimate both moments in one model evaluation, a shared backbone predicts the velocity mean and log variance,
\begin{equation}
 (\pfmVelocity{v_\theta},\pfmDiffusion{\ell_\theta})=F_\theta(x,t),\qquad
 \pfmDiffusion{\gamma_\theta}=\exp\pfmDiffusion{\ell_\theta},\qquad
 \pfmDiffusion{\Sigma_\theta}=\operatorname{diag}\pfmDiffusion{\gamma_\theta},
 \label{eq:heads}
\end{equation}
which define a diagonal-Gaussian approximation to the conditional velocity posterior:
\begin{equation}
 q_\theta(v\mid x,t)=\mathcal N(v;\pfmVelocity{v_\theta}(x,t),\pfmDiffusion{\Sigma_\theta}(x,t)).
 \label{eq:posterior}
\end{equation}
Both heads depend on the full state. Up to a constant, the training objective is
\begin{equation}
 \boxed{\begin{aligned}
 \mathcal L_{\mathrm{PFM}}(\theta)
 &=\frac12\mathbb E\!\left[(\pfmVelocity{V}-\pfmVelocity{v_\theta})^\top\pfmDiffusion{\Sigma_\theta}^{-1}(\pfmVelocity{V}-\pfmVelocity{v_\theta})
 +\log\det\pfmDiffusion{\Sigma_\theta}\right]\\
 &=\frac12\mathbb E\!\left[\sum_{i=1}^d
 \left(\frac{(\pfmVelocity{V_i}-\pfmVelocity{v_{\theta,i}})^2}{\pfmDiffusion{\gamma_{\theta,i}}}
 +\log\pfmDiffusion{\gamma_{\theta,i}}\right)\right].
 \end{aligned}}
 \label{eq:pfm-loss}
\end{equation}
The expectation uses the same reference pairs and time law as FM, with outputs evaluated at $(X_t,t)$. The weighted residual and log-variance terms jointly fit both moments. The proposition below asks whether this objective preserves the FM transport target while recovering conditional dispersion.

\begin{proposition}[Conditional moment recovery; proof in Appendix~\ref{app:moments}]
\label{prop:moments}
Fix $(x,t)$ and suppose $\pfmVelocity{V}\mid X_t=x$ has finite second moments and positive componentwise variances. The expected NLL of $q_\theta(\cdot\mid x,t)$ is uniquely minimized over its mean and positive diagonal variances at
\begin{equation}
 \pfmVelocity{v_\theta}(x,t)=\pfmVelocity{v}(x,t),\qquad
 \pfmDiffusion{\gamma_{\theta,i}}(x,t)=\operatorname{Var}(\pfmVelocity{V_i}\mid X_t=x).
 \label{eq:optimal-moments}
\end{equation}
\end{proposition}

Appendix~\ref{app:moments} proves the result. Fixing $\pfmDiffusion{\Sigma_\theta}=I$ recovers the FM objective in Eq.~\ref{eq:cfm-loss}, while learning it adds conditional variance estimation without changing the population mean target. The shared backbone and both heads are trained jointly by Algorithm~\ref{alg:training}.

To make the learned uncertainty useful for prediction and diagnostics, the affine map in Eq.~\ref{eq:endpoint-map} converts these velocity moments into endpoint-posterior readouts:
\begin{equation}
 \pfmVelocity{\mu_\theta}(x,t)=x+(1-t)\pfmVelocity{v_\theta}(x,t),\qquad
 \pfmDiffusion{\Sigma_\theta}^x(x,t)=(1-t)^2\pfmDiffusion{\Sigma_\theta}(x,t).
 \label{eq:endpoint-readout}
\end{equation}
These endpoint-posterior readouts require no additional head. Appendix~\ref{app:posterior-map} derives their moment transformation, and Appendix~\ref{app:endpoint} develops the alternative of predicting the endpoint mean directly.

\vspace{-4pt}
\subsection{Posterior-driven ODE and SDE sampling}
\label{sec:sampling}

PFM uses its learned mean and variance to separate transport from adaptive noise. Both samplers start from $Y_0\!\sim\!\mathcal N(0,I)$. The learned mean defines the PFM ODE and the derived score of Eq.~\ref{eq:velocity-score}:
\begin{equation}
 dY_t=\pfmVelocity{v_\theta}(Y_t,t)dt,\qquad
 \widehat s_\theta(x,t)=\frac{t\,\pfmVelocity{v_\theta}(x,t)-x}{1-t}.
 \label{eq:pfm-ode-score}
\end{equation}
To scale stochastic perturbations by learned uncertainty, we choose a diffusion covariance rate that is isotropic after normalization by the local velocity posterior standard deviations:
\begin{equation}
 \pfmDiffusion{\Sigma_\theta}^{-1/2}\pfmDiffusion{a_\theta}\pfmDiffusion{\Sigma_\theta}^{-1/2}=\lambda(t)I,\qquad \lambda(t)\geq0.
 \label{eq:posterior-normalization}
\end{equation}
This gives the diffusion tensor and a Brownian noise coefficient
\begin{equation}
 \pfmDiffusion{a_\theta}=\lambda(t)\pfmDiffusion{\Sigma_\theta},\qquad
 \pfmDiffusion{B_\theta}=\sqrt{\lambda(t)}\operatorname{diag}\sqrt{\pfmDiffusion{\gamma_\theta}}.
 \label{eq:posterior-diffusion}
\end{equation}
Thus posterior variances set relative noise strengths across states and coordinates, while the inference-time schedule $\lambda$ controls the overall rate. Appendix~\ref{app:diffusion-design} derives the corresponding increment covariance and its locally normalized form.

Adding diffusion changes density evolution, so it must be paired with a drift correction. The correspondence in Eq.~\ref{eq:correspondence} gives the PFM SDE
\begin{equation}
 \boxed{dY_t=\left[\pfmVelocity{v_\theta}+\frac12\nabla_x\!\cdot \pfmDiffusion{a_\theta}
 +\frac12\pfmDiffusion{a_\theta}\widehat s_\theta\right](Y_t,t)dt
 +\pfmDiffusion{B_\theta}(Y_t,t)dW_t.}
 \label{eq:pfm-sde}
\end{equation}
The score correction compensates for diffusion across density gradients, and the divergence correction accounts for spatially varying covariance. At the population mean in Proposition~\ref{prop:moments}, the derived score is exact and these dynamics preserve the reference marginals under the conditions in Appendix~\ref{app:compatibility}. Setting $\lambda=0$ removes diffusion and both corrections, recovering the PFM ODE. Appendix~\ref{app:coordinate-drift} derives the drift, and Appendix~\ref{app:numerics} gives the numerical samplers.

\textbf{Relative-entropy dissipation.}
Beyond preserving the reference marginals, the diffusion also determines how density perturbations evolve. Let $p_t$ be the reference density and let $\widetilde p_t$ start from another density under the same compatible coefficients, with the score in their common drift fixed by $p_t$~\citep{fontbona2016trajectorial,wibisono2017information}.
\begin{proposition}[Relative-entropy dissipation; proof in Appendix~\ref{app:kl-proof}]
\label{prop:kl}
On $[0,T]$ with $T<1$, suppose both densities are positive and smooth, the displayed quantities are finite, and boundary terms vanish under integration by parts. For the SDE with $\pfmDiffusion{a}=\lambda\pfmDiffusion{\Sigma_\theta}$,
\begin{equation}
 \frac{d}{dt}\operatorname{KL}(\widetilde p_t\|p_t)
 =-\frac{\lambda(t)}2\int\widetilde p_t\sum_i\pfmDiffusion{\gamma_{\theta,i}}
 \left(\partial_{x_i}\log\frac{\widetilde p_t}{p_t}\right)^2dx\leq0.
 \label{eq:kl-coordinate}
\end{equation}
When $\lambda\equiv0$, both densities follow the same ODE, and their KL remains constant.
\end{proposition}
Posterior variances weight coordinatewise dissipation; KL decreases strictly when $\lambda(t)>0$ and the integral above is positive. Appendix~\ref{app:kl-proof} proves the general matrix form, and Appendix~\ref{app:trajectory-dynamics} illustrates one-dimensional KL evolution and cumulative dissipation.

\textbf{Relation to the standard generators.}
To isolate the role of the learned variance, Table~\ref{tab:generative-dynamics} compares FM, diffusion, and PFM under the shared ODE--SDE correspondence. FM and PFM share the reference path and population velocity target. PFM additionally learns the variance that supplies local diffusion scales and endpoint diagnostics. Appendix~\ref{app:baseline-sampling} derives the baseline entries, and Appendix~\ref{app:trajectory-dynamics} visualizes the dynamics.
\begin{table}[!htbp]
\centering
% Keep the compact caption clear of the table rule.
\setlength{\belowcaptionskip}{1em}
\par\vspace{-1em}
\caption{Predicted fields, generation dynamics, and posterior moments. Fields use $(Y_t,t)$, with $f,\pfmDiffusion{g}$ at noising time $1-t$. The tensor $\pfmDiffusion{a}$ is the Brownian covariance rate. Semicolons separate velocity and endpoint readouts. Baseline endpoint means follow Eq.~\ref{eq:baseline-readouts}, with $\dagger$ denoting a Gaussian perturbation kernel. PFM readouts follow Eq.~\ref{eq:endpoint-readout}. Appendix~\ref{app:baseline-sampling} gives the reference paths and derivations.}
\label{tab:generative-dynamics}
\vspace{-1em}
\scriptsize
\setlength{\tabcolsep}{3pt}
\begin{tabular*}{\linewidth}{@{\extracolsep{\fill}}lllll@{}}
\toprule
Model / predicted fields & Sampler & Generation dynamics & $\pfmDiffusion{a}$ & Moment readouts \\
\midrule
FM ($\pfmVelocity{v_\eta}$) & ODE & $dY_t=\pfmVelocity{v_\eta}\,dt$ & $0$ & $\pfmVelocity{v_\eta};\ \pfmVelocity{\mu_\eta}$ \\
FM ($\pfmVelocity{v_\eta}$) & SDE & $dY_t=[\pfmVelocity{v_\eta}+\tfrac12\pfmDiffusion{\kappa}\widehat s_\eta]dt+\sqrt{\pfmDiffusion{\kappa}}\,dW_t$ & $\pfmDiffusion{\kappa} I$ & $\pfmVelocity{v_\eta};\ \pfmVelocity{\mu_\eta}$ \\
\addlinespace
Diffusion ($s_\phi$) & ODE & $dY_t=[-f+\tfrac12\pfmDiffusion{g}^2s_\phi]dt$ & $0$ & $\pfmVelocity{\mu_\phi}^{\dagger}$ \\
Diffusion ($s_\phi$) & SDE & $dY_t=[-f+\pfmDiffusion{g}^2s_\phi]dt+\pfmDiffusion{g}\,dW_t$ & $\pfmDiffusion{g}^2I$ & $\pfmVelocity{\mu_\phi}^{\dagger}$ \\
\addlinespace
PFM ($\pfmVelocity{v_\theta},\pfmDiffusion{\ell_\theta}$) & ODE & $dY_t=\pfmVelocity{v_\theta}\,dt$ & $0$ & $(\pfmVelocity{v_\theta},\pfmDiffusion{\Sigma_\theta});\ (\pfmVelocity{\mu_\theta},\pfmDiffusion{\Sigma_\theta}^x)$ \\
PFM ($\pfmVelocity{v_\theta},\pfmDiffusion{\ell_\theta}$) & SDE & $dY_t=[\pfmVelocity{v_\theta}+\tfrac12\nabla_x\!\cdot \pfmDiffusion{a_\theta}+\tfrac12\pfmDiffusion{a_\theta}\widehat s_\theta]dt+\pfmDiffusion{B_\theta}\,dW_t$ & $\lambda\pfmDiffusion{\Sigma_\theta}$ & $(\pfmVelocity{v_\theta},\pfmDiffusion{\Sigma_\theta});\ (\pfmVelocity{\mu_\theta},\pfmDiffusion{\Sigma_\theta}^x)$ \\
\bottomrule
\end{tabular*}
\par\vspace{-1em}
\end{table}

\vspace{-4pt}
\subsection{Posterior diagnostics}
\label{sec:diagnostics}

The endpoint-posterior readouts in Eq.~\ref{eq:endpoint-readout} make the evolving predictions interpretable under either sampler. At a generated state $Y_t=y$, they approximate the moments of the reference endpoint posterior $p(X_1\mid X_t=y)$. The mean describes a predicted endpoint, and the covariance describes the dispersion of compatible data endpoints, as illustrated across domains in Figures~\ref{fig:exp-image}--\ref{fig:exp-pusht}.

To assess these readouts, we construct reference states by interpolating independent noise and known data endpoints using Eq.~\ref{eq:reference-path}. At matched times, we compare predicted variances with mean squared prediction residuals. We also compare Gaussian NLLs at a shared predicted mean to assess the covariance fit. Appendix~\ref{app:diagnostics} explains these moment and likelihood checks, connecting posterior learning to the quantitative diagnostics in Section~\ref{sec:posterior-diagnostics}.

\vspace{-4pt}
% Favor mixed text/figure pages over sparsely filled float-only pages.
\let\pfmSavedFloatPageFraction\floatpagefraction
\renewcommand{\floatpagefraction}{0.85}
\section{Experiments}
\label{sec:experiments}
\vspace{-6pt}

\textbf{Experimental Setup}.
\label{sec:experimental-setup}
We evaluate PFM on image generation, point-cloud generation, molecular generation, and action generation. Across tasks, PFM learns Gaussian velocity posteriors, with mean and covariance parameterizations adapted to each task. ImageNet~\citep{deng2009imagenet} predicts the endpoint mean directly (Appendix~\ref{app:endpoint}), while molecular generation~\citep{ramakrishnan2014quantum} uses a covariance projected onto the centered coordinate subspace (Appendix~\ref{app:experiment-molecule}). PFM ODE and PFM SDE sampling share the trained model within each task and category.
% External conditions are omitted from the readout notation, following the main-text convention.
% We first examine generation across domains, then interpret the accompanying posterior readouts, and finally vary sampling stochasticity. 
Appendix~\ref{app:experiments} provides the core implementation settings, task-specific adaptations, and supplementary results.

\vspace{-8pt}
\begin{table}[!htbp]
\begin{minipage}[t]{0.32\linewidth}
\vspace{0pt}
\centering
% Keep the compact caption clear of the table rule.
\setlength{\belowcaptionskip}{1em}
\par\vspace{-1em}
\caption{PushT control scores $\uparrow$.}
\label{tab:pusht-table}
\vspace{-1em}
\scriptsize
\setlength{\tabcolsep}{4pt}
\begin{tabular}{@{}lrr@{}}
\toprule
Method & 20 steps & 50 steps \\
\midrule
FM & 0.825 & 0.872 \\
Diffusion Policy & 0.823 & 0.867 \\
PFM ODE & 0.827 & 0.878 \\
PFM SDE & \textbf{0.848} & \textbf{0.914} \\
\bottomrule
\end{tabular}
\par\vspace{-1em}

\end{minipage}\hfill
\begin{minipage}[t]{0.66\linewidth}
\vspace{0pt}
\centering
% Keep the compact caption clear of the table rule.
\setlength{\belowcaptionskip}{1em}
\par\vspace{-1em}
\caption{Molecule generation. R/P: recall/precision.}
\label{tab:molecule-table}
\vspace{-1em}
\scriptsize
\setlength{\tabcolsep}{2.4pt}
\begin{tabular}{@{}llrrrr@{}}
\toprule
Method & Sampler & COV-R (\%) $\uparrow$ & COV-P (\%) $\uparrow$ & MAT-R (\AA{}) $\downarrow$ & MAT-P (\AA{}) $\downarrow$ \\
\midrule
FM (ET-Flow) & ODE & 96.47 & 94.05 & 0.073 & 0.098 \\
FM (ET-Flow--SO(3)) & ODE & 95.98 & 92.10 & 0.076 & 0.110 \\
PFM & ODE & 96.44 & 93.99 & 0.075 & 0.108 \\
PFM & SDE & \textbf{98.37} & \textbf{96.42} & \textbf{0.063} & \textbf{0.081} \\
\bottomrule
\end{tabular}
\par\vspace{-1em}

\end{minipage}

\vspace{1em}
\begin{minipage}[t]{0.28\linewidth}
\vspace{0pt}
\centering
% Keep the compact caption clear of the table rule.
\setlength{\belowcaptionskip}{1em}
\par\vspace{-1em}
\caption{ImageNet-256.}
\label{tab:imagenet-table}
\vspace{-1em}
\scriptsize
\setlength{\tabcolsep}{4pt}
\begin{tabular}{@{}lrr@{}}
\toprule
Method & FID $\downarrow$ & IS $\uparrow$ \\
\midrule
FM & 25.42 & 63.2 \\
Diffusion & 27.38 & 62.1 \\
PFM ODE & 24.45 & 65.44 \\
PFM SDE & \textbf{22.28} & \textbf{68.06} \\
\bottomrule
\end{tabular}
\par\vspace{-1em}

\end{minipage}\hfill
\begin{minipage}[t]{0.70\linewidth}
\vspace{0pt}
\centering
% Keep the compact caption clear of the table rule.
\setlength{\belowcaptionskip}{1em}
\par\vspace{-1em}
\caption{ShapeNet (CD). Ideal 1-NNA: 50\%.}
\label{tab:pointcloud-table}
\vspace{-1em}
\scriptsize
\setlength{\tabcolsep}{2.5pt}
\begin{tabular*}{\linewidth}{@{\extracolsep{\fill}}lrrrrrr@{}}
\toprule
& \multicolumn{3}{c}{Chair} & \multicolumn{3}{c}{Airplane} \\
Method & MMD$\times10^3$ $\downarrow$ & COV (\%) $\uparrow$ & 1-NNA (\%) & MMD$\times10^3$ $\downarrow$ & COV (\%) $\uparrow$ & 1-NNA (\%) \\
\midrule
Diffusion & 12.28 & 48.94 & 60.11 & 3.28 & 48.71 & 64.83 \\
FM & 12.04 & 49.23 & 61.74 & 3.26 & 49.34 & 63.19 \\
PFM ODE & 11.57 & 50.98 & 59.95 & 3.23 & 50.13 & 62.70 \\
PFM SDE & \textbf{10.50} & \textbf{51.58} & \textbf{58.23} & \textbf{3.06} & \textbf{51.80} & \textbf{60.23} \\
\bottomrule
\end{tabular*}
\par\vspace{-1em}

\end{minipage}
\end{table}

\vspace{-4pt}

% \vspace{-4pt}
\subsection{Generation across Domains}
\label{sec:generation-domains}
\vspace{-4pt}

% The four domains instantiate the same continuous-state generative formulation. 
% Performance tables compare its ODE and posterior-driven SDE samplers with task-specific reference methods, while visualizations show representative outputs. 
% Their accompanying variance readouts are examined in Section~\ref{sec:posterior-diagnostics}.

\textbf{Image generation.}
PFM generates ImageNet-256~\citep{deng2009imagenet} images spanning objects and scenes. In Table~\ref{tab:imagenet-table}, PFM SDE improves Fr\'echet inception distance (FID)~\citep{heusel2017gans} from $24.45$ to $22.28$ and inception score (IS)~\citep{salimans2016improved} from $65.44$ to $68.06$ relative to PFM ODE, also outperforming the listed FM and diffusion baselines. Figure~\ref{fig:exp-image} shows representative samples; Appendix~\ref{app:experiment-image} adds CIFAR-10~\citep{Krizhevsky2009LearningML} results.
\begin{figure}[tbp]
\centering
\makebox[\linewidth][c]{\includegraphics[width=\linewidth]{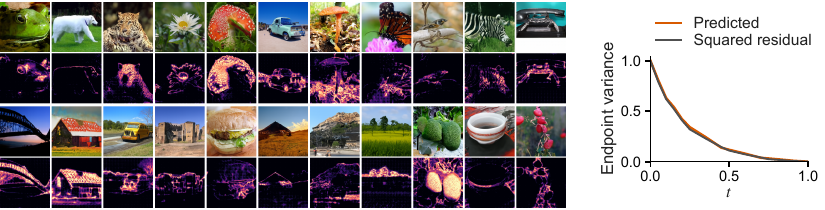}}
\par\vspace{-1em}
\caption{Left: ImageNet-256 endpoint means $\pfmVelocity{\mu_\theta}$, each paired with an earlier-time variance map from $\pfmDiffusion{\Sigma_\theta}^x$ on the same trajectory. Maps average over channels and brighten with increasing variance. Right: comparison of predicted endpoint variance and mean squared prediction residuals over time.}
\label{fig:exp-image}
\vspace{-1em}
\end{figure}

\textbf{3D point-cloud generation.}
PFM generates Chair and Airplane shapes on ShapeNet~\citep{chang2015shapenet}. Using Chamfer distance (CD), we evaluate minimum matching distance (MMD), coverage (COV)~\citep{achlioptas2018learning}, and 1-nearest-neighbor accuracy (1-NNA)~\citep{yang2019pointflow}. MMD averages the distance from each reference shape to its nearest generated shape. In both categories, PFM SDE lowers MMD-CD, increases coverage, and moves 1-NNA-CD closer to its ideal value of $50\%$ relative to PFM ODE and the listed baselines (Table~\ref{tab:pointcloud-table}). These gains span both geometric fidelity and distributional coverage. Figure~\ref{fig:exp-pointcloud} follows the emerging endpoint geometry.
\begin{figure}[tbp]
\centering
\makebox[\linewidth][c]{\includegraphics[width=0.95\linewidth]{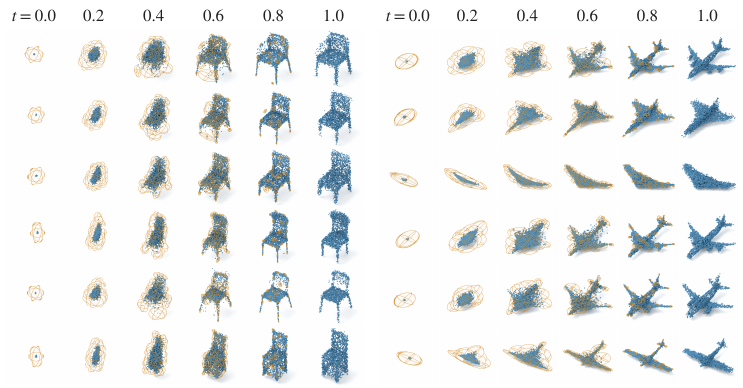}}
\par\vspace{-1em}
\caption{Chair (left) and Airplane (right) endpoint predictions over generation time. Each row follows one example. Blue points show $\pfmVelocity{\mu_\theta}$, and orange ellipsoids show one-standard-deviation contours from $\pfmDiffusion{\Sigma_\theta}^x$ at a fixed subset of points tracked across time.}
\label{fig:exp-pointcloud}
\vspace{-1em}
\end{figure}

\textbf{Molecular conformation generation.}
PFM generates molecular conformations on GEOM-QM9~\citep{axelrod2022geom}. PFM SDE improves all four metrics~\citep{hassan2024etflow} over PFM ODE and the listed reference methods (Table~\ref{tab:molecule-table}): recall and precision coverage (COV-R/P) rise to $98.37\%$ and $96.42\%$, while recall and precision matching distances (MAT-R/P) fall to $0.063$ and $0.081\,\text{\AA}$. Thus the improvement combines broader conformation coverage with more accurate generated conformations. Figure~\ref{fig:exp-molecule} follows the evolving conformations.
\begin{figure}[tbp]
\centering
\makebox[\linewidth][c]{\includegraphics[width=0.95\linewidth]{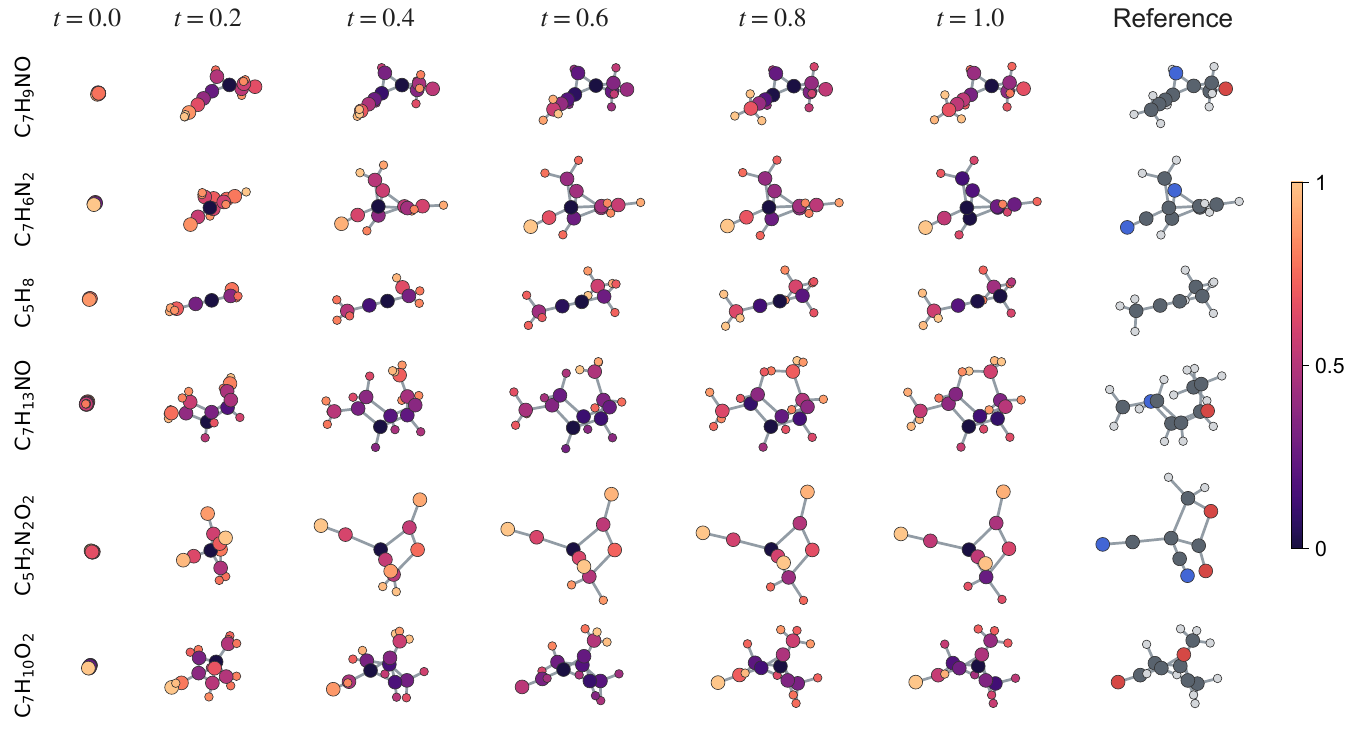}}
\par\vspace{-1em}
\caption{Molecular endpoint means $\pfmVelocity{\mu_\theta}$ over generation time, followed by reference conformations. Colors show atomwise variance from $\pfmDiffusion{\Sigma_\theta}^x$. Reference colors denote elements.}
\label{fig:exp-molecule}
\vspace{-1em}
\end{figure}

\textbf{Action generation on PushT.}
On PushT~\citep{florence2022implicit,chi2025diffusion}, PFM generates action sequences conditioned on observations for use in a control loop. PFM SDE improves the control score from $0.827$ to $0.848$ at 20 steps and from $0.878$ to $0.914$ at 50 steps, exceeding the listed FM and Diffusion Policy~\citep{chi2025diffusion} scores at both budgets (Table~\ref{tab:pusht-table}). Figure~\ref{fig:exp-pusht} links generated actions to execution across planning calls.
\begin{figure}[!htbp]
\centering
\makebox[\linewidth][c]{\includegraphics[width=0.95\linewidth]{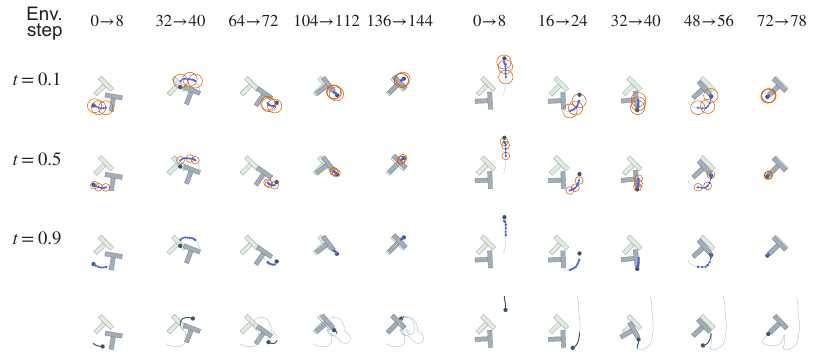}}
\par\vspace{-1em}
\caption{PushT posterior readouts for two episodes. Columns indicate the environment steps from planning to execution, and $t$ denotes generation time. Upper rows show endpoint action means $\pfmVelocity{\mu_\theta}$ (blue) and standard-deviation ellipses from $\pfmDiffusion{\Sigma_\theta}^x$ (orange). The bottom row shows execution.}
\label{fig:exp-pusht}
\vspace{-1em}
\end{figure}

\vspace{-4pt}
\subsection{Posterior Diagnostics}
\label{sec:posterior-diagnostics}
\vspace{-4pt}

\textbf{Visualizing endpoint predictions and variances.}
Alongside endpoint means $\pfmVelocity{\mu_\theta}$, PFM supplies covariances $\pfmDiffusion{\Sigma_\theta}^x$ through Eq.~\ref{eq:endpoint-readout}. Image maps highlight dispersion around object boundaries, while molecular maps show atomwise variance (Figures~\ref{fig:exp-image} and~\ref{fig:exp-molecule}). Point-cloud ellipsoids and action ellipses show endpoint standard deviations around predicted means and contract over time (Figures~\ref{fig:exp-pointcloud} and~\ref{fig:exp-pusht}).

\textbf{Quantitative posterior checks.}
The right panel of Figure~\ref{fig:exp-image} compares predicted endpoint variance with mean squared prediction residuals over time. The curves closely track each other as both decrease, linking the predicted uncertainty to the evolving prediction error.

On reference states, Figure~\ref{fig:exp-posterior_diagnostics}(a,b) shows broad agreement between mean predicted variances and mean squared residuals for Chair and Airplane point clouds, with departures in low-variance and sparse bins. Molecular covariance traces similarly agree with mean squared residual norms (c). With the molecular mean fixed, learned covariance yields lower Gaussian NLL at every evaluated time than a covariance that depends only on time and is fitted on validation data (d). Appendix~\ref{app:experiment-molecule} details this molecular NLL comparison.

\vspace{-6pt}
\begin{figure}[!htbp]
\centering
\includegraphics[width=0.9\linewidth]{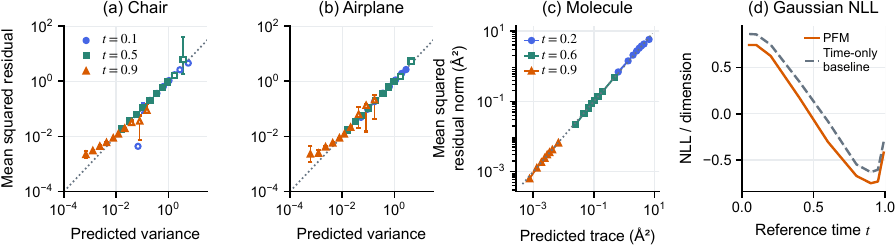}
\par\vspace{-1em}
\caption{Posterior diagnostics grouped at fixed reference times. (a,b) Chair/Airplane endpoint variances versus mean squared residuals. (c) Molecular covariance traces versus mean squared residual norms. (d) Molecular velocity-posterior NLL per subspace dimension, with a shared predicted mean (lower is better). Dotted lines in (a--c) indicate equality. Bars/bands show 95\% bootstrap intervals, and hollow markers indicate sparse bins.}
\label{fig:exp-posterior_diagnostics}
\vspace{-1em}
\end{figure}

\vspace{-4pt}
\subsection{Ablation Studies and Sampling Cost}
\label{sec:ablation-studies}
\vspace{-4pt}

\textbf{Role of posterior covariance.}
In $\pfmDiffusion{a_\theta}=\lambda(t)\pfmDiffusion{\Sigma_\theta}$ (Eq.~\ref{eq:posterior-diffusion}), the posterior covariance determines local noise allocation. The comparison uses a shared mean model and the same scalar noise schedule for the two SDEs. PFM SDE outperforms PFM ODE and PFM + isotropic SDE across all four tasks. Appendix~\ref{app:covariance-ablation} details the comparison.

\textbf{Inference-time stochasticity.}
With this covariance field fixed, $\lambda$ controls the overall stochasticity. Figure~\ref{fig:exp-lambda} illustrates gains at intermediate $\lambda_0$ in $\lambda(t)=\lambda_0(1-t)$ relative to PFM ODE ($\lambda_0=0$), followed by declining quality at larger scales. Thus the same posterior model supports inference-time adjustment: covariance determines local noise allocation, while its overall strength affects sampling quality. Appendix~\ref{app:stochasticity-ablation} connects this adjustment to the one-dimensional generation trajectories.

\vspace{-6pt}
\begin{figure}[!htbp]
\centering
\includegraphics[width=0.8\linewidth]{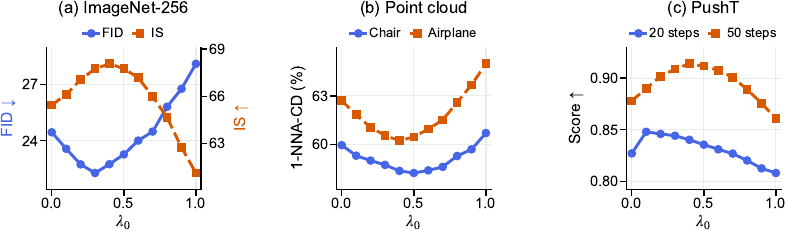}
\par\vspace{-1em}
\caption{Stochasticity sweeps with $\lambda(t)=\lambda_0(1-t)$. (a) ImageNet-256 FID and IS. (b) Chair and Airplane 1-NNA-CD (ideal: $50\%$). (c) PushT control scores at two integration budgets.}
\label{fig:exp-lambda}
\vspace{-1em}
\end{figure}

\textbf{Sampling time.}
These sampling choices also differ in cost. PFM ODE has sampling times close to FM ODE, while PFM SDE takes $1.16$--$1.33\times$ the PFM ODE time. Appendix~\ref{app:sampling-time} reports the per-task times and discusses the additional variance-field derivatives.

\vspace{-4pt}
\section{Conclusion}
\label{sec:conclusion}
\vspace{-4pt}

We introduced Posterior Flow Matching (PFM), a general generative framework that connects data generation with endpoint uncertainty estimation. Across image, point-cloud, molecular, and action generation, PFM improves generation quality while providing uncertainty readouts throughout the generation process.

\clearpage
\section*{AI Use Disclosure}
\label{sec:ai-use}

We used large language models to polish the writing and improve the readability of the manuscript. All AI-assisted edits were reviewed by the authors, who take full responsibility for the final content.

\bibliography{iclr2027_conference}

@inproceedings{lipman2023flow,
  title = {{Flow Matching for Generative Modeling}},
  author = {Lipman, Yaron and Chen, Ricky T. Q. and Ben-Hamu, Heli and Nickel, Maximilian and Le, Matt},
  year = {2023},
  booktitle = {International Conference on Learning Representations},
  url = {https://openreview.net/forum?id=PqvMRDCJT9t}
}

@inproceedings{song2021score,
  title = {{Score-Based Generative Modeling through Stochastic Differential Equations}},
  author = {Song, Yang and Sohl-Dickstein, Jascha and Kingma, Diederik P. and Kumar, Abhishek and Ermon, Stefano and Poole, Ben},
  year = {2021},
  booktitle = {International Conference on Learning Representations},
  url = {https://arxiv.org/abs/2011.13456}
}

@article{albergo2025stochastic,
  title = {{Stochastic Interpolants: A Unifying Framework for Flows and Diffusions}},
  author = {Albergo, Michael S. and Boffi, Nicholas M. and Vanden-Eijnden, Eric},
  year = {2025},
  journal = {Journal of Machine Learning Research},
  volume = {26},
  number = {209},
  pages = {1--80},
  url = {https://jmlr.org/papers/v26/23-1605.html}
}

@misc{han2026uncertainty,
  title = {{Flow Matching with Uncertainty Quantification and Guidance}},
  author = {Han, Juyeop and Beyer, Lukas Lao and Karaman, Sertac},
  year = {2026},
  howpublished = {arXiv:2602.10326},
  eprint = {2602.10326},
  archivePrefix = {arXiv},
  url = {https://arxiv.org/abs/2602.10326}
}

@inproceedings{kou2024bayesdiff,
  title = {{BayesDiff: Estimating Pixel-wise Uncertainty in Diffusion via Bayesian Inference}},
  author = {Kou, Siqi and Gan, Lei and Wang, Dequan and Li, Chongxuan and Deng, Zhijie},
  year = {2024},
  booktitle = {International Conference on Learning Representations},
  url = {https://arxiv.org/abs/2310.11142}
}

@inproceedings{bao2022analytic,
  title = {{Analytic-DPM: an Analytic Estimate of the Optimal Reverse Variance in Diffusion Probabilistic Models}},
  author = {Bao, Fan and Li, Chongxuan and Zhu, Jun and Zhang, Bo},
  year = {2022},
  booktitle = {International Conference on Learning Representations},
  url = {https://openreview.net/forum?id=0xiJLKH-ufZ}
}

@inproceedings{manor2024posterior,
  title = {{On the Posterior Distribution in Denoising: Application to Uncertainty Quantification}},
  author = {Manor, Hila and Michaeli, Tomer},
  year = {2024},
  booktitle = {International Conference on Learning Representations},
  url = {https://openreview.net/forum?id=adSGeugiuj}
}

@misc{stoica2026posterior,
  title = {{Posterior Augmented Flow Matching}},
  author = {Stoica, George and Paul, Sayak and Wallingford, Matthew and Ramanujan, Vivek and Nori, Abhay and Han, Winson and Farhadi, Ali and Krishna, Ranjay and Hoffman, Judy},
  year = {2026},
  howpublished = {arXiv:2605.00825},
  eprint = {2605.00825},
  archivePrefix = {arXiv},
  url = {https://arxiv.org/abs/2605.00825}
}

@misc{dekruiff2026confidence,
  title = {{Flowing with Confidence}},
  author = {de Kruiff, Friso and Coscia, Dario and Welling, Max and Bekkers, Erik},
  year = {2026},
  howpublished = {arXiv:2605.18472},
  eprint = {2605.18472},
  archivePrefix = {arXiv},
  url = {https://arxiv.org/abs/2605.18472}
}

@misc{lv2026singlepass,
  title = {{Uncertainty in a Single Pass: A Closed-Form Identity for One-Step Flow Matching}},
  author = {Lv, Li and Xing, Jiarui and Liang, Xinyan and Wang, Jian},
  year = {2026},
  howpublished = {arXiv:2605.00941v5},
  eprint = {2605.00941v5},
  archivePrefix = {arXiv},
  url = {https://arxiv.org/pdf/2605.00941v5}
}

@inproceedings{devita2024aleatoric,
  title = {{Diffusion Model Guided Sampling with Pixel-Wise Aleatoric Uncertainty Estimation}},
  author = {De Vita, Michele and Belagiannis, Vasileios},
  booktitle = {Proceedings of the Winter Conference on Applications of Computer Vision (WACV)},
  year = {2025},
  pages = {3844--3854},
  url = {https://openaccess.thecvf.com/content/WACV2025/html/De_Vita_Diffusion_Model_Guided_Sampling_with_Pixel-Wise_Aleatoric_Uncertainty_Estimation_WACV_2025_paper.html}
}

@inproceedings{ho2020denoising,
  title = {{Denoising Diffusion Probabilistic Models}},
  author = {Ho, Jonathan and Jain, Ajay and Abbeel, Pieter},
  booktitle = {Advances in Neural Information Processing Systems},
  year = {2020},
  volume = {33},
  pages = {6840--6851},
  publisher = {Curran Associates, Inc.},
  url = {https://proceedings.neurips.cc/paper/2020/hash/4c5bcfec8584af0d967f1ab10179ca4b-Abstract.html}
}

@inproceedings{nichol2021improved,
  title = 	 {Improved Denoising Diffusion Probabilistic Models},
  author =       {Nichol, Alexander Quinn and Dhariwal, Prafulla},
  booktitle = 	 {Proceedings of the 38th International Conference on Machine Learning},
  pages = 	 {8162--8171},
  year = 	 {2021},
  editor = 	 {Meila, Marina and Zhang, Tong},
  volume = 	 {139},
  series = 	 {Proceedings of Machine Learning Research},
  month = 	 {18--24 Jul},
  publisher =    {PMLR},
  url = 	 {https://proceedings.mlr.press/v139/nichol21a.html},
}

@inproceedings{jazbec2025generative,
  title = 	 {Generative Uncertainty in Diffusion Models},
  author =       {Jazbec, Metod and Wong-Toi, Eliot and Xia, Guoxuan and Zhang, Dan and Nalisnick, Eric and Mandt, Stephan},
  booktitle = 	 {Proceedings of the Forty-first Conference on Uncertainty in Artificial Intelligence},
  pages = 	 {1837--1858},
  year = 	 {2025},
  editor = 	 {Chiappa, Silvia and Magliacane, Sara},
  volume = 	 {286},
  series = 	 {Proceedings of Machine Learning Research},
  month = 	 {21--25 Jul},
  publisher =    {PMLR},
  url = 	 {https://proceedings.mlr.press/v286/jazbec25a.html},
}

@inproceedings{ma2024sit,
  title={Sit: Exploring flow and diffusion-based generative models with scalable interpolant transformers},
  author={Ma, Nanye and Goldstein, Mark and Albergo, Michael S and Boffi, Nicholas M and Vanden-Eijnden, Eric and Xie, Saining},
  booktitle={European Conference on Computer Vision},
  pages={23--40},
  year={2024},
  organization={Springer}
}

@inproceedings{
liu2023flow,
title={Flow Straight and Fast: Learning to Generate and Transfer Data with Rectified Flow},
author={Xingchao Liu and Chengyue Gong and qiang liu},
booktitle={The Eleventh International Conference on Learning Representations },
year={2023},
url={https://openreview.net/forum?id=XVjTT1nw5z}
}

@inproceedings{
guo2025variational,
title={Variational Rectified Flow Matching},
author={Pengsheng Guo and Alex Schwing},
booktitle={Forty-second International Conference on Machine Learning},
year={2025},
url={https://openreview.net/forum?id=Rk18ZikrFI}
}

@inproceedings{
albergo2023building,
title={Building Normalizing Flows with Stochastic Interpolants},
author={Michael Samuel Albergo and Eric Vanden-Eijnden},
booktitle={The Eleventh International Conference on Learning Representations },
year={2023},
url={https://openreview.net/forum?id=li7qeBbCR1t}
}

@inproceedings{
esser2024scaling,
title={Scaling Rectified Flow Transformers for High-Resolution Image Synthesis},
author={Patrick Esser and Sumith Kulal and Andreas Blattmann and Rahim Entezari and Jonas M{\"u}ller and Harry Saini and Yam Levi and Dominik Lorenz and Axel Sauer and Frederic Boesel and Dustin Podell and Tim Dockhorn and Zion English and Robin Rombach},
booktitle={Forty-first International Conference on Machine Learning},
year={2024},
url={https://openreview.net/forum?id=FPnUhsQJ5B}
}

@inproceedings{hoogeboom2023simple,
  title={simple diffusion: End-to-end diffusion for high resolution images},
  author={Hoogeboom, Emiel and Heek, Jonathan and Salimans, Tim},
  booktitle={International Conference on Machine Learning},
  pages={13213--13232},
  year={2023},
  organization={PMLR}
}

@inproceedings{pooladian2023multisample,
  title={Multisample Flow Matching: Straightening Flows with Minibatch Couplings},
  author={Pooladian, Aram-Alexandre and Ben-Hamu, Heli and Domingo-Enrich, Carles and Amos, Brandon and Lipman, Yaron and Chen, Ricky TQ},
  booktitle={International Conference on Machine Learning},
  pages={28100--28127},
  year={2023},
  organization={PMLR}
}

@article{ma2025stochastic,
  title={Stochastic Interpolants via Conditional Dependent Coupling},
  author={Ma, Chenrui and Xiao, Xi and Wang, Tianyang and Wang, Xiao and Shen, Yanning},
  journal={arXiv preprint arXiv:2509.23122},
  year={2025},
  url={https://arxiv.org/abs/2509.23122}
}

@inproceedings{albergo2024stochastic,
  title={Stochastic Interpolants with Data-Dependent Couplings},
  author={Albergo, Michael Samuel and Goldstein, Mark and Boffi, Nicholas Matthew and Ranganath, Rajesh and Vanden-Eijnden, Eric},
  booktitle={International Conference on Machine Learning},
  pages={921--937},
  year={2024},
  organization={PMLR}
}

@article{tong2023improving,
  title = {{Improving and Generalizing Flow-Based Generative Models with Minibatch Optimal Transport}},
  author = {Tong, Alexander and Fatras, Kilian and Malkin, Nikolay and Huguet, Guillaume and Zhang, Yanlei and Rector-Brooks, Jarrid and Wolf, Guy and Bengio, Yoshua},
  journal = {Transactions on Machine Learning Research},
  year = {2024},
  url = {https://openreview.net/forum?id=CD9Snc73AW}
}

@inproceedings{
zhang2025towards,
title={Towards Hierarchical Rectified Flow},
author={Yichi Zhang and Yici Yan and Alex Schwing and Zhizhen Zhao},
booktitle={The Thirteenth International Conference on Learning Representations},
year={2025},
url={https://openreview.net/forum?id=6F6qwdycgJ}
}

@article{lee2024improving,
  title={Improving the training of rectified flows},
  author={Lee, Sangyun and Lin, Zinan and Fanti, Giulia},
  journal={Advances in neural information processing systems},
  volume={37},
  pages={63082--63109},
  year={2024}
}

@inproceedings{wang2024rectified,
  title = {{Rectified Diffusion: Straightness Is Not Your Need in Rectified Flow}},
  author = {Wang, Fu-Yun and Yang, Ling and Huang, Zhaoyang and Wang, Mengdi and Li, Hongsheng},
  booktitle = {International Conference on Learning Representations},
  year = {2025},
  pages = {31420--31445},
  url = {https://proceedings.iclr.cc/paper_files/paper/2025/hash/4df9a5e6bad9e64ebcea453e031142bb-Abstract-Conference.html}
}

@article{zhang2025hierarchical,
  title={Hierarchical Rectified Flow Matching with Mini-Batch Couplings},
  author={Zhang, Yichi and Yan, Yici and Schwing, Alex and Zhao, Zhizhen},
  journal={arXiv preprint arXiv:2507.13350},
  year={2025}
}

@article{klein2023equivariant,
  title={Equivariant flow matching},
  author={Klein, Leon and Kr{\"a}mer, Andreas and No{\'e}, Frank},
  journal={Advances in Neural Information Processing Systems},
  volume={36},
  pages={59886--59910},
  year={2023}
}

@inproceedings{kingma2013auto,
  title = {{Auto-Encoding Variational Bayes}},
  author = {Kingma, Diederik P. and Welling, Max},
  booktitle = {International Conference on Learning Representations},
  year = {2014},
  url = {https://arxiv.org/abs/1312.6114}
}

@inproceedings{ma2026cadvae,
  title = {{{CAD}-{VAE}: Leveraging Correlation-Aware Latents for Comprehensive Fair Disentanglement}},
  author = {Ma, Chenrui and Xiao, Xi and Wang, Tianyang and Wang, Xiao and Shen, Yanning},
  booktitle = {Proceedings of the AAAI Conference on Artificial Intelligence},
  year = {2026},
  volume = {40},
  pages = {7744--7754},
  doi = {10.1609/aaai.v40i10.37717},
  url = {https://ojs.aaai.org/index.php/AAAI/article/view/37717}
}

@article{heusel2017gans,
  title={Gans trained by a two time-scale update rule converge to a local nash equilibrium},
  author={Heusel, Martin and Ramsauer, Hubert and Unterthiner, Thomas and Nessler, Bernhard and Hochreiter, Sepp},
  journal={Advances in neural information processing systems},
  volume={30},
  year={2017}
}

@article{ma2025beyond,
  title={Beyond Editing Pairs: Fine-Grained Instructional Image Editing via Multi-Scale Learnable Regions},
  author={Ma, Chenrui and Xiao, Xi and Wang, Tianyang and Shen, Yanning},
  journal={arXiv preprint arXiv:2505.19352},
  year={2025}
}

@inproceedings{
jin2025pyramidal,
title={Pyramidal Flow Matching for Efficient Video Generative Modeling},
author={Yang Jin and Zhicheng Sun and Ningyuan Li and Kun Xu and Kun Xu and Hao Jiang and Nan Zhuang and Quzhe Huang and Yang Song and Yadong MU and Zhouchen Lin},
booktitle={The Thirteenth International Conference on Learning Representations},
year={2025},
url={https://openreview.net/forum?id=66NzcRQuOq}
}

@article{kornilov2024optimal,
  title={Optimal flow matching: Learning straight trajectories in just one step},
  author={Kornilov, Nikita and Mokrov, Petr and Gasnikov, Alexander and Korotin, Aleksandr},
  journal={Advances in Neural Information Processing Systems},
  volume={37},
  pages={104180--104204},
  year={2024}
}

@inproceedings{cai2025diffusion,
  title={Diffusion self-distillation for zero-shot customized image generation},
  author={Cai, Shengqu and Chan, Eric Ryan and Zhang, Yunzhi and Guibas, Leonidas and Wu, Jiajun and Wetzstein, Gordon},
  booktitle={Proceedings of the Computer Vision and Pattern Recognition Conference},
  pages={18434--18443},
  year={2025}
}

@article{karras2022elucidating,
  title={Elucidating the design space of diffusion-based generative models},
  author={Karras, Tero and Aittala, Miika and Aila, Timo and Laine, Samuli},
  journal={Advances in neural information processing systems},
  volume={35},
  pages={26565--26577},
  year={2022}
}

@inproceedings{ma2025learningstraightflowsvariational,
  title = {{Learning Straight Flows: Variational Flow Matching for Efficient Generation}},
  author = {Ma, Chenrui and Xiao, Xi and Wang, Tianyang and Wang, Xiao and Shen, Yanning},
  booktitle = {Proceedings of the IEEE/CVF Conference on Computer Vision and Pattern Recognition (CVPR)},
  year = {2026},
  pages = {38154--38164},
  url = {https://openaccess.thecvf.com/content/CVPR2026/html/Ma_Learning_Straight_Flows_Variational_Flow_Matching_for_Efficient_Generation_CVPR_2026_paper.html}
}

@misc{lai2025principlesdiffusionmodels,
      title={The Principles of Diffusion Models}, 
      author={Chieh-Hsin Lai and Yang Song and Dongjun Kim and Yuki Mitsufuji and Stefano Ermon},
      year={2025},
      eprint={2510.21890},
      archivePrefix={arXiv},
      primaryClass={cs.LG},
      url={https://arxiv.org/abs/2510.21890}, 
}

@misc{lipman2024flowmatchingguidecode,
      title={Flow Matching Guide and Code}, 
      author={Yaron Lipman and Marton Havasi and Peter Holderrieth and Neta Shaul and Matt Le and Brian Karrer and Ricky T. Q. Chen and David Lopez-Paz and Heli Ben-Hamu and Itai Gat},
      year={2024},
      eprint={2412.06264},
      archivePrefix={arXiv},
      primaryClass={cs.LG},
      url={https://arxiv.org/abs/2412.06264}, 
}

@techreport{Krizhevsky2009LearningML,
  title = {{Learning Multiple Layers of Features from Tiny Images}},
  author = {Krizhevsky, Alex},
  institution = {University of Toronto},
  year = {2009},
  url = {https://www.cs.toronto.edu/~kriz/learning-features-2009-TR.pdf}
}

@inproceedings{rombach2022highresolutionimagesynthesislatent,
  title = {{High-Resolution Image Synthesis With Latent Diffusion Models}},
  author = {Rombach, Robin and Blattmann, Andreas and Lorenz, Dominik and Esser, Patrick and Ommer, Bj\"orn},
  booktitle = {Proceedings of the IEEE/CVF Conference on Computer Vision and Pattern Recognition (CVPR)},
  year = {2022},
  pages = {10684--10695},
  url = {https://openaccess.thecvf.com/content/CVPR2022/html/Rombach_High-Resolution_Image_Synthesis_With_Latent_Diffusion_Models_CVPR_2022_paper.html}
}

@inproceedings{dosovitskiy2021imageworth16x16words,
  title = {{An Image is Worth 16x16 Words: Transformers for Image Recognition at Scale}},
  author = {Alexey Dosovitskiy and Lucas Beyer and Alexander Kolesnikov and Dirk Weissenborn and Xiaohua Zhai and Thomas Unterthiner and Mostafa Dehghani and Matthias Minderer and Georg Heigold and Sylvain Gelly and Jakob Uszkoreit and Neil Houlsby},
  booktitle = {International Conference on Learning Representations},
  year = {2021},
  url = {https://openreview.net/forum?id=YicbFdNTTy}
}

@article{ma2026transition,
  title={Transition Flow Matching},
  author={Ma, Chenrui},
  journal={arXiv preprint arXiv:2603.15689},
  year={2026}
}

@inproceedings{yu2024representation,
  title = {{Representation Alignment for Generation: Training Diffusion Transformers Is Easier Than You Think}},
  author = {Yu, Sihyun and Kwak, Sangkyung and Jang, Huiwon and Jeong, Jongheon and Huang, Jonathan and Shin, Jinwoo and Xie, Saining},
  booktitle = {International Conference on Learning Representations},
  year = {2025},
  pages = {87400--87442},
  url = {https://proceedings.iclr.cc/paper_files/paper/2025/hash/d9e42b4d7163931f3689d6d6fbaa11d0-Abstract-Conference.html}
}

@inproceedings{zheng2025diffusion,
  title = {{Diffusion Transformers with Representation Autoencoders}},
  author = {Zheng, Boyang and Ma, Nanye and Tong, Shengbang and Xie, Saining},
  booktitle = {International Conference on Learning Representations},
  year = {2026},
  pages = {35791--35820},
  url = {https://proceedings.iclr.cc/paper_files/paper/2026/hash/3c4141c12660ad3625eb4ae845e0a6f9-Abstract-Conference.html}
}

@article{salimans2016improved,
  title={Improved techniques for training gans},
  author={Salimans, Tim and Goodfellow, Ian and Zaremba, Wojciech and Cheung, Vicki and Radford, Alec and Chen, Xi},
  journal={Advances in neural information processing systems},
  volume={29},
  year={2016}
}

@article{chi2025diffusion,
  title={Diffusion policy: Visuomotor policy learning via action diffusion},
  author={Chi, Cheng and Xu, Zhenjia and Feng, Siyuan and Cousineau, Eric and Du, Yilun and Burchfiel, Benjamin and Tedrake, Russ and Song, Shuran},
  journal={The International Journal of Robotics Research},
  volume={44},
  number={10-11},
  pages={1684--1704},
  year={2025},
  publisher={Sage Publications Sage UK: London, England}
}

@inproceedings{florence2022implicit,
  title={Implicit behavioral cloning},
  author={Florence, Pete and Lynch, Corey and Zeng, Andy and Ramirez, Oscar A and Wahid, Ayzaan and Downs, Laura and Wong, Adrian and Lee, Johnny and Mordatch, Igor and Tompson, Jonathan},
  booktitle={Conference on robot learning},
  pages={158--168},
  year={2022},
  organization={PMLR}
}

@article{rumelhart1986learning,
  title={Learning representations by back-propagating errors},
  author={Rumelhart, David E and Hinton, Geoffrey E and Williams, Ronald J},
  journal={nature},
  volume={323},
  number={6088},
  pages={533--536},
  year={1986},
  publisher={Nature Publishing Group UK London}
}

@inproceedings{li2025backtobasics,
  title = {{Back to Basics: Let Denoising Generative Models Denoise}},
  author = {Li, Tianhong and He, Kaiming},
  booktitle = {Proceedings of the IEEE/CVF Conference on Computer Vision and Pattern Recognition (CVPR)},
  year = {2026},
  pages = {36115--36125},
  url = {https://openaccess.thecvf.com/content/CVPR2026/html/Li_Back_to_Basics_Let_Denoising_Generative_Models_Denoise_CVPR_2026_paper.html}
}

@inproceedings{luo2021diffusion,
  title = {Diffusion Probabilistic Models for {3D} Point Cloud Generation},
  author = {Luo, Shitong and Hu, Wei},
  booktitle = {Proceedings of the IEEE/CVF Conference on Computer Vision and Pattern Recognition},
  pages = {2837--2845},
  year = {2021},
  url = {https://openaccess.thecvf.com/content/CVPR2021/html/Luo_Diffusion_Probabilistic_Models_for_3D_Point_Cloud_Generation_CVPR_2021_paper.html}
}

@inproceedings{tholke2022torchmdnet,
  title = {{TorchMD-NET}: Equivariant Transformers for Neural Network based Molecular Potentials},
  author = {Th{\"o}lke, Philipp and De Fabritiis, Gianni},
  booktitle = {International Conference on Learning Representations},
  year = {2022},
  url = {https://arxiv.org/abs/2202.02541}
}

@inproceedings{hassan2024etflow,
  title = {{ET-Flow: Equivariant Flow-Matching for Molecular Conformer Generation}},
  author = {Hassan, Majdi and Shenoy, Nikhil and Lee, Jungyoon and St\"{a}rk, Hannes and Thaler, Stephan and Beaini, Dominique},
  booktitle = {Advances in Neural Information Processing Systems},
  year = {2024},
  volume = {37},
  pages = {128798--128824},
  publisher = {Curran Associates, Inc.},
  doi = {10.52202/079017-4091},
  url = {https://proceedings.neurips.cc/paper_files/paper/2024/hash/e8bd617e7dd0394ceadf37b4a7773179-Abstract-Conference.html}
}

@inproceedings{barber2003variational,
  title = {The {IM} Algorithm: A Variational Approach to Information Maximization},
  author = {Barber, David and Agakov, Felix V.},
  booktitle = {Advances in Neural Information Processing Systems},
  volume = {16},
  year = {2003},
  url = {https://papers.nips.cc/paper_files/paper/2003/file/a6ea8471c120fe8cc35a2954c9b9c595-Paper.pdf}
}

@inproceedings{poole2019variational,
  title = {On Variational Bounds of Mutual Information},
  author = {Poole, Ben and Ozair, Sherjil and van den Oord, Aaron and Alemi, Alex and Tucker, George},
  booktitle = {Proceedings of the 36th International Conference on Machine Learning},
  volume = {97},
  pages = {5171--5180},
  publisher = {PMLR},
  year = {2019},
  url = {https://proceedings.mlr.press/v97/poole19a.html}
}

@inproceedings{mohamed2015variational,
  title = {Variational Information Maximisation for Intrinsically Motivated Reinforcement Learning},
  author = {Mohamed, Shakir and Rezende, Danilo Jimenez},
  booktitle = {Advances in Neural Information Processing Systems},
  volume = {28},
  year = {2015},
  url = {https://arxiv.org/abs/1509.08731}
}

@article{fontbona2016trajectorial,
  title = {A trajectorial interpretation of the dissipations of entropy and {Fisher} information for stochastic differential equations},
  author = {Fontbona, Joaquin and Jourdain, Benjamin},
  journal = {The Annals of Probability},
  volume = {44},
  number = {1},
  pages = {131--170},
  year = {2016},
  doi = {10.1214/14-AOP969},
  url = {https://arxiv.org/abs/1107.3300}
}

@inproceedings{wibisono2017information,
  title = {Information and estimation in {Fokker--Planck} channels},
  author = {Wibisono, Andre and Jog, Varun and Loh, Po-Ling},
  booktitle = {IEEE International Symposium on Information Theory},
  year = {2017},
  url = {https://arxiv.org/abs/1702.03656}
}

@inproceedings{deng2009imagenet,
  title = {{ImageNet}: A Large-Scale Hierarchical Image Database},
  author = {Deng, Jia and Dong, Wei and Socher, Richard and Li, Li-Jia and Li, Kai and Fei-Fei, Li},
  booktitle = {IEEE Conference on Computer Vision and Pattern Recognition},
  pages = {248--255},
  year = {2009},
  doi = {10.1109/CVPR.2009.5206848},
  url = {https://www.image-net.org/static_files/papers/imagenet_cvpr09.pdf}
}

@article{russakovsky2015imagenet,
  title = {{ImageNet} Large Scale Visual Recognition Challenge},
  author = {Russakovsky, Olga and Deng, Jia and Su, Hao and Krause, Jonathan and Satheesh, Sanjeev and Ma, Sean and Huang, Zhiheng and Karpathy, Andrej and Khosla, Aditya and Bernstein, Michael and Berg, Alexander C. and Fei-Fei, Li},
  journal = {International Journal of Computer Vision},
  volume = {115},
  number = {3},
  pages = {211--252},
  year = {2015},
  doi = {10.1007/s11263-015-0816-y},
  url = {https://arxiv.org/abs/1409.0575}
}

@misc{chang2015shapenet,
  title = {{ShapeNet}: An Information-Rich {3D} Model Repository},
  author = {Chang, Angel X. and Funkhouser, Thomas and Guibas, Leonidas and Hanrahan, Pat and Huang, Qixing and Li, Zimo and Savarese, Silvio and Savva, Manolis and Song, Shuran and Su, Hao and Xiao, Jianxiong and Yi, Li and Yu, Fisher},
  year = {2015},
  eprint = {1512.03012},
  archivePrefix = {arXiv},
  url = {https://arxiv.org/abs/1512.03012}
}

@article{axelrod2022geom,
  title = {{GEOM}, energy-annotated molecular conformations for property prediction and molecular generation},
  author = {Axelrod, Simon and G{\'o}mez-Bombarelli, Rafael},
  journal = {Scientific Data},
  volume = {9},
  pages = {185},
  year = {2022},
  doi = {10.1038/s41597-022-01288-4},
  url = {https://www.nature.com/articles/s41597-022-01288-4}
}

@article{ramakrishnan2014quantum,
  title = {Quantum chemistry structures and properties of 134 kilo molecules},
  author = {Ramakrishnan, Raghunathan and Dral, Pavlo O. and Rupp, Matthias and von Lilienfeld, O. Anatole},
  journal = {Scientific Data},
  volume = {1},
  pages = {140022},
  year = {2014},
  doi = {10.1038/sdata.2014.22},
  url = {https://www.nature.com/articles/sdata201422}
}

@inproceedings{achlioptas2018learning,
  title = {Learning Representations and Generative Models for {3D} Point Clouds},
  author = {Achlioptas, Panos and Diamanti, Olga and Mitliagkas, Ioannis and Guibas, Leonidas},
  booktitle = {Proceedings of the 35th International Conference on Machine Learning},
  volume = {80},
  pages = {40--49},
  publisher = {PMLR},
  year = {2018},
  url = {https://proceedings.mlr.press/v80/achlioptas18a.html}
}

@inproceedings{yang2019pointflow,
  title = {{PointFlow}: {3D} Point Cloud Generation with Continuous Normalizing Flows},
  author = {Yang, Guandao and Huang, Xun and Hao, Zekun and Liu, Ming-Yu and Belongie, Serge and Hariharan, Bharath},
  booktitle = {Proceedings of the IEEE/CVF International Conference on Computer Vision},
  year = {2019},
  url = {https://arxiv.org/abs/1906.12320}
}

@inproceedings{qi2017pointnet,
  title = {{PointNet}: Deep Learning on Point Sets for {3D} Classification and Segmentation},
  author = {Qi, Charles R. and Su, Hao and Mo, Kaichun and Guibas, Leonidas J.},
  booktitle = {Proceedings of the IEEE Conference on Computer Vision and Pattern Recognition},
  pages = {652--660},
  year = {2017},
  url = {https://openaccess.thecvf.com/content_cvpr_2017/html/Qi_PointNet_Deep_Learning_CVPR_2017_paper.html}
}

@inproceedings{kingma2015adam,
  title = {{Adam}: A Method for Stochastic Optimization},
  author = {Kingma, Diederik P. and Ba, Jimmy},
  booktitle = {International Conference on Learning Representations},
  year = {2015},
  url = {https://arxiv.org/abs/1412.6980}
}

@inproceedings{loshchilov2019decoupled,
  title = {Decoupled Weight Decay Regularization},
  author = {Loshchilov, Ilya and Hutter, Frank},
  booktitle = {International Conference on Learning Representations},
  year = {2019},
  url = {https://openreview.net/forum?id=Bkg6RiCqY7}
}

@article{ma2026drift,
  title = {{Drift Flow Matching}},
  author = {Ma, Chenrui and Xiao, Xi and Zhao, Lin and Wang, Tianyang and Fioretto, Ferdinando and Shen, Yanning},
  journal = {arXiv preprint arXiv:2605.17244},
  year = {2026},
  url = {https://arxiv.org/abs/2605.17244}
}

@inproceedings{xiao2026directions,
  title = {{Not All Directions Matter: Towards Structured and Task-Aware Low-Rank Model Adaptation}},
  author = {Xiao, Xi and Ma, Chenrui and Zhang, Yunbei and Liu, Chen and Wang, Zhuxuanzi and Li, Yanshu and Zhao, Lin and Hu, Guosheng and Wang, Tianyang and Xu, Hao},
  booktitle = {Proceedings of the 64th Annual Meeting of the Association for Computational Linguistics (Volume 1: Long Papers)},
  year = {2026},
  pages = {2132--2154},
  publisher = {Association for Computational Linguistics},
  doi = {10.18653/v1/2026.acl-long.97},
  url = {https://aclanthology.org/2026.acl-long.97/}
}
\bibliographystyle{iclr2027_conference}

\clearpage
% Retain the existing float-page settings for the appendix.
\let\floatpagefraction\pfmSavedFloatPageFraction
\appendix
\startcontents[appendix]
\noindent\rule{\textwidth}{0.8pt}
\begingroup
\small
\setlength{\parskip}{0pt}
\hypersetup{hidelinks}
\titlecontents{lsection}[1.5em]
  {\addvspace{3pt}\small\bfseries}{\contentslabel{1.5em}}{}
  {\titlerule*[0.75em]{.}\contentspage}
\titlecontents{lsubsection}[3.8em]
  {\small}{\contentslabel{2.3em}}{}
  {\titlerule*[0.75em]{.}\contentspage}
\printcontents[appendix]{l}{1}{\setcounter{tocdepth}{2}\section*{Appendix Contents}}
\endgroup
\noindent\rule{\textwidth}{0.8pt}
\clearpage

\begin{figure}[!htbp]
\centering
\makebox[\linewidth][c]{%
  \includegraphics[trim=20bp 8bp 20bp 8bp,clip,width=1.05\linewidth]{figure/image/method_overview.pdf}%
}
\par\vspace{-1em}
\begingroup
\pfmMethodApplicationstrue
\caption[Posterior Flow Matching: full overview.]{\input{figure/method_caption}}
\label{fig:pfm-method-full}
\endgroup
\vspace{-1em}
\end{figure}

\section{Notation and Shared Foundations}
\label{app:foundations}

This appendix establishes the foundations used in Section~\ref{sec:preliminaries} and the generator comparison in Section~\ref{sec:sampling}. For FM and PFM, independent endpoints $X_0\sim\mathcal N(0,I)$ and $X_1\sim p_{\mathrm{data}}$ define $X_t=(1-t)X_0+tX_1$ and the velocity target $\pfmVelocity{V}=X_1-X_0$. We first identify its conditional mean as the transport field. We then derive the diffusion score and endpoint readouts, establish the shared ODE--SDE correspondence, and recover the baseline entries of Table~\ref{tab:generative-dynamics}.

\subsection{Notation reference}
\label{app:notation}

Table~\ref{tab:notation} collects the shared symbols. Parameter subscripts identify the model, and $\pfmVelocity{\mu}$ denotes an endpoint-mean readout. Barred processes, densities, and scores refer to diffusion noising time. Appendix~\ref{app:conditional} introduces the external generation condition, and Appendix~\ref{app:endpoint} changes the directly predicted mean coordinate.
\begin{table}[!ht]
\centering
% Keep the compact caption clear of the table rule.
\setlength{\belowcaptionskip}{1em}
\par\vspace{-1em}
\caption{Core notation used in the main text. Local notation is defined in the relevant appendix. Here $d$ is the state dimension and $I$ the identity matrix.}
\label{tab:notation}
\vspace{-1em}
\scriptsize
\begin{tabular}{@{}p{0.27\linewidth}p{0.68\linewidth}@{}}
\toprule
Symbol & Meaning \\
\midrule
$X_0,X_1;\ Z$ & Noise and data endpoints; their pair defining a reference path. \\
$X_t;\ Y_t;\ \bar X_\tau$ & FM/PFM reference state; generation process; diffusion noising process. \\
$t;\ \tau=1-t;\ C$ & Generation time, noising time, and external generation condition. \\
$p_0;\ p_1=p_{\mathrm{data}}$ & Source and data distributions. \\
$p_t;\ \bar p_\tau$ & FM/PFM reference density; diffusion noising density. \\
\addlinespace
$\pfmVelocity{V};\ \pfmVelocity{v}$ & Sampled velocity target; its conditional mean. \\
$\eta;\phi;\theta$ & Parameters of FM, diffusion, and PFM. \\
$\pfmVelocity{v_\eta};\ \pfmVelocity{v_\theta}$ & FM velocity predictor; PFM velocity mean. \\
$q_\theta;\ q_\theta^x$ & Velocity and endpoint posterior models. \\
$F_\theta$ & Network predicting velocity mean and log variance. \\
$\pfmDiffusion{\ell_\theta};\pfmDiffusion{\gamma_\theta};\pfmDiffusion{\Sigma_\theta}$ & Velocity log-variance vector, componentwise variance vector, and diagonal covariance. \\
$\mathcal L_{\mathrm{PFM}};\ \mathcal L_{\mathrm{CFM}}$ & PFM Gaussian negative log-likelihood (NLL) objective; CFM regression objective. \\
$H;\ I(\cdot;\cdot)$ & Differential entropy; mutual information (conditional when given $C$). \\
$r$ & Training time law. \\
\addlinespace
$s;\ \bar s;\ s_\phi$ & FM/PFM score; diffusion noising score; predictor of $\bar s(x,1-t)$ in generation time. \\
$\widehat s_\eta;\widehat s_\theta$ & Scores derived from the corresponding velocity predictions. \\
$b;\ b_\theta$ & General SDE drift; PFM drift computed from posterior outputs. \\
$\pfmDiffusion{B};\ \pfmDiffusion{a}=\pfmDiffusion{B}\pfmDiffusion{B}^\top$ & Brownian noise coefficient; diffusion tensor (covariance rate). \\
$\pfmDiffusion{B_\theta};\ \pfmDiffusion{a_\theta}=\lambda\pfmDiffusion{\Sigma_\theta}$ & PFM Brownian noise coefficient; posterior-driven diffusion tensor. \\
$f;\ \pfmDiffusion{g};\ W$ & Prescribed noising drift, scalar noise coefficient, and standard Brownian motion. \\
$\pfmDiffusion{\kappa}(t);\lambda(t)$ & FM isotropic covariance rate; PFM posterior-covariance scale. \\
$\widetilde p_t$ & Density evolved from a different initial law under the same exact compatible dynamics as $p_t$. \\
$T$ & Terminal integration time. \\
\addlinespace
$\pfmVelocity{\mu_\eta};\ \pfmVelocity{\mu_\phi};\ \pfmVelocity{\mu_\theta}$ & Endpoint-mean readouts of FM, diffusion, and PFM. \\
$\pfmDiffusion{\Sigma_\theta}^x$ & PFM endpoint-posterior diagonal covariance. \\
$\alpha_t;\sigma_t$ & Diffusion kernel signal scale and cumulative noise standard deviation. \\
\bottomrule
\end{tabular}
\par\vspace{-1em}
\end{table}

\subsection{Conditional-to-marginal flow matching}
\label{app:fm}

We prove the continuity identity in Eq.~\ref{eq:marginal-velocity} and the gradient equivalence stated after Eq.~\ref{eq:cfm-loss}~\citep{lipman2023flow,albergo2023building,lipman2024flowmatchingguidecode}. For a smooth compactly supported test function $\psi$, the reference interpolation gives
\begin{equation*}
 \frac{d}{dt}\mathbb E[\psi(X_t)]
 =\mathbb E[\nabla\psi(X_t)^\top \pfmVelocity{V}]
 =\mathbb E[\nabla\psi(X_t)^\top \pfmVelocity{v}(X_t,t)].
\end{equation*}
Finite endpoint second moments imply integrability of $\pfmVelocity{V}$. Together with the bounded gradient of $\psi$, this justifies differentiating along the path. This is the weak form of the continuity equation. A well-defined ODE flow with unique density evolution and initialization $p_0$ follows these marginals.

The marginal regression objective corresponding to CFM is
\begin{equation}
 \mathcal L_{\mathrm{FM}}(\eta)
 =\frac12\mathbb E_{t,X_t}\|\pfmVelocity{v_\eta}(X_t,t)-\pfmVelocity{v}(X_t,t)\|^2.
 \label{eq:marginal-fm-loss}
\end{equation}
Both losses use the same parameter-independent time law $r$ and reference interpolation. Write $D_\eta \pfmVelocity{v_\eta}$ for the parameter Jacobian, with output coordinates indexing rows. Under finite risks and conditions allowing differentiation under expectation,
\begin{equation}
 \begin{aligned}
 \nabla_\eta\mathcal L_{\mathrm{CFM}}
 &=\mathbb E[(D_\eta \pfmVelocity{v_\eta})^\top(\pfmVelocity{v_\eta}-\pfmVelocity{V})]\\
 &=\mathbb E[(D_\eta \pfmVelocity{v_\eta})^\top(\pfmVelocity{v_\eta}-\mathbb E[\pfmVelocity{V}\mid X_t,t])]\\
 &=\mathbb E[(D_\eta \pfmVelocity{v_\eta})^\top(\pfmVelocity{v_\eta}-\pfmVelocity{v})]
 =\nabla_\eta\mathcal L_{\mathrm{FM}}.
 \end{aligned}
 \label{eq:fm-gradient-proof}
\end{equation}
All network quantities are evaluated at $(X_t,t)$. Conditioning on these inputs leaves the prediction and its parameter Jacobian fixed, giving the second equality.

Expanding $\pfmVelocity{V}-\pfmVelocity{v_\eta}=(\pfmVelocity{V}-\pfmVelocity{v})+(\pfmVelocity{v}-\pfmVelocity{v_\eta})$ gives a second view of the same result:
\begin{equation}
 \mathbb E\|\pfmVelocity{V}-\pfmVelocity{v_\eta}(X_t,t)\|^2
 =\mathbb E\|\pfmVelocity{V}-\pfmVelocity{v}(X_t,t)\|^2
 +\mathbb E\|\pfmVelocity{v}(X_t,t)-\pfmVelocity{v_\eta}(X_t,t)\|^2.
 \label{eq:fm-decomposition}
\end{equation}
The cross term vanishes because $\mathbb E[\pfmVelocity{V}-\pfmVelocity{v}\mid X_t,t]=0$. The first term is independent of the predictor, and the second is minimized by the conditional mean. Equations~\ref{eq:fm-gradient-proof} and~\ref{eq:fm-decomposition} justify the sample-based training in Section~\ref{sec:reference-path} and identify the velocity target recovered by Eq.~\ref{eq:fm-recovery}.

\subsection{Diffusion training and conditional scores}
\label{app:diffusion-training}

For the noising process in Section~\ref{sec:diffusion}, pair the data endpoint $X_1$ with $\bar X_0=X_1$ and draw $\bar X_{1-t}$ from its perturbation kernel $\bar p_{1-t}(\cdot\mid X_1)$. Denoising score matching uses~\citep{song2021score}
\begin{equation}
 \mathcal L_{\mathrm{score}}(\phi)=\frac12\mathbb E_{t,X_1,\bar X_{1-t}}
 \left[w(t)\left\|s_\phi(\bar X_{1-t},t)
 -\left.\nabla_x\log\bar p_{1-t}(x\mid X_1)\right|_{x=\bar X_{1-t}}\right\|^2\right],
 \label{eq:score-loss}
\end{equation}
where $X_1\sim p_{\mathrm{data}}$, the time law is fixed, and $w(t)>0$ is a prescribed weight. A Gaussian kernel with signal scale $\alpha_t$ and cumulative noise standard deviation $\sigma_t>0$ gives
\begin{equation}
 \bar p_{1-t}(x\mid x_1)=\mathcal N(x;\alpha_t x_1,\sigma_t^2I),\qquad
 \nabla_x\log\bar p_{1-t}(x\mid x_1)=-\frac{x-\alpha_t x_1}{\sigma_t^2}.
 \label{eq:diffusion-kernel-score}
\end{equation}
The cumulative perturbation scale $\sigma_t$ and instantaneous noise coefficient $\pfmDiffusion{g}(1-t)$ have distinct roles.

For a positive marginal density, differentiation under the mixture integral and Bayes' rule give
\begin{equation}
 \begin{aligned}
 \bar s(x,1-t)
 &=\frac{\int\nabla_x\bar p_{1-t}(x\mid x_1)\,p_1(dx_1)}{\bar p_{1-t}(x)}\\
 &=\mathbb E[\nabla_x\log\bar p_{1-t}(x\mid X_1)\mid\bar X_{1-t}=x].
 \end{aligned}
 \label{eq:score-averaging}
\end{equation}
Assuming these derivatives and expectations exist and the squared risks are finite, the conditional-mean decomposition in Eq.~\ref{eq:fm-decomposition} identifies the unrestricted minimizer as $s_\phi(x,t)=\bar s(x,1-t)$ wherever the training time law has support.

For the Gaussian kernel, substituting Eq.~\ref{eq:diffusion-kernel-score} into Eq.~\ref{eq:score-averaging} yields
\begin{equation}
 \bar s(x,1-t)=\frac{\alpha_t\mathbb E[X_1\mid\bar X_{1-t}=x]-x}{\sigma_t^2},\qquad
 \mathbb E[X_1\mid\bar X_{1-t}=x]=\frac{x+\sigma_t^2\bar s(x,1-t)}{\alpha_t},
 \label{eq:diffusion-endpoint-mean}
\end{equation}
with the second identity requiring $\alpha_t\ne0$. This relation supplies the diffusion endpoint readout used in Table~\ref{tab:generative-dynamics}. The baseline comparison in Appendix~\ref{app:baseline-sampling} collects it together with the FM readout after constructing both samplers.

\subsection{The Gaussian velocity--score identity}
\label{app:score}

We derive Eq.~\ref{eq:velocity-score}, which supplies the scores used by FM and PFM. Return to the independent Gaussian linear interpolation of Section~\ref{sec:reference-path}~\citep{lipman2024flowmatchingguidecode}. Its density is
\begin{equation*}
 p_t(x)=\int\mathcal N(x;tx_1,(1-t)^2I)\,p_1(dx_1).
\end{equation*}
Applying the mixture calculation in Eq.~\ref{eq:score-averaging} gives
\begin{equation}
 s(x,t)=\mathbb E\left[-\frac{x-tX_1}{(1-t)^2}\,\middle|\,X_t=x\right]
 =\frac{t\mathbb E[X_1\mid X_t=x]-x}{(1-t)^2}.
 \label{eq:endpoint-score-exact}
\end{equation}
Taking conditional expectations in Eq.~\ref{eq:endpoint-map} yields $\mathbb E[X_1\mid X_t=x]=x+(1-t)\pfmVelocity{v}(x,t)$. Substitution into Eq.~\ref{eq:endpoint-score-exact} gives
\begin{equation*}
 s(x,t)=\frac{t[x+(1-t)\pfmVelocity{v}(x,t)]-x}{(1-t)^2}
       =\frac{t\pfmVelocity{v}(x,t)-x}{1-t},
\end{equation*}
proving Eq.~\ref{eq:velocity-score}. The corresponding learned score obeys
\begin{equation}
 \widehat s_\theta-s=\frac{t}{1-t}(\pfmVelocity{v_\theta}-\pfmVelocity{v}),
 \label{eq:score-error}
\end{equation}
and the same relation holds with $\eta$ for FM. These derived scores enter the FM SDE in Eq.~\ref{eq:fm-sde-expanded} and the PFM SDE in Eq.~\ref{eq:pfm-sde}.

The same calculation at a fixed external condition gives the score identity in Appendix~\ref{app:conditional-dynamics}.

\subsection{Fokker--Planck correspondence and marginal compatibility}
\label{app:compatibility}

We expand Eq.~\ref{eq:correspondence} in Section~\ref{sec:correspondence}~\citep{song2021score,albergo2025stochastic}. For $dY_t=b\,dt+\pfmDiffusion{B}\,dW_t$, the tensor $\pfmDiffusion{a}=\pfmDiffusion{B}\pfmDiffusion{B}^\top$ is symmetric positive semidefinite. Brownian increments have covariance $hI$, so, under local regularity,
\begin{equation*}
 \mathbb E[Y_{t+h}-Y_t\mid Y_t=x]=h\,b(x,t)+o(h),\qquad
 \operatorname{Cov}(Y_{t+h}-Y_t\mid Y_t=x)=h\,\pfmDiffusion{a}(x,t)+o(h).
\end{equation*}
The density equation is
\begin{equation}
 \partial_t p_t=-\nabla_x\!\cdot(bp_t)
 +\frac12\sum_{i,j}\partial_{x_i}\partial_{x_j}(\pfmDiffusion{a_{ij}}p_t).
 \label{eq:fokker-planck}
\end{equation}
For $s=\nabla_x\log p_t$, the product rule gives
\begin{equation}
 \sum_j\partial_{x_j}(\pfmDiffusion{a_{ij}}p_t)
 =\left[\sum_j\partial_{x_j}\pfmDiffusion{a_{ij}}+\sum_j \pfmDiffusion{a_{ij}}s_j\right]p_t
 =[(\nabla_x\!\cdot \pfmDiffusion{a})+\pfmDiffusion{a}\,s]_i p_t.
 \label{eq:diffusion-product-rule}
\end{equation}
Consequently, Eq.~\ref{eq:fokker-planck} becomes
\begin{equation*}
 \partial_t p_t=-\nabla_x\!\cdot\left[
 \left(b-\tfrac12\nabla_x\!\cdot \pfmDiffusion{a}-\tfrac12\pfmDiffusion{a}\,s\right)p_t\right].
\end{equation*}
The bracket is the probability-current velocity $\pfmVelocity{v}$ in Eq.~\ref{eq:correspondence}. Solving for $b$ gives the converse construction from an ODE to an SDE.

For the compatibility claim, let a positive density path solve the continuity equation with velocity $\pfmVelocity{v}$ on $[0,T]$, $T<1$. Choose $\pfmDiffusion{a}=\pfmDiffusion{B}\pfmDiffusion{B}^\top$ and $b=\pfmVelocity{v}+(\nabla_x\!\cdot \pfmDiffusion{a}+\pfmDiffusion{a}\nabla_x\log p_t)/2$. Assume sufficient coefficient and density regularity for the displayed derivatives, well-posed ODE and SDE dynamics, and unique density evolution with the specified initial law. Decay at infinity or boundary conditions eliminating boundary fluxes complete the density calculation. Substituting Eq.~\ref{eq:diffusion-product-rule} yields
\begin{equation*}
 -\nabla_x\!\cdot(bp_t)
 +\frac12\sum_{i,j}\partial_{x_i}\partial_{x_j}(\pfmDiffusion{a_{ij}}p_t)
 =-\nabla_x\!\cdot(\pfmVelocity{v}p_t).
\end{equation*}
Thus the reference density solves both equations. Matching initialization and uniqueness then identify the SDE marginals with $p_t$. This proves the correspondence used in Section~\ref{sec:correspondence}. It concerns one-time marginals, while the processes can have different joint laws across time.

\subsection{Diffusion and FM as sampling special cases}
\label{app:baseline-sampling}

We now apply the shared density construction to the diffusion and FM baselines. The resulting dynamics and endpoint means justify the comparison with PFM in Section~\ref{sec:sampling} and Table~\ref{tab:generative-dynamics}. Diffusion uses the prescribed noising fields $f,\pfmDiffusion{g}$ and score predictor $s_\phi$. FM uses the velocity predictor $\pfmVelocity{v_\eta}$ on the independent Gaussian reference path.

The noising density $\bar p_\tau$ satisfies
\begin{equation*}
 \partial_\tau\bar p_\tau=-\nabla_x\!\cdot(f\bar p_\tau)
 +\tfrac12\pfmDiffusion{g}(\tau)^2\Delta_x\bar p_\tau.
\end{equation*}
Reversing time and using $\Delta_x\bar p_{1-t}=\nabla_x\!\cdot(\bar s(x,1-t)\bar p_{1-t})$ gives
\begin{equation*}
 \partial_t\bar p_{1-t}
 =\nabla_x\!\cdot(f\bar p_{1-t})-\tfrac12\pfmDiffusion{g}^2\Delta_x\bar p_{1-t}
 =-\nabla_x\!\cdot[(-f+\tfrac12\pfmDiffusion{g}^2\bar s)\bar p_{1-t}],
\end{equation*}
where $f,\pfmDiffusion{g},\bar s$ are evaluated at noising time $1-t$. Applying Appendix~\ref{app:compatibility} to this density path gives the probability-flow ODE~\citep{song2021score}. Choosing $\pfmDiffusion{a}=\pfmDiffusion{g}^2I$ adds $\pfmDiffusion{g}^2\bar s/2$ to its velocity, yielding the reverse SDE. With the exact score and the terminal noising law as initialization, these dynamics reproduce $\bar p_{1-t}$. Replacing $\bar s(x,1-t)$ by $s_\phi(x,t)$ gives the diffusion rows of Table~\ref{tab:generative-dynamics}.

For FM, Appendix~\ref{app:score} supplies $\widehat s_\eta=(t\pfmVelocity{v_\eta}-x)/(1-t)$. Choose $\pfmDiffusion{a}=\pfmDiffusion{\kappa}(t)I$ with $\pfmDiffusion{\kappa}(t)\geq0$~\citep{albergo2025stochastic,lipman2024flowmatchingguidecode}. Then $\nabla_x\!\cdot \pfmDiffusion{a}=0$ and $\pfmDiffusion{B}=\sqrt{\pfmDiffusion{\kappa}} I$, so
\begin{align}
 dY_t&=\pfmVelocity{v_\eta}(Y_t,t)dt,\nonumber\\
 dY_t&=[\pfmVelocity{v_\eta}+\tfrac12\pfmDiffusion{\kappa}\widehat s_\eta](Y_t,t)dt
       +\sqrt{\pfmDiffusion{\kappa}(t)}\,dW_t.
 \label{eq:fm-sde-expanded}
\end{align}
Both start from $\mathcal N(0,I)$. At $\pfmVelocity{v_\eta}=\pfmVelocity{v}$, the score is exact and Appendix~\ref{app:compatibility} gives the reference marginals. Setting $\pfmDiffusion{\kappa}=0$ returns the ODE. Appendix~\ref{app:updates} supplies the discrete SDE step.

For the endpoint means, taking conditional expectations in $X_1=X_t+(1-t)\pfmVelocity{V}$ gives the FM readout. Replacing the exact score in Eq.~\ref{eq:diffusion-endpoint-mean} by $s_\phi$ gives its diffusion counterpart:
\begin{equation}
 \pfmVelocity{\mu_\eta}(x,t)=x+(1-t)\pfmVelocity{v_\eta}(x,t),\qquad
 \pfmVelocity{\mu_\phi}(x,t)=\frac{x+\sigma_t^2s_\phi(x,t)}{\alpha_t},\quad \alpha_t\ne0.
 \label{eq:baseline-readouts}
\end{equation}
Together with the dynamics above, these readouts give the baseline entries of Table~\ref{tab:generative-dynamics}. PFM uses the same correspondence with the posterior-dependent tensor derived in Appendix~\ref{app:diffusion-design}, and its additional variance readouts follow from Appendix~\ref{app:posterior-map}.

\section{Posterior Learning and Coordinate Transformations}
\label{app:posterior}
\label{app:information}

This appendix develops the learning argument in Section~\ref{sec:learning}, from endpoint information to the moments fitted by the PFM objective. Independent $X_0\sim\mathcal N(0,I)$ and $X_1\sim p_{\mathrm{data}}$ define $X_t=(1-t)X_0+tX_1$ and $\pfmVelocity{V}=X_1-X_0$. At fixed $t<1$, we fit the endpoint posterior $p(X_1\mid X_t)$ with $q_\theta^x$. We first derive the information bound in Eq.~\ref{eq:mi-bound}, then express the likelihood in velocity coordinates and prove Proposition~\ref{prop:moments}. The information identities assume conditional densities, finite displayed entropies and cross-entropies, and a model density positive on the posterior support.

\subsection{Endpoint posterior fitting and information}
\label{app:information-bounds}

At fixed time, the available endpoint information and remaining conditional uncertainty satisfy
\begin{equation}
 I(X_t;X_1)=H(X_1)-H(X_1\mid X_t).
 \label{eq:mi}
\end{equation}
To connect this identity to posterior fitting, write $p(\cdot\mid X_t)$ for the true endpoint posterior, suppressing fixed time. We use $\operatorname{KL}(p\|q)=\mathbb E_p[\log(p/q)]$ for the Kullback--Leibler divergence, or relative entropy. Adding and subtracting its log density in the expected NLL gives
\begin{equation}
 \begin{aligned}
 \mathbb E[-\log q_\theta^x(X_1\mid X_t,t)]
 &=-\mathbb E\log p(X_1\mid X_t)
 +\mathbb E\log\frac{p(X_1\mid X_t)}{q_\theta^x(X_1\mid X_t,t)}\\
 &=H(X_1\mid X_t)
 +\mathbb E_{X_t}\operatorname{KL}\big(p(\cdot\mid X_t)\|q_\theta^x(\cdot\mid X_t,t)\big).
 \end{aligned}
 \label{eq:unconditional-cross-entropy}
\end{equation}
Subtracting this identity from $H(X_1)$ and using the nonnegativity of KL gives~\citep{barber2003variational,poole2019variational}
\begin{equation*}
 I(X_t;X_1)\geq H(X_1)+\mathbb E\log q_\theta^x(X_1\mid X_t,t),
\end{equation*}
the bound in Eq.~\ref{eq:mi-bound}. Its gap is the expected posterior KL in Eq.~\ref{eq:unconditional-cross-entropy}.

Equality in Eq.~\ref{eq:mi-bound} holds when the model matches the true posterior almost everywhere. The reference path fixes the MI, while posterior fitting reduces the approximation gap. The affine relation below expresses the same objective in velocity coordinates.

\subsection{Velocity and endpoint coordinates}
\label{app:posterior-map}
\label{app:entropy-map}

The reference path implies $X_1=X_t+(1-t)\pfmVelocity{V}$, the affine identity in Eq.~\ref{eq:endpoint-map}. At fixed $(x,t)$ with $t<1$, the map $x_1=x+(1-t)v$ is invertible and has determinant $(1-t)^d$. The induced endpoint model therefore has density
\begin{equation}
 q_\theta^x(x_1\mid x,t)
 =\frac{q_\theta((x_1-x)/(1-t)\mid x,t)}{(1-t)^d}.
 \label{eq:posterior-density-map}
\end{equation}
Taking negative logarithms yields
\begin{equation}
 -\log q_\theta^x(X_1\mid X_t,t)
 =-\log q_\theta(\pfmVelocity{V}\mid X_t,t)+d\log(1-t).
 \label{eq:nll-map}
\end{equation}
The last term is independent of $\theta$, which justifies the velocity-coordinate objective in Section~\ref{sec:learning}. The same map also determines the endpoint moments in Eq.~\ref{eq:endpoint-readout}.

For any true conditional law with finite second moments, the same affine map gives
\begin{align}
 \mathbb E[X_1\mid X_t=x]&=x+(1-t)\pfmVelocity{v}(x,t),\nonumber\\
 \operatorname{Cov}(X_1\mid X_t=x)
 &=(1-t)^2\operatorname{Cov}(\pfmVelocity{V}\mid X_t=x).
 \label{eq:true-endpoint-moments}
\end{align}
Subtracting the conditional mean from the affine identity multiplies the centered velocity by $1-t$. Taking its outer product therefore proves the full covariance relation. The model retains the diagonal of this covariance at the population optimum in Proposition~\ref{prop:moments}.

The same Jacobian determines the entropy transformation. Assuming finite differential entropies, taking the expected negative log density of the transformed true posterior gives
\begin{equation}
 H(X_1\mid X_t)=H(\pfmVelocity{V}\mid X_t)+d\log(1-t).
 \label{eq:entropy-map}
\end{equation}
The translation by the observed state leaves conditional entropy unchanged, while the scale $1-t$ contributes its log determinant. This connects the information bound to the velocity-coordinate likelihood.

\subsection{Gaussian model and likelihood}
\label{app:gaussian-likelihood}

The model $q_\theta=\mathcal N(\pfmVelocity{v_\theta},\pfmDiffusion{\Sigma_\theta})$ uses $\pfmDiffusion{\Sigma_\theta}=\operatorname{diag}\pfmDiffusion{\gamma_\theta}$ with positive variances. Applying the affine map to this Gaussian gives
\begin{equation*}
 q_\theta^x=\mathcal N(\pfmVelocity{\mu_\theta},\pfmDiffusion{\Sigma_\theta}^x),\qquad
 \pfmVelocity{\mu_\theta}=x+(1-t)\pfmVelocity{v_\theta},\qquad \pfmDiffusion{\Sigma_\theta}^x=(1-t)^2\pfmDiffusion{\Sigma_\theta}.
\end{equation*}
These are the moments of the model distribution in Eq.~\ref{eq:endpoint-readout}.

The complete velocity negative log-likelihood is
\begin{equation}
 -\log q_\theta(\pfmVelocity{V}\mid X_t,t)
 =\frac12\left[(\pfmVelocity{V}-\pfmVelocity{v_\theta})^\top\pfmDiffusion{\Sigma_\theta}^{-1}(\pfmVelocity{V}-\pfmVelocity{v_\theta})
 +\log\det\pfmDiffusion{\Sigma_\theta}\right]+\frac d2\log(2\pi).
 \label{eq:full-velocity-nll}
\end{equation}
Define $\mathcal L_{\mathrm{PFM}}=\mathbb E[-\log q_\theta(\pfmVelocity{V}\mid X_t,t)]-\tfrac d2\log(2\pi)$, which is the objective in Eq.~\ref{eq:pfm-loss}. Combining Eqs.~\ref{eq:posterior-density-map} and~\ref{eq:full-velocity-nll} and averaging under the fixed training law gives
\begin{equation}
 \mathbb E[-\log q_\theta^x(X_1\mid X_t,t)]
 =\mathcal L_{\mathrm{PFM}}
 +d\,\mathbb E_t\log(1-t)+\frac d2\log(2\pi).
 \label{eq:full-endpoint-nll}
\end{equation}
The averaged identity assumes finite expected negative log-likelihoods and integrability of $\log(1-t)$. At each fixed $t<1$, the pointwise density identity already gives the parameter-independent difference. Since the additional terms in Eq.~\ref{eq:full-endpoint-nll} are independent of $\theta$, minimizing $\mathcal L_{\mathrm{PFM}}$ maximizes the variational lower bound in Eq.~\ref{eq:mi-bound} averaged over the training time law $r$.

\paragraph{Diagonal covariance.}
\label{app:diagonal}

At fixed $(x,t)$,
\begin{equation*}
 \operatorname{Cov}_{q_\theta}(\pfmVelocity{V_i},\pfmVelocity{V_j}\mid x,t)=0\quad(i\ne j).
\end{equation*}
The approximation uses $d$ variance outputs instead of the $d(d+1)/2$ entries of a full symmetric covariance. Each output still depends on the complete state, so $\partial_{x_j}\pfmDiffusion{\gamma_{\theta,i}}$ can be nonzero for $i\ne j$. Conditional covariance of the modeled target and the Jacobian of the predicted variance field are different objects: the former is diagonal here, while the latter can be dense. Appendix~\ref{app:derivatives} computes the diagonal of this input Jacobian for the SDE drift.

\subsection{Proof of conditional moment recovery}
\label{app:moments}

We prove Proposition~\ref{prop:moments} and obtain the residual identity used by the diagnostics in Section~\ref{sec:diagnostics}.
\begin{proof}[Proof of Proposition~\ref{prop:moments}]
Fix a state, time, and coordinate $i$. Let $\pfmVelocity{m_i}$ and $\pfmDiffusion{\sigma_i}^2>0$ denote the true conditional mean and variance. For a candidate mean $\pfmVelocity{u_i}$,
\begin{equation*}
 \mathbb E[(\pfmVelocity{V_i}-\pfmVelocity{u_i})^2\mid X_t=x]
 =\mathbb E[(\pfmVelocity{V_i}-\pfmVelocity{m_i})^2\mid X_t=x]+(\pfmVelocity{m_i}-\pfmVelocity{u_i})^2,
\end{equation*}
since the centered cross term has zero expectation. A candidate positive variance $\pfmDiffusion{\gamma_i}$ therefore gives conditional Gaussian risk
\begin{equation}
 R_i(\pfmVelocity{u_i},\pfmDiffusion{\gamma_i})
 =\frac12\left[\frac{\pfmDiffusion{\sigma_i}^2+(\pfmVelocity{m_i}-\pfmVelocity{u_i})^2}{\pfmDiffusion{\gamma_i}}+\log\pfmDiffusion{\gamma_i}\right].
 \label{eq:conditional-risk}
\end{equation}
For any fixed positive $\pfmDiffusion{\gamma_i}$, the unique mean minimizer is $\pfmVelocity{u_i}=\pfmVelocity{m_i}$. At that mean,
\begin{equation*}
 \partial_{\pfmDiffusion{\gamma_i}}R_i=\frac{\pfmDiffusion{\gamma_i}-\pfmDiffusion{\sigma_i}^2}{2\pfmDiffusion{\gamma_i}^2}.
\end{equation*}
It is negative below $\pfmDiffusion{\sigma_i}^2$ and positive above, giving the unique variance minimum. Summing the coordinate risks proves the result. Only the true conditional first and second moments enter this calculation.
\end{proof}
For a fixed biased mean and a positive residual second moment, minimizing Eq.~\ref{eq:conditional-risk} over variance gives
\begin{equation}
 \pfmDiffusion{\gamma_i}=\mathbb E[(\pfmVelocity{V_i}-\pfmVelocity{u_i})^2\mid X_t=x]
 =\pfmDiffusion{\sigma_i}^2+(\pfmVelocity{m_i}-\pfmVelocity{u_i})^2.
 \label{eq:bias-variance}
\end{equation}
This supplies the residual interpretation in Appendix~\ref{app:diagnostic-calibration}.

The implementations in Appendix~\ref{app:implementation-details} bound each log variance between fixed values $\pfmDiffusion{\ell_-}<\pfmDiffusion{\ell_+}$. For the diagonal covariance model, the same risk gives the optimum over the closed interval $[e^{\pfmDiffusion{\ell_-}},e^{\pfmDiffusion{\ell_+}}]$. For a fixed mean $\pfmVelocity{u_i}$, its variance derivative has the sign of $\pfmDiffusion{\gamma_i}-\mathbb E[(\pfmVelocity{V_i}-\pfmVelocity{u_i})^2\mid X_t=x]$, so
\begin{equation}
 \pfmDiffusion{\gamma_i^\star}(\pfmVelocity{u_i})
 =\operatorname{clip}\!\left(\mathbb E[(\pfmVelocity{V_i}-\pfmVelocity{u_i})^2\mid X_t=x],
 e^{\pfmDiffusion{\ell_-}},e^{\pfmDiffusion{\ell_+}}\right).
 \label{eq:bounded-variance-optimum}
\end{equation}
Here $\operatorname{clip}(r,a,b)=\min\{b,\max\{a,r\}\}$. With a freely optimized mean at each $(x,t)$, the mean optimum remains $\pfmVelocity{u_i^\star}=\pfmVelocity{m_i}$ and Eq.~\ref{eq:bounded-variance-optimum} becomes the clipped conditional variance $\operatorname{clip}(\pfmDiffusion{\sigma_i}^2,e^{\pfmDiffusion{\ell_-}},e^{\pfmDiffusion{\ell_+}})$. Thus the bounds retain conditional moment recovery within the variance interval and set its limiting scales outside it. Appendix~\ref{app:training} gives the smooth parameterization used to enforce these bounds.

Fixing all variances to one in Eq.~\ref{eq:pfm-loss} removes likelihood weighting and the log determinant, leaving
\begin{equation}
 \mathcal L_{\mathrm{PFM}}=\frac12\mathbb E\|\pfmVelocity{V}-\pfmVelocity{v_\theta}(X_t,t)\|^2.
 \label{eq:fm-recovery}
\end{equation}
This is the FM specialization stated in Section~\ref{sec:learning}. Together, the likelihood transformation and moment recovery show how one PFM objective fits the velocity moments used for generation and the endpoint moments used for diagnostics.

\section{Posterior-Driven Dynamics}
\label{app:dynamics}

This appendix derives the PFM SDE in Section~\ref{sec:sampling} and proves its relative-entropy result. Let $p_t$ be a reference density transported by $\pfmVelocity{v}$, with score $s=\nabla_x\log p_t$. PFM predicts $\pfmVelocity{v_\theta}$ and the positive diagonal covariance $\pfmDiffusion{\Sigma_\theta}=\operatorname{diag}\pfmDiffusion{\gamma_\theta}$. We use the covariance to set local diffusion scales, pair the resulting tensor with a compatible drift, and then derive the dissipation identity for two densities evolving under those same coefficients.

\subsection{From posterior normalization to the diffusion tensor}
\label{app:diffusion-design}

The design principle in Eq.~\ref{eq:posterior-normalization} makes the noise isotropic after dividing each noise coordinate by the predicted velocity posterior standard deviation at the current state: $\pfmDiffusion{\Sigma_\theta}^{-1/2}\pfmDiffusion{a_\theta}\pfmDiffusion{\Sigma_\theta}^{-1/2}=\lambda(t)I$ with $\lambda(t)\geq0$. At fixed $(x,t)$, the positive diagonal covariance is invertible. Multiplying both sides by $\pfmDiffusion{\Sigma_\theta}^{1/2}$ yields $\pfmDiffusion{a_\theta}=\lambda\pfmDiffusion{\Sigma_\theta}$. A square root is $\pfmDiffusion{B_\theta}=\sqrt\lambda\operatorname{diag}\sqrt{\pfmDiffusion{\gamma_\theta}}$, for which $\pfmDiffusion{B_\theta} \pfmDiffusion{B_\theta}^\top=\pfmDiffusion{a_\theta}$.

The normalization in Eq.~\ref{eq:posterior-normalization} is a local matrix calculation at the current state. For a Brownian increment $\Delta W\sim\mathcal N(0,hI)$, freezing the coefficients at this state gives
\begin{equation*}
 \operatorname{Cov}(\pfmDiffusion{B_\theta}\Delta W\mid x,t)=h\lambda\pfmDiffusion{\Sigma_\theta},
 \qquad
 \operatorname{Cov}(\pfmDiffusion{\Sigma_\theta}^{-1/2}\pfmDiffusion{B_\theta}\Delta W\mid x,t)=h\lambda I.
\end{equation*}
Thus coordinate $i$ receives noise variance $h\lambda\pfmDiffusion{\gamma_{\theta,i}}$. The posterior controls relative strengths, and $\lambda$ converts velocity covariance into a state-increment covariance rate. This explains the componentwise noise allocation in Section~\ref{sec:sampling}.

\subsection{Compatible drift and its coordinate form}
\label{app:coordinate-drift}

For the independent Gaussian reference path, the mean defines $\widehat s_\theta=(t\pfmVelocity{v_\theta}-x)/(1-t)$, as derived in Appendix~\ref{app:score}. Pairing the chosen diffusion tensor with this score and the mean gives the drift in Eq.~\ref{eq:pfm-sde}:
\begin{equation}
 b_\theta=\pfmVelocity{v_\theta}+\tfrac12\nabla_x\!\cdot \pfmDiffusion{a_\theta}
                   +\tfrac12\pfmDiffusion{a_\theta}\widehat s_\theta.
 \label{eq:pfm-drift-expanded}
\end{equation}
The positive correction follows from solving the shared correspondence for the SDE drift. Since $\lambda$ depends only on time,
\begin{equation*}
 (\nabla_x\!\cdot \pfmDiffusion{a_\theta})_i
 =\sum_j\partial_{x_j}\big[\lambda(t)\pfmDiffusion{\gamma_{\theta,i}}\delta_{ij}\big]
 =\lambda(t)\partial_{x_i}\pfmDiffusion{\gamma_{\theta,i}},
\end{equation*}
where $\delta_{ij}$ is the Kronecker delta. Defining $D_i=\partial_{x_i}\pfmDiffusion{\gamma_{\theta,i}}$ gives
\begin{equation}
 b_{\theta,i}=\pfmVelocity{v_{\theta,i}}+\frac{\lambda(t)}2
 \left[D_i+\pfmDiffusion{\gamma_{\theta,i}}\frac{t\pfmVelocity{v_{\theta,i}}-x_i}{1-t}\right].
 \label{eq:coordinate-drift}
\end{equation}
The three contributions are transport, correction for spatially varying covariance, and correction along the density score. Equations~\ref{eq:pfm-drift-expanded} and~\ref{eq:coordinate-drift} are the vector and coordinate forms of the sampler in Section~\ref{sec:sampling}.

At the population mean $\pfmVelocity{v_\theta}=\pfmVelocity{v}$, Eq.~\ref{eq:score-error} gives $\widehat s_\theta=s$. The conditions in Appendix~\ref{app:compatibility} then establish the reference marginals for the chosen $\pfmDiffusion{a_\theta}$. Posterior fitting determines which diffusion tensor is used within this compatible family. Setting $\lambda=0$ eliminates both corrections and the Brownian term, recovering Eq.~\ref{eq:pfm-ode-score}. Appendix~\ref{app:updates} turns these coefficients into numerical updates.

\subsection{Proof of relative-entropy dissipation}
\label{app:kl-proof}

Proposition~\ref{prop:kl} describes an initial density perturbation under exact compatible dynamics. To prove it, take any diffusion tensor $\pfmDiffusion{a}=\pfmDiffusion{B}\pfmDiffusion{B}^\top$ and fix the drift $b=\pfmVelocity{v}+(\nabla_x\!\cdot\pfmDiffusion{a}+\pfmDiffusion{a}\nabla_x\log p_t)/2$. The reference density $p_t$ and a second density $\widetilde p_t$ evolve under these same coefficients from their respective initial laws. Under the smoothness, finiteness, and vanishing-boundary conditions of the proposition, we first establish the general identity~\citep{fontbona2016trajectorial,wibisono2017information}
\begin{equation}
 \frac{d}{dt}\operatorname{KL}(\widetilde p_t\|p_t)
 =-\frac12\int\widetilde p_t
 \left(\nabla_x\log\frac{\widetilde p_t}{p_t}\right)^\top
 \pfmDiffusion{a}\left(\nabla_x\log\frac{\widetilde p_t}{p_t}\right)dx\leq0.
 \label{eq:kl-dissipation}
\end{equation}
\begin{proof}[Proof of Proposition~\ref{prop:kl}]
With the common coefficients fixed as above, applying Eq.~\ref{eq:diffusion-product-rule} to the diffusion term for $\widetilde p_t$ gives
\begin{equation*}
 \sum_j\partial_{x_j}(\pfmDiffusion{a_{ij}}\widetilde p_t)
 =[(\nabla_x\!\cdot \pfmDiffusion{a})+\pfmDiffusion{a}\nabla_x\log\widetilde p_t]_i\widetilde p_t.
\end{equation*}
Substitution of the common drift into the Fokker--Planck equation therefore yields
\begin{equation}
 \partial_t\widetilde p_t=-\nabla_x\!\cdot(\pfmVelocity{v}\widetilde p_t)
 +\frac12\nabla_x\!\cdot\left[\pfmDiffusion{a}\widetilde p_t
 \nabla_x\log\frac{\widetilde p_t}{p_t}\right].
 \label{eq:perturbed-density}
\end{equation}
The reference satisfies $\partial_t p_t=-\nabla_x\!\cdot(\pfmVelocity{v}p_t)$. Mass conservation removes the integral of $\partial_t\widetilde p_t$ when differentiating KL, leaving
\begin{equation}
 \frac{d}{dt}\operatorname{KL}(\widetilde p_t\|p_t)
 =\int\log\frac{\widetilde p_t}{p_t}\,\partial_t\widetilde p_t\,dx
 -\int\frac{\widetilde p_t}{p_t}\,\partial_t p_t\,dx.
 \label{eq:kl-derivative-expanded}
\end{equation}
All integrals below are over $\mathbb R^d$. Under the regularity and vanishing-boundary assumptions in Proposition~\ref{prop:kl}, integration by parts turns the two transport contributions into
\begin{equation*}
 \int \pfmVelocity{v}^\top\widetilde p_t\nabla_x\log\frac{\widetilde p_t}{p_t}\,dx
 -\int \pfmVelocity{v}^\top p_t\nabla_x\frac{\widetilde p_t}{p_t}\,dx=0,
\end{equation*}
because $\widetilde p_t\nabla_x\log(\widetilde p_t/p_t)=p_t\nabla_x(\widetilde p_t/p_t)$. The remaining term is
\begin{equation*}
 \begin{aligned}
 &\frac12\int\log\frac{\widetilde p_t}{p_t}\,
 \nabla_x\!\cdot\left[\pfmDiffusion{a}\widetilde p_t\nabla_x\log\frac{\widetilde p_t}{p_t}\right]dx\\
 &\qquad=-\frac12\int\widetilde p_t
 \left(\nabla_x\log\frac{\widetilde p_t}{p_t}\right)^\top
 \pfmDiffusion{a}\left(\nabla_x\log\frac{\widetilde p_t}{p_t}\right)dx.
 \end{aligned}
\end{equation*}
This proves Eq.~\ref{eq:kl-dissipation}, with the sign following from positive semidefiniteness. Substituting $\pfmDiffusion{a}=\lambda\operatorname{diag}\pfmDiffusion{\gamma_\theta}$ gives the coordinate sum in Eq.~\ref{eq:kl-coordinate}. Taking $\pfmDiffusion{a}=0$ leaves only the canceling transport terms and proves the ODE statement in Proposition~\ref{prop:kl}.
\end{proof}
For the isotropic FM choice $\pfmDiffusion{a}=\pfmDiffusion{\kappa}I$, all coordinate weights are $\pfmDiffusion{\kappa}(t)$. PFM uses $\lambda(t)\pfmDiffusion{\gamma_{\theta,i}}(x,t)$, so the posterior covariance in Eq.~\ref{eq:posterior-diffusion} determines the coordinatewise weights in relative-entropy dissipation.

When the dissipation is integrable in time, Eq.~\ref{eq:kl-dissipation} gives the cumulative identity
\begin{equation}
 \begin{aligned}
 &\operatorname{KL}(\widetilde p_T\|p_T)\\
 &\quad+\frac12\int_0^T\!\int\widetilde p_t
 \left(\nabla_x\log\frac{\widetilde p_t}{p_t}\right)^\top
 \pfmDiffusion{a}\left(\nabla_x\log\frac{\widetilde p_t}{p_t}\right)dx\,dt
 =\operatorname{KL}(\widetilde p_0\|p_0).
 \end{aligned}
 \label{eq:kl-cumulative}
\end{equation}
\section{Posterior Diagnostics}
\label{app:diagnostics}

This appendix gives the statistical basis for Section~\ref{sec:diagnostics}. Independent noise and data endpoints define $X_t=(1-t)X_0+tX_1$ and $\pfmVelocity{V}=X_1-X_0$. A PFM model predicts velocity mean $\pfmVelocity{v_\theta}$ and covariance $\pfmDiffusion{\Sigma_\theta}$. Their endpoint readouts are $\pfmVelocity{\mu_\theta}=x+(1-t)\pfmVelocity{v_\theta}$ and $\pfmDiffusion{\Sigma_\theta}^x=(1-t)^2\pfmDiffusion{\Sigma_\theta}$. We derive how prediction residuals transform between these coordinates, explain moment checks and Gaussian NLL comparisons, and finally interpret the readouts at states visited by a sampler.

\subsection{Posterior readouts and residual transformations}
\label{app:diagnostic-residuals}

We develop Eq.~\ref{eq:endpoint-readout} for reference pairs. The model distribution in Eq.~\ref{eq:posterior} has exact moments
\begin{equation*}
 \mathbb E_{q_\theta}[\pfmVelocity{V}\mid x,t]=\pfmVelocity{v_\theta}(x,t),\qquad
 \operatorname{Cov}_{q_\theta}(\pfmVelocity{V}\mid x,t)=\pfmDiffusion{\Sigma_\theta}(x,t).
\end{equation*}
Its induced endpoint moments follow from Appendix~\ref{app:posterior-map}. Proposition~\ref{prop:moments} identifies their true population targets.

For evaluation, pair a data endpoint $X_1$ not used in training with independent noise $X_0$, and construct the reference state $X_t$ and target velocity $\pfmVelocity{V}$ by Eq.~\ref{eq:reference-path}. The known endpoint then supplies coordinate residuals
\begin{equation*}
 e_i^v=\pfmVelocity{V_i}-\pfmVelocity{v_{\theta,i}}(X_t,t),\qquad
 e_i^x=X_{1,i}-\pfmVelocity{\mu_{\theta,i}}(X_t,t).
\end{equation*}
Substituting $X_1=X_t+(1-t)\pfmVelocity{V}$ and $\pfmVelocity{\mu_\theta}=X_t+(1-t)\pfmVelocity{v_\theta}$ gives $e_i^x=(1-t)e_i^v$, and hence
\begin{equation}
 (X_{1,i}-\pfmVelocity{\mu_{\theta,i}}(X_t,t))^2
 =(1-t)^2(\pfmVelocity{V_i}-\pfmVelocity{v_{\theta,i}}(X_t,t))^2.
 \label{eq:residual-map}
\end{equation}
This exact relation holds for every reference pair and predicted mean. This relation expresses the same residual comparison in velocity or endpoint coordinates.

\subsection{Moment and likelihood checks}
\label{app:diagnostic-calibration}

We explain how the residuals in Eq.~\ref{eq:residual-map} assess the moments predicted in Section~\ref{sec:diagnostics}. At fixed $(x,t)$, let $\pfmVelocity{m_i}$ and $\pfmDiffusion{\sigma_i}^2$ be the true velocity mean and variance. Equation~\ref{eq:bias-variance} gives
\begin{equation*}
 \mathbb E[e_i^v\mid X_t=x]=\pfmVelocity{m_i}-\pfmVelocity{v_{\theta,i}},\qquad
 \mathbb E[(e_i^v)^2\mid X_t=x]
 =\pfmDiffusion{\sigma_i}^2+(\pfmVelocity{m_i}-\pfmVelocity{v_{\theta,i}})^2.
\end{equation*}
Thus squared residuals measure conditional variance at an accurate mean and also include squared mean bias otherwise.

A practical check groups reference states at matched times by their inputs or predicted variances. If the predicted moments equal their conditional targets, averaging the preceding identities within a group gives zero mean residual and agreement between mean predicted variance and mean squared residual. These comparisons assess the fitted moments without requiring repeated observations at an identical continuous state. Equation~\ref{eq:residual-map} gives the corresponding endpoint comparison by multiplying both variance and squared residual by $(1-t)^2$. This is the basis of the variance--residual plots in Section~\ref{sec:posterior-diagnostics}.

A complementary check compares Gaussian NLLs at a shared predicted mean. Evaluating covariance models on the same reference states keeps the prediction residuals fixed, so the covariance-weighted residual term and log determinant in Eq.~\ref{eq:full-velocity-nll} jointly assess how each covariance fits those residuals. Appendix~\ref{app:experiment-molecule} applies this comparison within the centered coordinate subspace, using the learned molecular covariance and a covariance that depends only on time and is fitted on validation data.

\subsection{Readouts along generated trajectories}
\label{app:diagnostic-reporting}

At a visited state $Y_t=y$, the model approximates the reference endpoint posterior $p(X_1\mid X_t=y)$. Its mean $\pfmVelocity{\mu_\theta}(y,t)=y+(1-t)\pfmVelocity{v_\theta}(y,t)$ predicts the endpoint, and its covariance $\pfmDiffusion{\Sigma_\theta}^x(y,t)=(1-t)^2\pfmDiffusion{\Sigma_\theta}(y,t)$ describes the dispersion of compatible data endpoints. Both readouts are available under ODE and SDE sampling. Variance maps and standard-deviation contours around the predicted means display these readouts throughout generation, as in Section~\ref{sec:posterior-diagnostics}.

The same velocity covariance also determines the SDE diffusion tensor $\pfmDiffusion{a_\theta}=\lambda(t)\pfmDiffusion{\Sigma_\theta}$. Thus the endpoint factor $(1-t)^2$ converts velocity dispersion into endpoint dispersion, while $\lambda(t)$ scales the stochastic perturbations used for sampling. Appendix~\ref{app:diffusion-design} develops this noise construction. The moment and likelihood checks above assess the posterior readouts on reference states with known endpoints.

\section{Conditional Generation}
\label{app:conditional}

This appendix extends the posterior learning, sampling, and diagnostics of Sections~\ref{sec:learning}--\ref{sec:diagnostics} to an external generation input $C$. After defining the conditional reference distribution, we develop posterior learning and endpoint readouts, construct ODE and SDE sampling, and assess the readouts through posterior diagnostics. The condition can specify a class, observation, or reference structure~\citep{ma2025beyond}. It is held fixed within each generation trajectory, while the sampled endpoint pair $Z=(X_0,X_1)$ specifies the reference interpolation.

\subsection{Conditional reference construction}
\label{app:conditional-reference}

Let $C\sim p_C$ and $X_1\mid C=c\sim p_{\mathrm{data}}(\cdot\mid c)$. Draw an independent source $X_0\sim p_0=\mathcal N(0,I)$, so the joint law is
\begin{equation}
 p_C(dc)\,p_0(dx_0)\,p_{\mathrm{data}}(dx_1\mid c).
 \label{eq:conditional-reference-law}
\end{equation}
For each pair, retain the reference interpolation and velocity target~\citep{lipman2023flow,lipman2024flowmatchingguidecode}
\begin{equation}
 X_t=(1-t)X_0+tX_1,\qquad \pfmVelocity{V}=X_1-X_0,\qquad
 p_t(\cdot\mid c)=\operatorname{Law}(X_t\mid C=c).
 \label{eq:conditional-reference-path}
\end{equation}
At fixed $c$, the marginal velocity is
\begin{equation}
 \pfmVelocity{v}(x,t,c)=\mathbb E[\pfmVelocity{V}\mid X_t=x,C=c].
 \label{eq:conditional-velocity}
\end{equation}
Applying the conditional-expectation calculation in Appendix~\ref{app:fm} within each $c$ gives
$\partial_t p_t(\cdot\mid c)+\nabla_x\!\cdot(p_t(\cdot\mid c)\pfmVelocity{v}(\cdot,t,c))=0$.
The source density at $t=0$ is $p_0$ for every condition. The affine identity $X_1=X_t+(1-t)\pfmVelocity{V}$ holds for each reference pair and provides the posterior correspondence used next.

\subsection{Learning conditional velocity posteriors}
\label{app:conditional-learning}

At fixed $t<1$, the information supplied by the reference state after observing $C$ is
\begin{equation}
 I(X_t;X_1\mid C)=H(X_1\mid C)-H(X_1\mid X_t,C).
 \label{eq:cmi}
\end{equation}
Model the remaining endpoint posterior $p(x_1\mid x,t,c)$ by $q_\theta^x(x_1\mid x,t,c)$. Under finite conditional entropies and cross-entropies and a model density positive on the posterior support, the decomposition in Appendix~\ref{app:information-bounds} applies conditionally and gives
\begin{equation}
 \begin{aligned}
 \mathbb E[-\log q_\theta^x(X_1\mid X_t,t,C)]
 &=H(X_1\mid X_t,C)\\
 &\quad+\mathbb E_{X_t,C}\operatorname{KL}\!\left(
 p(\cdot\mid X_t,C)\,\|\,q_\theta^x(\cdot\mid X_t,t,C)\right).
 \end{aligned}
 \label{eq:conditional-cross-entropy-expanded}
\end{equation}
Subtracting this identity from $H(X_1\mid C)$ gives the conditional counterpart of the variational lower bound in Eq.~\ref{eq:mi-bound}~\citep{barber2003variational,poole2019variational,mohamed2015variational}:
\begin{equation}
 I(X_t;X_1\mid C)\geq H(X_1\mid C)
 +\mathbb E\log q_\theta^x(X_1\mid X_t,t,C).
 \label{eq:cmi-bound}
\end{equation}
The reference construction fixes the CMI. Minimizing the conditional posterior NLL reduces the expected posterior KL and tightens this bound, as in Appendix~\ref{app:information-bounds}.

Holding $(x,t,c)$ fixed, the affine change of variables has the same Jacobian as in Appendix~\ref{app:posterior-map}:
\begin{equation}
 q_\theta^x(x_1\mid x,t,c)
 =\frac{q_\theta((x_1-x)/(1-t)\mid x,t,c)}{(1-t)^d}.
 \label{eq:conditional-density-map}
\end{equation}
Consequently the endpoint and velocity NLLs differ by $d\log(1-t)$, and their true conditional entropies obey
$H(X_1\mid X_t,C)=H(\pfmVelocity{V}\mid X_t,C)+d\log(1-t)$.
The conditional velocity model is
\begin{equation}
 \begin{aligned}
 (\pfmVelocity{v_\theta},\pfmDiffusion{\ell_\theta})&=F_\theta(x,t,c),\qquad
 \pfmDiffusion{\gamma_\theta}=\exp\pfmDiffusion{\ell_\theta},\qquad
 \pfmDiffusion{\Sigma_\theta}=\operatorname{diag}\pfmDiffusion{\gamma_\theta},\\
 q_\theta(v\mid x,t,c)&=\mathcal N\!\left(v;\pfmVelocity{v_\theta}(x,t,c),\pfmDiffusion{\Sigma_\theta}(x,t,c)\right).
 \end{aligned}
 \label{eq:conditional-model}
\end{equation}
With $t\sim r$ independent of $(C,X_0,X_1)$, the training objective becomes
\begin{equation}
 \begin{aligned}
 \mathcal L_{\mathrm{PFM}}^C
 &=\frac12\mathbb E\left[
 (\pfmVelocity{V}-\pfmVelocity{v_\theta})^\top\pfmDiffusion{\Sigma_\theta}^{-1}
 (\pfmVelocity{V}-\pfmVelocity{v_\theta})+\log\det\pfmDiffusion{\Sigma_\theta}\right]\\
 &=\frac12\mathbb E\left[\sum_{i=1}^d
 \left(\frac{(\pfmVelocity{V_i}-\pfmVelocity{v_{\theta,i}})^2}{\pfmDiffusion{\gamma_{\theta,i}}}
 +\log\pfmDiffusion{\gamma_{\theta,i}}\right)\right].
 \end{aligned}
 \label{eq:conditional-pfm-loss}
\end{equation}
where all outputs are evaluated at $(X_t,t,C)$. Under the finite-expectation conditions of Appendix~\ref{app:gaussian-likelihood}, the complete endpoint NLL is
\begin{equation}
 \mathbb E[-\log q_\theta^x(X_1\mid X_t,t,C)]
 =\mathcal L_{\mathrm{PFM}}^C+d\,\mathbb E_t\log(1-t)+\frac d2\log(2\pi).
 \label{eq:conditional-full-nll}
\end{equation}

The conditional risk in Appendix~\ref{app:moments} applies at each $(x,t,c)$. With freely varying means and positive diagonal variances and finite, strictly positive true coordinate variances, its optimum satisfies
\begin{equation}
 \pfmVelocity{v_\theta}(x,t,c)=\pfmVelocity{v}(x,t,c),\qquad
 \pfmDiffusion{\gamma_{\theta,i}}(x,t,c)=\operatorname{Var}(\pfmVelocity{V_i}\mid X_t=x,C=c).
 \label{eq:conditional-moments}
\end{equation}
Fixing $\pfmDiffusion{\Sigma_\theta}=I$ reduces Eq.~\ref{eq:conditional-pfm-loss} to $\tfrac12\mathbb E\|\pfmVelocity{V}-\pfmVelocity{v_\theta}(X_t,t,C)\|^2$, the FM objective for conditional generation. Algorithm~\ref{alg:training} implements this formulation by sampling $(C,X_1)$ jointly, drawing independent $X_0,t$, and passing $C$ to both posterior heads in the same NLL update.

The affine map in Eq.~\ref{eq:conditional-density-map} converts these velocity moments into the conditional endpoint-posterior readouts corresponding to Eq.~\ref{eq:endpoint-readout}:
\begin{equation}
 \begin{aligned}
 \pfmVelocity{\mu_\theta}(x,t,c)&=x+(1-t)\pfmVelocity{v_\theta}(x,t,c),\\
 \pfmDiffusion{\Sigma_\theta}^x(x,t,c)&=(1-t)^2\pfmDiffusion{\Sigma_\theta}(x,t,c).
 \end{aligned}
 \label{eq:conditional-readouts}
\end{equation}
These readouts require no additional head and are available under both samplers below.

\subsection{Posterior-driven ODE and SDE sampling}
\label{app:conditional-dynamics}

Fix $C=c$ throughout a generation trajectory. For the independent Gaussian source, Appendix~\ref{app:score} applied to $p_t(\cdot\mid c)$ gives
\begin{equation}
 s(x,t,c)=\nabla_x\log p_t(x\mid c)
 =\frac{t\pfmVelocity{v}(x,t,c)-x}{1-t},\qquad
 \widehat s_\theta(x,t,c)=\frac{t\pfmVelocity{v_\theta}(x,t,c)-x}{1-t}.
 \label{eq:conditional-score}
\end{equation}
The mean defines the ODE, initialized from $Y_0\sim\mathcal N(0,I)$:
\begin{equation}
 dY_t=\pfmVelocity{v_\theta}(Y_t,t,c)\,dt.
 \label{eq:conditional-ode}
\end{equation}
To scale stochastic perturbations by learned uncertainty, apply the posterior normalization principle of Eq.~\ref{eq:posterior-normalization}. At fixed $c$, requiring isotropic noise after the same normalization gives the diffusion tensor and its Brownian coefficient,
\begin{equation}
 \pfmDiffusion{a_\theta}(x,t,c)=\lambda(t)\pfmDiffusion{\Sigma_\theta}(x,t,c),\qquad
 \pfmDiffusion{B_\theta}(x,t,c)=\sqrt{\lambda(t)}\operatorname{diag}\sqrt{\pfmDiffusion{\gamma_\theta}(x,t,c)}.
 \label{eq:conditional-diffusion}
\end{equation}
Appendix~\ref{app:diffusion-design} derives this construction. Adding diffusion requires the corresponding drift correction, so the shared ODE--SDE correspondence determines the full drift and sampler:
\begin{align}
 b_\theta(x,t,c)
 &=\pfmVelocity{v_\theta}(x,t,c)
 +\tfrac12\nabla_x\!\cdot\pfmDiffusion{a_\theta}(x,t,c)
 +\tfrac12\pfmDiffusion{a_\theta}(x,t,c)\widehat s_\theta(x,t,c),\nonumber\\
 dY_t&=b_\theta(Y_t,t,c)\,dt+\pfmDiffusion{B_\theta}(Y_t,t,c)\,dW_t.
 \label{eq:conditional-sde}
\end{align}
All spatial derivatives act on $x$, with $c$ held fixed. At $\lambda\equiv0$ the SDE reduces to Eq.~\ref{eq:conditional-ode}. Algorithms~\ref{alg:ode} and~\ref{alg:sde} evaluate every field and readout at the same input $c$ throughout sampling.

With the exact conditional velocity and score, the compatibility proof in Appendix~\ref{app:compatibility} applies to $p_t(\cdot\mid c)$ under its stated regularity assumptions. The resulting ODE and SDE have these conditional marginals. Likewise, the proof of Proposition~\ref{prop:kl} in Appendix~\ref{app:kl-proof} applies at each fixed $c$: for two densities evolving under the same exact conditional coefficients, the relative-entropy dissipation is weighted by $\pfmDiffusion{a_\theta}(\cdot,t,c)$. This is the conditional counterpart of the dynamics developed in Appendix~\ref{app:dynamics}.

\subsection{Posterior diagnostics}
\label{app:conditional-diagnostics}

At a generated state $Y_t=y$, the endpoint readouts in Eq.~\ref{eq:conditional-readouts} approximate the moments of $p(X_1\mid X_t=y,C=c)$. They describe the data endpoints compatible with that reference state and generation condition, following the interpretation in Section~\ref{sec:diagnostics}.

For a conditional reference pair, construct $X_t$ from independent noise and a data endpoint with condition $C$, as in Appendix~\ref{app:conditional-reference}. Define $e_i^v=\pfmVelocity{V_i}-\pfmVelocity{v_{\theta,i}}(X_t,t,C)$ and $e_i^x=X_{1,i}-\pfmVelocity{\mu_{\theta,i}}(X_t,t,C)$. The identity $e_i^x=(1-t)e_i^v$ holds pointwise. Conditioning the residual identities of Appendix~\ref{app:diagnostic-calibration} on both $X_t$ and $C$ gives the same mean-residual and variance--residual checks at matched times. Evaluation uses data endpoints not seen during training, as in Appendix~\ref{app:diagnostic-residuals}. To assess covariance fit, compare conditional Gaussian NLLs on the same reference inputs $(X_t,t,C)$ with a shared predicted mean, following Appendix~\ref{app:diagnostic-calibration}. Appendix~\ref{app:experiment-molecule} applies this comparison to molecular conformations. This completes the conditional formulation of posterior learning, generation, and diagnostics in Sections~\ref{sec:learning}--\ref{sec:diagnostics}. Appendix~\ref{app:experiments} specifies how each task supplies its generation condition.

\section{Endpoint-Mean Parameterization}
\label{app:endpoint}

This appendix develops direct endpoint-mean prediction as an alternative parameterization of Section~\ref{sec:learning}. On the same reference path, $X_1=X_t+(1-t)\pfmVelocity{V}$ makes endpoint and velocity means interchangeable for $t<1$~\citep{lipman2024flowmatchingguidecode}. The variance head continues to predict velocity log variance. We define this model and its likelihood, identify its conditional moment targets, and construct the fields used by the samplers and diagnostics in Sections~\ref{sec:sampling}--\ref{sec:diagnostics}.

\subsection{Direct endpoint-mean model}
\label{app:endpoint-training}

We develop the alternative to direct velocity prediction introduced in Section~\ref{sec:learning}. An endpoint network $F_\theta^x$ outputs
\begin{equation}
 (\pfmVelocity{\mu_\theta},\pfmDiffusion{\ell_\theta})=F_\theta^x(x,t),\qquad
 \pfmDiffusion{\gamma_\theta}=\exp\pfmDiffusion{\ell_\theta},\qquad
 q_\theta^x=\mathcal N(\pfmVelocity{\mu_\theta},(1-t)^2\operatorname{diag}\pfmDiffusion{\gamma_\theta}).
 \label{eq:direct-endpoint-model}
\end{equation}
Here $\pfmDiffusion{\ell_\theta}$ is the velocity log variance. Inverting Eq.~\ref{eq:endpoint-map} induces
\begin{equation}
 \pfmVelocity{v_\theta}(x,t)=\frac{\pfmVelocity{\mu_\theta}(x,t)-x}{1-t},\qquad
 \pfmVelocity{V}-\pfmVelocity{v_\theta}(X_t,t)=\frac{X_1-\pfmVelocity{\mu_\theta}(X_t,t)}{1-t}.
 \label{eq:endpoint-induced-velocity}
\end{equation}
\subsection{Likelihood and conditional moments}
\label{app:endpoint-likelihood}

Substituting the residual in Eq.~\ref{eq:endpoint-induced-velocity} into Eq.~\ref{eq:pfm-loss} gives
\begin{equation}
 \mathcal L_{\mathrm{PFM}}^x=\frac12\mathbb E\sum_i
 \left[\frac{(X_{1,i}-\pfmVelocity{\mu_{\theta,i}})^2}{(1-t)^2\pfmDiffusion{\gamma_{\theta,i}}}
 +\log\pfmDiffusion{\gamma_{\theta,i}}\right].
 \label{eq:endpoint-loss}
\end{equation}
All outputs are evaluated at $(X_t,t)$. Algorithm~\ref{alg:training} applies with the endpoint network and this integrand. Under the finite-expectation conditions in Appendix~\ref{app:gaussian-likelihood}, the complete objective is
\begin{equation}
 \mathbb E[-\log q_\theta^x(X_1\mid X_t,t)]
 =\mathcal L_{\mathrm{PFM}}^x+d\,\mathbb E_t\log(1-t)+\frac d2\log(2\pi).
 \label{eq:direct-endpoint-full-nll}
\end{equation}
Thus Eqs.~\ref{eq:full-endpoint-nll} and~\ref{eq:direct-endpoint-full-nll} express the same change of variables for the two mean parameterizations. Setting $\pfmDiffusion{\gamma_\theta}=\mathbf1$ in Eq.~\ref{eq:endpoint-loss} yields endpoint-parameterized FM with squared-error weight $(1-t)^{-2}$, corresponding to Eq.~\ref{eq:fm-recovery} in velocity coordinates.

We connect the endpoint network to Proposition~\ref{prop:moments}. At fixed $(x,t)$ with $t<1$, the transformation between candidate endpoint and velocity means is bijective. Under the moment conditions of Proposition~\ref{prop:moments}, the conditional population optimum of Eq.~\ref{eq:endpoint-loss} has
\begin{equation*}
 \pfmVelocity{\mu_\theta}=\mathbb E[X_1\mid X_t=x],\qquad
 \pfmDiffusion{\gamma_{\theta,i}}=\operatorname{Var}(\pfmVelocity{V_i}\mid X_t=x).
\end{equation*}
Equation~\ref{eq:true-endpoint-moments} maps this velocity variance to the diagonal endpoint covariance.

\subsection{Generation dynamics}
\label{app:endpoint-generation}

Substitution of Eq.~\ref{eq:endpoint-induced-velocity} into Eq.~\ref{eq:pfm-ode-score} gives
\begin{equation}
 \widehat s_\theta=\frac{t(\pfmVelocity{\mu_\theta}-x)/(1-t)-x}{1-t}
 =\frac{t\pfmVelocity{\mu_\theta}-x}{(1-t)^2},\qquad
 \pfmDiffusion{a_\theta}=\lambda\operatorname{diag}\pfmDiffusion{\gamma_\theta}
 =\frac{\lambda(t)}{(1-t)^2}\pfmDiffusion{\Sigma_\theta}^x.
 \label{eq:endpoint-sampling-fields}
\end{equation}
With these fields and $\pfmDiffusion{B_\theta}=\sqrt\lambda\operatorname{diag}\sqrt{\pfmDiffusion{\gamma_\theta}}$, the samplers are
\begin{align*}
 dY_t&=\frac{\pfmVelocity{\mu_\theta}(Y_t,t)-Y_t}{1-t}\,dt,\\
 dY_t&=\left[\frac{\pfmVelocity{\mu_\theta}-x}{1-t}
 +\tfrac12\nabla_x\!\cdot \pfmDiffusion{a_\theta}+\tfrac12\pfmDiffusion{a_\theta}\widehat s_\theta\right](Y_t,t)dt
 +\pfmDiffusion{B_\theta}(Y_t,t)dW_t.
\end{align*}
Both initialize from $\mathcal N(0,I)$. Their discrete versions use Algorithms~\ref{alg:ode} and~\ref{alg:sde} after constructing the fields in Eq.~\ref{eq:endpoint-sampling-fields}.

\subsection{Readouts and parameterization correspondence}
\label{app:endpoint-correspondence}

The endpoint mean $\pfmVelocity{\mu_\theta}$ is predicted directly, while $\pfmVelocity{v_\theta}=(\pfmVelocity{\mu_\theta}-x)/(1-t)$ is derived. The variance outputs represent $\pfmDiffusion{\Sigma_\theta}=\operatorname{diag}\exp\pfmDiffusion{\ell_\theta}$ in velocity coordinates and $\pfmDiffusion{\Sigma_\theta}^x=(1-t)^2\pfmDiffusion{\Sigma_\theta}$ in endpoint coordinates. Thus both parameterizations supply the same types of posterior readouts in Section~\ref{sec:diagnostics}.

When the two mean parameterizations satisfy $\pfmVelocity{\mu_\theta}(x,t)=x+(1-t)\pfmVelocity{v_\theta}(x,t)$ and use the same velocity covariance field and schedule $\lambda(t)$, they yield the same velocity, score, diffusion tensor, and drift. The likelihood relation follows from Appendix~\ref{app:gaussian-likelihood}, and the residual relation follows from Appendix~\ref{app:diagnostic-residuals}. These field and likelihood correspondences explain the two mean parameterizations: the main text develops direct velocity prediction, while the ImageNet implementation in Appendix~\ref{app:experiment-image} uses direct endpoint-mean prediction.

For conditional generation, evaluate $F_\theta^x(x,t,c)$ under the reference construction in Appendix~\ref{app:conditional-reference}. Holding $c$ fixed leaves the affine transformation unchanged. The likelihood averages over $(C,X_1,X_0,t)$, and the fields above enter the conditional samplers and diagnostics of Appendices~\ref{app:conditional-dynamics}--\ref{app:conditional-diagnostics}.

\section{Training and Numerical Sampling}
\label{app:numerics}

This appendix turns the objective and continuous dynamics of Sections~\ref{sec:learning}--\ref{sec:sampling} into training and sampling algorithms. The network $F_\theta(x,t)$ outputs a velocity mean $\pfmVelocity{v_\theta}$ and log variance $\pfmDiffusion{\ell_\theta}$, with $\pfmDiffusion{\gamma_\theta}=\exp\pfmDiffusion{\ell_\theta}$ and $\pfmDiffusion{\Sigma_\theta}=\operatorname{diag}\pfmDiffusion{\gamma_\theta}$. We specify the joint training update, compute the variance derivatives required by the SDE drift, and then discretize both samplers. The final two subsections derive the effect of randomized derivatives on a sampling step and explain the terminal time and endpoint output. Multiplication $\odot$, exponentials, and square roots of vectors are elementwise.

Figure~\ref{fig:pfm-method-full} expands Figure~\ref{fig:pfm-method} with the application readouts, connecting the algorithms below to the task implementations in Appendix~\ref{app:experiments}.

\subsection{Training update}
\label{app:training}

We expand the single-loss training procedure in Section~\ref{sec:learning}. Since $\pfmDiffusion{\Sigma_\theta}$ is diagonal, its inverse has entries $\exp(-\pfmDiffusion{\ell_{\theta,i}})$ and its log determinant is $\sum_i\pfmDiffusion{\ell_{\theta,i}}$. Equation~\ref{eq:pfm-loss} is therefore
\begin{equation}
 \mathcal L_{\mathrm{PFM}}=\frac12\mathbb E\sum_i
 \left[\exp(-\pfmDiffusion{\ell_{\theta,i}})(\pfmVelocity{V_i}-\pfmVelocity{v_{\theta,i}})^2+\pfmDiffusion{\ell_{\theta,i}}\right].
 \label{eq:coordinate-loss}
\end{equation}
The variance head produces an unconstrained value $\rho_{\theta,i}$, which is mapped to the log variance used in this loss by
\begin{equation}
 \pfmDiffusion{\ell_{\theta,i}}
 =\pfmDiffusion{\ell_-}+(\pfmDiffusion{\ell_+}-\pfmDiffusion{\ell_-})\operatorname{sigmoid}(\rho_{\theta,i}),
 \qquad \pfmDiffusion{\gamma_{\theta,i}}=\exp\pfmDiffusion{\ell_{\theta,i}}.
 \label{eq:bounded-log-variance}
\end{equation}
For finite $\rho_{\theta,i}$, this gives an open variance interval. Interior optima in Eq.~\ref{eq:bounded-variance-optimum} are attained, while boundary optima are approached as $\rho_{\theta,i}\to\pm\infty$. The same bounded variance supplies the likelihood, diffusion tensor, and endpoint readouts. Table~\ref{tab:implementation-details} lists the task-specific bounds. Algorithm~\ref{alg:training} averages the loss over a minibatch and differentiates through this mapping and the complete shared model.

\noindent\begin{minipage}{\linewidth}
\begin{pfmalgorithm}[Joint posterior training]
\label{alg:training}
\mbox{}\par
\textbf{Input:} network $F_\theta$, data law, parameter-independent time law $r$, and an optimizer.\par
\textbf{Repeat} until the specified training budget is reached:
\begin{enumerate}
 \item Draw a minibatch of endpoints $X_1\sim p_{\mathrm{data}}$. Independently draw $X_0\sim\mathcal N(0,I)$ and $t\sim r$ for each item.
 \item Form $X_t=(1-t)X_0+tX_1$ and $\pfmVelocity{V}=X_1-X_0$.
 \item Evaluate $(\pfmVelocity{v_\theta},\pfmDiffusion{\ell_\theta})=F_\theta(X_t,t)$ once and average the integrand of Eq.~\ref{eq:coordinate-loss} over the minibatch.
 \item Backpropagate this loss through both heads and the shared backbone~\citep{rumelhart1986learning}, and update $\theta$ with the optimizer.
\end{enumerate}
\textbf{Output:} trained parameters $\theta$.
\end{pfmalgorithm}
\end{minipage}\par\medskip
The resulting model supplies the mean, variance, and derived score used by the samplers below. Appendix~\ref{app:conditional-learning} gives its conditional-generation extension.

\subsection{Variance derivatives for the SDE drift}
\label{app:derivatives}

For a trained model, the ODE uses $\pfmVelocity{v_\theta}$ directly. The SDE also uses $\widehat s_\theta=(t\pfmVelocity{v_\theta}-x)/(1-t)$ and $\pfmDiffusion{a_\theta}=\lambda(t)\operatorname{diag}\pfmDiffusion{\gamma_\theta}$ for a chosen nonnegative schedule $\lambda$. The compatible drift derived in Appendix~\ref{app:coordinate-drift} is
\begin{equation*}
 b_\theta=\pfmVelocity{v_\theta}+\frac{\lambda(t)}2
 \left(D+\pfmDiffusion{\gamma_\theta}\odot\widehat s_\theta\right),
 \qquad D_i=\partial_{x_i}\pfmDiffusion{\gamma_{\theta,i}}.
\end{equation*}
To compute the required derivatives at fixed time, define
\begin{equation}
 (J_\gamma)_{ij}=\partial_{x_j}\pfmDiffusion{\gamma_{\theta,i}},\qquad
 D=\operatorname{diag}J_\gamma,\qquad
 D_i=\pfmDiffusion{\gamma_{\theta,i}}\partial_{x_i}\pfmDiffusion{\ell_{\theta,i}}.
 \label{eq:variance-jacobian}
\end{equation}
Rows index variance outputs and columns index input coordinates. Here $\operatorname{diag}$ extracts a matrix diagonal. Applied to a vector, as in Eq.~\ref{eq:heads}, it instead constructs a diagonal matrix. Appendix~\ref{app:diagonal} explains why $J_\gamma$ can be dense.

\paragraph{Exact derivatives.}
A Jacobian--vector product (JVP) computes $J_\gamma \zeta$ for a direction $\zeta$. A vector--Jacobian product (VJP), written as a column vector, computes $J_\gamma^\top \zeta$. For the standard basis vector $e_i$,
\begin{equation}
 D_i=(J_\gamma e_i)_i=(J_\gamma^\top e_i)_i.
 \label{eq:exact-variance-diagonal}
\end{equation}
Using the $d$ basis directions gives the exact diagonal without storing the full Jacobian.

\paragraph{Randomized derivatives.}
Let $\zeta^{(1)},\ldots,\zeta^{(m)}$ be independent Rademacher vectors, whose coordinates independently take $+1$ and $-1$ with equal probability. The JVP estimator~\citep{han2026uncertainty} is
\begin{equation}
 \widehat D=\frac1m\sum_{k=1}^m \zeta^{(k)}\odot J_\gamma \zeta^{(k)}.
 \label{eq:derivative-estimator}
\end{equation}
The estimator is unbiased, with coordinate variance
\begin{equation}
 \mathbb E[\widehat D\mid x,t]=D,\qquad
 \operatorname{Var}(\widehat D_i\mid x,t)
 =\frac1m\sum_{j\ne i}(J_\gamma)_{ij}^2.
 \label{eq:derivative-moments}
\end{equation}
Replacing $J_\gamma \zeta$ in Eq.~\ref{eq:derivative-estimator} by $J_\gamma^\top \zeta$ gives the VJP estimator. It remains unbiased, with coordinate variance $m^{-1}\sum_{j\ne i}(J_\gamma)_{ji}^2$. Appendix~\ref{app:derivative-step-error} derives these moments and their effect on the sampling update.

Model parameters remain fixed during sampling, while differentiation with respect to the current state retains the full variance-field dependence. The SDE therefore requires these derivative evaluations in addition to network forwards.

\subsection{Numerical sampling}
\label{app:updates}

We discretize the samplers in Eqs.~\ref{eq:pfm-ode-score} and~\ref{eq:pfm-sde}. Let $0=t_0<\cdots<t_K=T$ and $h_k=t_{k+1}-t_k$. Initialize $y_0\sim\mathcal N(0,I)$. Euler ODE sampling uses~\citep{lipman2024flowmatchingguidecode}
\begin{equation}
 y_{k+1}=y_k+h_k\pfmVelocity{v_\theta}(y_k,t_k).
 \label{eq:euler}
\end{equation}
This update evaluates the mean at the beginning of each interval. Appendix~\ref{app:terminal} explains the terminal-time convention.

\noindent\begin{minipage}{\linewidth}
\begin{pfmalgorithm}[PFM ODE sampling]
\label{alg:ode}
\mbox{}\par
\textbf{Input:} trained $F_\theta$, a grid ending at $T\leq1$.\par
Draw $y_0\sim\mathcal N(0,I)$.\par
\textbf{For} $k=0,\ldots,K-1$:
\begin{enumerate}
 \item Set $h_k=t_{k+1}-t_k$ and evaluate $\pfmVelocity{v_\theta}(y_k,t_k)$.
 \item If posterior diagnostics are requested, also record $\pfmDiffusion{\Sigma_\theta}$, $\pfmVelocity{\mu_\theta}$, and $\pfmDiffusion{\Sigma_\theta}^x$ using Eq.~\ref{eq:endpoint-readout} at this state.
 \item Apply Eq.~\ref{eq:euler}.
\end{enumerate}
\textbf{Output:} $y_K$, with any recorded posterior readouts.
\end{pfmalgorithm}
\end{minipage}\par\medskip

For SDE sampling, let $\xi_k\sim\mathcal N(0,I)$ be independent across steps and independent of initialization. The Brownian increment is represented by $\sqrt{h_k}\xi_k$. Euler--Maruyama gives~\citep{song2021score,albergo2025stochastic}
\begin{equation}
 y_{k+1}=y_k+h_kb_\theta(y_k,t_k)
 +\sqrt{h_k\lambda(t_k)\pfmDiffusion{\gamma_\theta}(y_k,t_k)}\odot\xi_k.
 \label{eq:em}
\end{equation}
Its Brownian step covariance is $h_k \pfmDiffusion{a_\theta}(y_k,t_k)$. Algorithm~\ref{alg:sde} computes the drift from the same posterior outputs, using exact or randomized derivatives as detailed in Appendix~\ref{app:derivatives}.

\noindent\begin{minipage}{\linewidth}
\begin{pfmalgorithm}[PFM SDE sampling]
\label{alg:sde}
\mbox{}\par
\textbf{Input:} trained $F_\theta$, schedule $\lambda$, derivative procedure, and a grid ending at $T<1$.\par
Draw $y_0\sim\mathcal N(0,I)$.\par
\textbf{For} $k=0,\ldots,K-1$:
\begin{enumerate}
 \item Set $h_k=t_{k+1}-t_k$ and evaluate $(\pfmVelocity{v_\theta},\pfmDiffusion{\ell_\theta})=F_\theta(y_k,t_k)$.
 \item Compute $\pfmDiffusion{\gamma_\theta}=\exp\pfmDiffusion{\ell_\theta}$ and $\widehat s_\theta=(t_k\pfmVelocity{v_\theta}-y_k)/(1-t_k)$.
 \item Compute $D_i=\partial_{y_i}\pfmDiffusion{\gamma_{\theta,i}}$, or its estimate $\widehat D_i$ from fresh probes. Form $b_\theta=\pfmVelocity{v_\theta}+\lambda(t_k)(D+\pfmDiffusion{\gamma_\theta}\odot\widehat s_\theta)/2$, substituting $\widehat D$ when used.
 \item Draw an independent $\xi_k\sim\mathcal N(0,I)$ and apply Eq.~\ref{eq:em}. Posterior readouts at the evaluated state follow from Eq.~\ref{eq:endpoint-readout}.
\end{enumerate}
\textbf{Output:} $y_K$, with any recorded posterior readouts. At steps with $\lambda(t_k)=0$, use Eq.~\ref{eq:euler} directly.
\end{pfmalgorithm}
\end{minipage}\par\medskip
Setting $\lambda\equiv0$ removes the Brownian term and both drift corrections, so Algorithm~\ref{alg:sde} reduces to Algorithm~\ref{alg:ode} without variance derivatives. Both implement the same trained model from Section~\ref{sec:learning}.

\paragraph{FM sampling.}

For the FM rows of Table~\ref{tab:generative-dynamics}, apply Euler--Maruyama to Eq.~\ref{eq:fm-sde-expanded} on the same grid:
\begin{equation}
 \begin{aligned}
 y_{k+1}={}&y_k+h_k\left[\pfmVelocity{v_\eta}(y_k,t_k)
 +\tfrac12\pfmDiffusion{\kappa}(t_k)\widehat s_\eta(y_k,t_k)\right]\\
 &+\sqrt{h_k\pfmDiffusion{\kappa}(t_k)}\xi_k.
 \end{aligned}
 \label{eq:fm-em}
\end{equation}
Each step uses the velocity and its derived score from Appendix~\ref{app:score}, with independent $\xi_k\sim\mathcal N(0,I)$. The prescribed diffusion is isotropic and has no spatial derivative term. Choosing $\pfmDiffusion{\kappa}=0$ gives Euler sampling of the FM ODE, completing the discrete baseline construction from Appendix~\ref{app:baseline-sampling}.

\subsection{Effect of derivative estimation}
\label{app:derivative-step-error}

We first derive the estimator moments in Eq.~\ref{eq:derivative-moments}, then use them to quantify the additional randomness in Algorithm~\ref{alg:sde}. For one probe at fixed $(x,t)$, the error in coordinate $i$ expands as
\begin{equation*}
 \zeta_i(J_\gamma \zeta)_i-D_i=\sum_{j\ne i}(J_\gamma)_{ij}\zeta_i \zeta_j.
\end{equation*}
Each summand has zero mean. Since $\zeta_i^2=1$ and $\mathbb E[\zeta_j\zeta_l]=0$ for distinct $j,l$, the second moment is
\begin{equation*}
 \mathbb E\left[\left(\sum_{j\ne i}(J_\gamma)_{ij}\zeta_i \zeta_j\right)^2\right]
 =\sum_{j\ne i}(J_\gamma)_{ij}^2.
\end{equation*}
Averaging independent probes divides this variance by $m$, giving Eq.~\ref{eq:derivative-moments}. For VJP probes, the same calculation exchanges rows and columns.

Let $\mathcal F_k$ contain initialization and all randomness preceding step $k$, so $y_k$ is measurable with respect to $\mathcal F_k$. Draw the current probes independently of this history and the current Brownian vector $\xi_k$. Replacing $D$ by $\widehat D$ in Eq.~\ref{eq:em} adds
\begin{equation}
 \varepsilon_k=\frac{h_k\lambda(t_k)}2\big[\widehat D(y_k,t_k)-D(y_k,t_k)\big].
 \label{eq:derivative-update-error}
\end{equation}
By Eq.~\ref{eq:derivative-moments}, the JVP version satisfies
\begin{equation}
 \mathbb E[\varepsilon_k\mid\mathcal F_k]=0,\qquad
 \operatorname{Var}(\varepsilon_{k,i}\mid\mathcal F_k)
 =\frac{h_k^2\lambda(t_k)^2}{4m}
 \sum_{j\ne i}(J_\gamma(y_k,t_k))_{ij}^2.
 \label{eq:derivative-step-moments}
\end{equation}
For VJP probes the sum instead uses $(J_\gamma)_{ji}^2$. This identifies the additional step randomness introduced by the derivative procedure in Algorithm~\ref{alg:sde}, separately from its Brownian increment.

\subsection{Terminal time and endpoint output}
\label{app:terminal}

This subsection specifies the terminal convention for the samplers in Section~\ref{sec:sampling} and Algorithms~\ref{alg:ode}--\ref{alg:sde}. Euler ODE sampling may end at $T\leq1$, while all coefficient evaluations occur at $t_k<1$. The SDE construction evaluates the score containing $(1-t)^{-1}$ and uses $T<1$. Both algorithms return $y_K$ at their chosen stopping time $T$. At $T<1$, Eq.~\ref{eq:endpoint-readout} also supplies the endpoint-mean readout $\pfmVelocity{\mu_\theta}(y_K,T)$.

For the continuous-time endpoint statement, suppose a single compatible process is defined on $[0,1)$ and satisfies Appendix~\ref{app:compatibility} on every $[0,T]$ with $T<1$. Its law is then $p_t$ for every $t<1$. The reference interpolation satisfies
\begin{equation*}
 X_t-X_1=(1-t)(X_0-X_1)\longrightarrow0
 \quad\text{almost surely as }t\uparrow1.
\end{equation*}
For every bounded continuous test function $\psi$, dominated convergence gives
\begin{equation*}
 \int\psi\,dp_t=\mathbb E[\psi(X_t)]
 \longrightarrow\mathbb E[\psi(X_1)]=\int\psi\,dp_{\mathrm{data}}.
\end{equation*}
This proves weak convergence of probability measures,
\begin{equation}
 \operatorname{Law}(Y_t)=p_t\Rightarrow p_{\mathrm{data}}
 \quad\text{as }t\uparrow1.
 \label{eq:terminal-law-limit}
\end{equation}

\clearpage
\section{Supplementary Experiments}
\label{app:experiments}

This appendix specifies how the formulation in Sections~\ref{sec:learning}--\ref{sec:diagnostics} is instantiated for the four tasks in Section~\ref{sec:experiments}. The task sections explain the state and conditioning choices, extend the image and action examples and point-cloud moment checks, and detail the molecular covariance comparison. Appendix~\ref{app:trajectory-dynamics} connects one-dimensional sampling trajectories to the evolution of density perturbations. Appendix~\ref{app:sampling-analysis} then examines the covariance choice, stochasticity scale, and sampling cost introduced in Section~\ref{sec:ablation-studies}. Diagnostic intervals describe held-out-data bootstrap resampling, and displayed trajectory times $t$ are rounded to one decimal.

\subsection{Implementation details}
\label{app:implementation-details}
Table~\ref{tab:implementation-details} summarizes the model and training settings. Its log-variance bounds specify $(\pfmDiffusion{\ell_-},\pfmDiffusion{\ell_+})$ in Eq.~\ref{eq:bounded-log-variance}, whose effect on diagonal Gaussian moment recovery is derived in Appendix~\ref{app:moments}. All four models jointly learn mean and variance through the PFM NLL, with the task-specific parameterizations described below; the ShapeNet latent model also uses a KL regularizer. Conditions follow Appendix~\ref{app:conditional} and remain implicit in the readout notation.
\begin{table}[!htbp]
\centering
\caption{Core implementation settings for the four generation tasks in Section~\ref{sec:experiments}. Parameter counts include both PFM output heads and trainable conditioning networks. Batch sizes are effective global batches. Maximum learning rates are reached after warmup, where used. The bounds apply to velocity log-variance outputs, including the pre-projection atomwise scales for molecules.}
\label{tab:implementation-details}
\scriptsize
\setlength{\tabcolsep}{3pt}
\renewcommand{\arraystretch}{1.2}
\begin{tabular}{@{}>{\raggedright\arraybackslash}p{0.14\textwidth}>{\raggedright\arraybackslash}p{0.195\textwidth}>{\raggedright\arraybackslash}p{0.21\textwidth}>{\raggedright\arraybackslash}p{0.195\textwidth}>{\raggedright\arraybackslash}p{0.20\textwidth}@{}}
\toprule
Setting & ImageNet-256 & ShapeNet & GEOM-QM9 & PushT \\
\midrule
Generated state & $256\times256$ RGB image & 2,048 3D points & Centered atom coordinates & 16 two-dimensional actions \\
Condition & Class label & Shape latent & Molecular graph and atom/bond features & Two image/agent-position observations \\
Backbone & JiT-B/16~\citep{li2025backtobasics} & Pointwise network with PointNet encoder and latent prior~\citep{luo2021diffusion} & SO(3)-equivariant TorchMD Transformer~\citep{tholke2022torchmdnet,hassan2024etflow} & Image encoder and conditional 1D U-Net~\citep{chi2025diffusion} \\
PFM parameters (M) & 132.90 & 3.88 & 9.09 & 262.71 \\
Model dimensions & 12 blocks; width 768; 12 heads; patch 16 & Hidden widths 128, 256, 512, 256, 128 & 20 layers; width 160; 8 heads & U-Net channels 512, 1,024, 2,048 \\
Optimizer & AdamW~\citep{loshchilov2019decoupled} & Adam~\citep{kingma2015adam} & AdamW & AdamW \\
Maximum learning rate & $2\times10^{-4}$ & $2\times10^{-3}$ & $7\times10^{-4}$ & $10^{-4}$ \\
Warmup & 5 epochs & None & 2 epochs per 200-epoch stage & 500 updates \\
Learning-rate decay & Constant after warmup & Linear to $10^{-4}$ over updates 200k--400k, then constant & Cosine to $10^{-8}$ in each 200-epoch stage & Cosine to zero \\
Effective batch & 1,024 & 128 & 512 & 64 \\
Training length & 600 epochs & 800k updates & 400 epochs & 3,050 epochs \\
Log-variance bounds & $\pfmDiffusion{[-1,1]}$ & $\pfmDiffusion{[-3,3]}$ & $\pfmDiffusion{[-4,4]}$ & $\pfmDiffusion{[-3,3]}$ \\
\bottomrule
\end{tabular}
\end{table}

PFM SDE sampling uses the linear schedule $\lambda(t)=\lambda_0(1-t)$. The isotropic SDE ablation in Appendix~\ref{app:covariance-ablation} uses the same temporal form, $\pfmDiffusion{\kappa}(t)=\pfmDiffusion{\kappa_0}(1-t)$. For each covariance-ablation comparison, we set $\pfmDiffusion{\kappa_0}=\lambda_0$, so the two SDEs share the same scalar noise schedule.

\subsection{Image generation}
\label{app:experiment-image}
For ImageNet~\citep{deng2009imagenet,russakovsky2015imagenet}, the class label supplies $C$, and the JiT-B/16 model~\citep{li2025backtobasics,dosovitskiy2021imageworth16x16words} generates images in pixel space. Its mean head predicts the endpoint directly using the parameterization in Appendix~\ref{app:endpoint}. Its second head predicts velocity log-variance, so $\pfmVelocity{v_\theta}=(\pfmVelocity{\mu_\theta}-x)/(1-t)$ and $\pfmDiffusion{\Sigma_\theta}^x=(1-t)^2\pfmDiffusion{\Sigma_\theta}$ give the same generation fields and diagnostics as the main formulation.

CIFAR-10~\citep{Krizhevsky2009LearningML} extends the image-generation comparison in Section~\ref{sec:generation-domains}, with a lower FID for the evaluated SDE (Table~\ref{tab:cifar10-table}). Figure~\ref{fig:exp-cifar10} shows generated images together with spatial variance maps, following the main ImageNet presentation. Image metrics use 50,000 class-balanced samples per setting, with 250 integration steps for ImageNet and 50 for CIFAR-10. The CIFAR-10 SDE uses $\lambda_0=0.5$. Images show endpoint-mean readouts near the end of sampling; their variance maps come from an earlier state on the same trajectory and share a contrast transform.
\begin{figure}[!htbp]
\centering
\begin{minipage}[t]{0.7\linewidth}
\vspace{0pt}
\centering
\includegraphics[width=\linewidth]{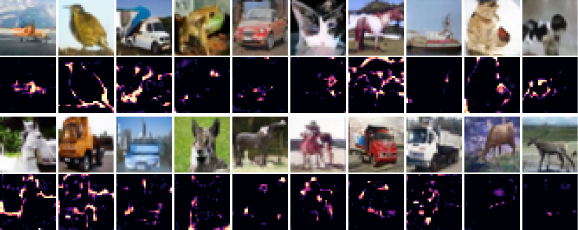}
\par\vspace{-1em}
\caption{CIFAR-10 generation with posterior readouts. Twenty endpoint means are arranged in two rows of ten, each paired with its variance map directly below, as in Figure~\ref{fig:exp-image}.}
\label{fig:exp-cifar10}
\vspace{-1em}
\end{minipage}%
\begin{minipage}[t]{0.3\linewidth}
\vspace{0pt}
\centering
\begin{minipage}[t]{0.92\linewidth}
\makeatletter
\def\@captype{table}
\makeatother
\centering
% Keep the compact caption clear of the table rule.
\setlength{\belowcaptionskip}{1em}
\par\vspace{-1em}
\caption{CIFAR-10 FID.}
\label{tab:cifar10-table}
\vspace{-1em}
\scriptsize
\setlength{\tabcolsep}{4pt}
\begin{tabular}{@{}lr@{}}
\toprule
Method & FID $\downarrow$ \\
\midrule
FM & 3.75 \\
Diffusion & 4.12 \\
PFM ODE & 3.49 \\
PFM SDE & 3.03 \\
\bottomrule
\end{tabular}
\par\vspace{-1em}

\end{minipage}
\end{minipage}
\end{figure}

\subsection{Point-cloud generation}
\label{app:experiment-pointcloud}
Each category uses a pointwise mean/variance network conditioned on a shape latent $C$~\citep{luo2021diffusion}. The PointNet encoder~\citep{qi2017pointnet}, latent prior, and PFM network are jointly trained with the PFM NLL and latent KL regularization of weight $10^{-3}$. The KL term aligns the encoder distribution with the prior used to supply generation conditions~\citep{kingma2013auto}. The NLL is averaged over shapes, points, and coordinates, whereas the encoder-to-prior KL uses the full latent vector and is averaged over shapes. For a given encoder, the joint distribution of shape latents and data defines the conditional reference law in Appendix~\ref{app:conditional-reference}. Generation samples $C$ from the learned latent prior, shared by the ODE and SDE samplers of that category. Reference diagnostics condition on encodings of held-out shapes.

The generation visualizations in Figure~\ref{fig:exp-pointcloud} track the same points over time with a fixed camera and scale. The variance--residual correspondence persists across Chair and Airplane at six reference times, with more visible departures in low-variance and sparsely populated bins (Figure~\ref{fig:exp-pointcloud_full_diagnostics}). This extends the selected-time comparisons in Section~\ref{sec:posterior-diagnostics}.
\begin{figure}[tbp]
\centering
\includegraphics[width=\linewidth]{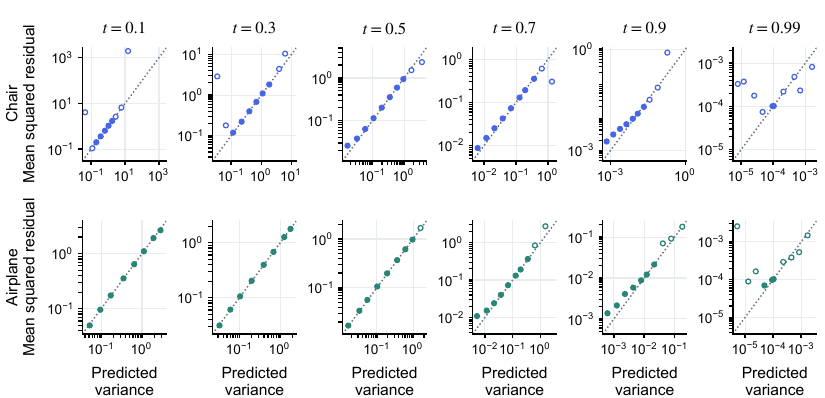}
\par\vspace{-1em}
\caption{Point-cloud posterior diagnostics over six reference times: Chair (top) and Airplane (bottom). Axes compare bin-averaged endpoint coordinate variance and mean squared residual. Dotted lines indicate equality. Hollow markers indicate sparse bins.}
\label{fig:exp-pointcloud_full_diagnostics}
\vspace{-1em}
\end{figure}

\subsection{Molecular conformation generation}
\label{app:experiment-molecule}
The molecular graph and atom/bond features supply $C$; generation predicts atom coordinates with their centroid fixed at zero. For $N$ atoms, let $\mathcal H$ be this $3(N-1)$-dimensional subspace and $P$ its orthogonal projector. The equivariant network~\citep{tholke2022torchmdnet,hassan2024etflow} predicts a centered velocity and one positive, rotation-invariant variance scale $\pfmDiffusion{\gamma_{\theta,i}^{\mathrm{atom}}}$ per atom. Writing the pre-projection covariance as $\pfmDiffusion{G_\theta}$ gives
\begin{equation}
 \begin{aligned}
 P&=\left(I_N-\frac{\mathbf1\mathbf1^\top}{N}\right)\otimes I_3,
 &\pfmDiffusion{G_\theta}&=\operatorname{diag}_{i=1}^N\!\left(\pfmDiffusion{\gamma_{\theta,i}^{\mathrm{atom}}}I_3\right),\\
 \pfmDiffusion{\Sigma_\theta}&=P\pfmDiffusion{G_\theta}P,
 &\pfmDiffusion{\Sigma_\theta}^x&=(1-t)^2\pfmDiffusion{\Sigma_\theta}.
 \end{aligned}
 \label{eq:molecular-covariance}
\end{equation}
This projected covariance replaces the ambient diagonal covariance in Eq.~\ref{eq:posterior}. Because centering couples atom coordinates, the projection introduces correlations between atoms. The reference noise is a standard Gaussian projected by $P$, and the NLL uses the covariance inverse and determinant restricted to $\mathcal H$, normalized by $3(N-1)$. The same posterior-driven construction gives
\begin{equation}
 \begin{aligned}
 \pfmDiffusion{a_\theta}&=\lambda\pfmDiffusion{\Sigma_\theta},\qquad
 \pfmDiffusion{B_\theta}=\sqrt{\lambda}\,P\pfmDiffusion{G_\theta}^{1/2},\\
 b_\theta&=\pfmVelocity{v_\theta}+\tfrac12\left(\nabla_{\mathcal H}\!\cdot\pfmDiffusion{a_\theta}+\pfmDiffusion{a_\theta}\widehat s_\theta\right),\qquad
 \widehat s_\theta=\frac{t\,\pfmVelocity{v_\theta}-x}{1-t}\in\mathcal H.
 \end{aligned}
 \label{eq:molecular-dynamics}
\end{equation}
Here $\nabla_{\mathcal H}$ denotes differentiation within the centered subspace. Both drift and noise stay in $\mathcal H$; numerical updates are recentered. Equations~\ref{eq:molecular-covariance}--\ref{eq:molecular-dynamics} apply the posterior readouts and compatible drift of Sections~\ref{sec:sampling}--\ref{sec:diagnostics} to the molecular geometry.

To evaluate the divergence in Eq.~\ref{eq:molecular-dynamics}, let $\pfmDiffusion{\gamma_\theta^{\mathrm{pre}}}=\operatorname{diag}\pfmDiffusion{G_\theta}\in\mathbb R^{3N}$ collect the pre-projection scales, repeating each atom's scale over its three coordinates. Let $J_{\mathrm{pre}}$ be its input Jacobian at fixed time and molecular condition. Since $P$ is constant,
\begin{equation*}
 \nabla_{\mathcal H}\!\cdot\pfmDiffusion{a_\theta}
 =\lambda(t)P\operatorname{diag}(J_{\mathrm{pre}}P).
\end{equation*}
The diagonal-extraction procedures in Appendix~\ref{app:derivatives} apply to $J_{\mathrm{pre}}P$, whose JVP and VJP are $J_{\mathrm{pre}}(P\zeta)$ and $P J_{\mathrm{pre}}^\top\zeta$, respectively. Multiplying the exact diagonal or its randomized estimate by $\lambda(t)P$ gives the divergence used in the drift.

GEOM-QM9 coverage uses a $0.5\,\text{\AA}$ matching threshold. Performance metrics use chirality correction without force-field relaxation~\citep{hassan2024etflow}. The visualizations and posterior checks use model readouts before this correction. For molecular sampling, $\lambda_0=0.3$ is selected by precision matching distance in the evaluated sweep.

To assess the learned covariance at a shared mean, Figure~\ref{fig:exp-posterior_diagnostics}(d) compares it with a time-only baseline that is isotropic on $\mathcal H$. At each reference time, the baseline's scalar velocity variance is fitted by averaging the squared velocity residual norm divided by $3(N-1)$ over validation molecules, with equal weight per molecule. Both models use PFM's predicted mean. Gaussian NLL is evaluated in velocity coordinates, omits common constants, and is normalized by the subspace dimension $3(N-1)$. The learned covariance gives lower NLL at every evaluated time, supporting a better conditional Gaussian fit than a single variance scale per time. This comparison complements the variance--residual checks in Section~\ref{sec:posterior-diagnostics} by evaluating covariance allocation while holding the mean fixed.

\subsection{Action generation on PushT}
\label{app:experiment-pusht}
The policy conditions on two observations, each containing a $96\times96$ image and the two-dimensional agent position; the image encoder uses $84\times84$ crops. Within each planning call, this condition $C$ is fixed while the conditional 1D U-Net~\citep{chi2025diffusion} generates a 16-step sequence of two-dimensional action targets. The controller executes eight steps and replans from new observations~\citep{chi2025diffusion}. This instantiates Appendix~\ref{app:conditional} within the control loop shown in Section~\ref{sec:generation-domains}. Control scores average the episode maximum normalized overlap over 50 test initial states; the table reports ODE and SDE scores at 20 and 50 integration steps. Appendix~\ref{app:stochasticity-ablation} describes the noise-scale comparison at these two budgets. Figures~\ref{fig:exp-pusht_appendix_06} and~\ref{fig:exp-pusht_appendix_35} complement the two main examples with four additional episodes. Ellipses show endpoint action standard deviations in workspace coordinates, using the same scale as the actions.
\begin{figure}[!htbp]
\centering
\includegraphics[width=\linewidth]{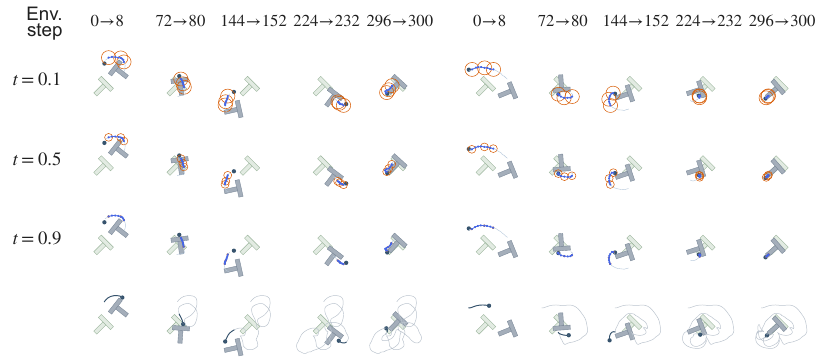}
\par\vspace{-1em}
\caption{PushT posterior readouts for two additional episodes, following Figure~\ref{fig:exp-pusht}. Columns indicate the environment steps from planning to execution, and $t$ denotes generation time. Upper rows show endpoint action means $\pfmVelocity{\mu_\theta}$ (blue) and standard-deviation ellipses from $\pfmDiffusion{\Sigma_\theta}^x$ (orange). The bottom row shows execution.}
\label{fig:exp-pusht_appendix_06}
\vspace{-1em}
\end{figure}

\begin{figure}[!htbp]
\centering
\includegraphics[width=\linewidth]{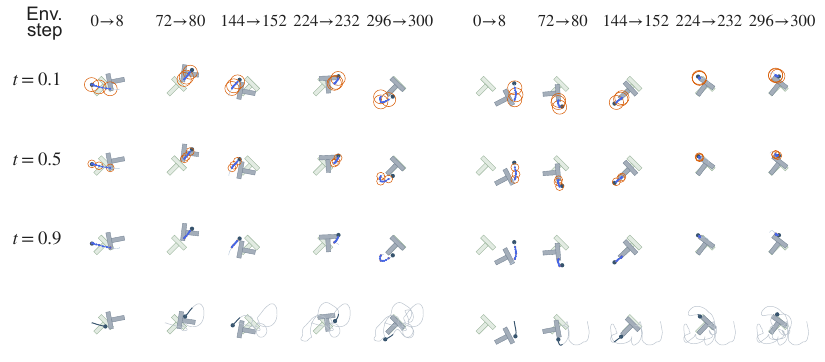}
\par\vspace{-1em}
\caption{PushT posterior readouts for two additional episodes, following Figure~\ref{fig:exp-pusht}. Columns indicate the environment steps from planning to execution, and $t$ denotes generation time. Upper rows show endpoint action means $\pfmVelocity{\mu_\theta}$ (blue) and standard-deviation ellipses from $\pfmDiffusion{\Sigma_\theta}^x$ (orange). The bottom row shows execution.}
\label{fig:exp-pusht_appendix_35}
\vspace{-1em}
\end{figure}

\clearpage
\subsection{One-dimensional generation dynamics}
\label{app:trajectory-dynamics}

This section develops three views of the sampling construction in Section~\ref{sec:sampling}: transport along individual trajectories, adjustment of stochasticity within one posterior model, and the evolution of density perturbations. A one-dimensional Gaussian mixture provides the common target, with component weights $(0.25,0.45,0.30)$, means $(-2.5,-0.4,2.0)$, and standard deviations $(0.22,0.40,0.28)$. Sampling starts from $\mathcal N(0,1)$. PFM uses the linear reference path of Eq.~\ref{eq:reference-path}, while diffusion uses a variance-preserving noising process with $f(x,\tau)=-6x$ and $\pfmDiffusion{g}=\sqrt{12}$. Both networks have three hidden layers of width 128 with SiLU activations and are trained for 60,000 updates. PFM uses Eq.~\ref{eq:pfm-loss} with velocity log-variance bounded in $\pfmDiffusion{[-6,4]}$, and diffusion uses denoising score matching as in Appendix~\ref{app:diffusion-training}.

Figure~\ref{fig:trajectory-dynamics} compares diffusion and PFM under ODE and SDE sampling. Backgrounds show the respective analytic reference densities, while trajectories follow the learned fields. For PFM, the posterior mean supplies deterministic transport. The SDE pairs posterior-scaled noise with the drift correction in Eq.~\ref{eq:pfm-sde}, so both terms change as stochasticity is introduced. The right panels show all four samplers resolving the three target modes. Sampling uses Euler or Euler--Maruyama with 2,000 uniform steps to $T=0.999$, including the variance derivative in the full PFM drift. Generated densities use 20,000 samples.
% Source: pfm-trajectory-visualization/runs/update_060k/main/figures/
% comparison_four_dynamics.pdf (exact update-60k EMA, training seed 42).
\begin{figure}[p]
\centering
\includegraphics[width=\linewidth]{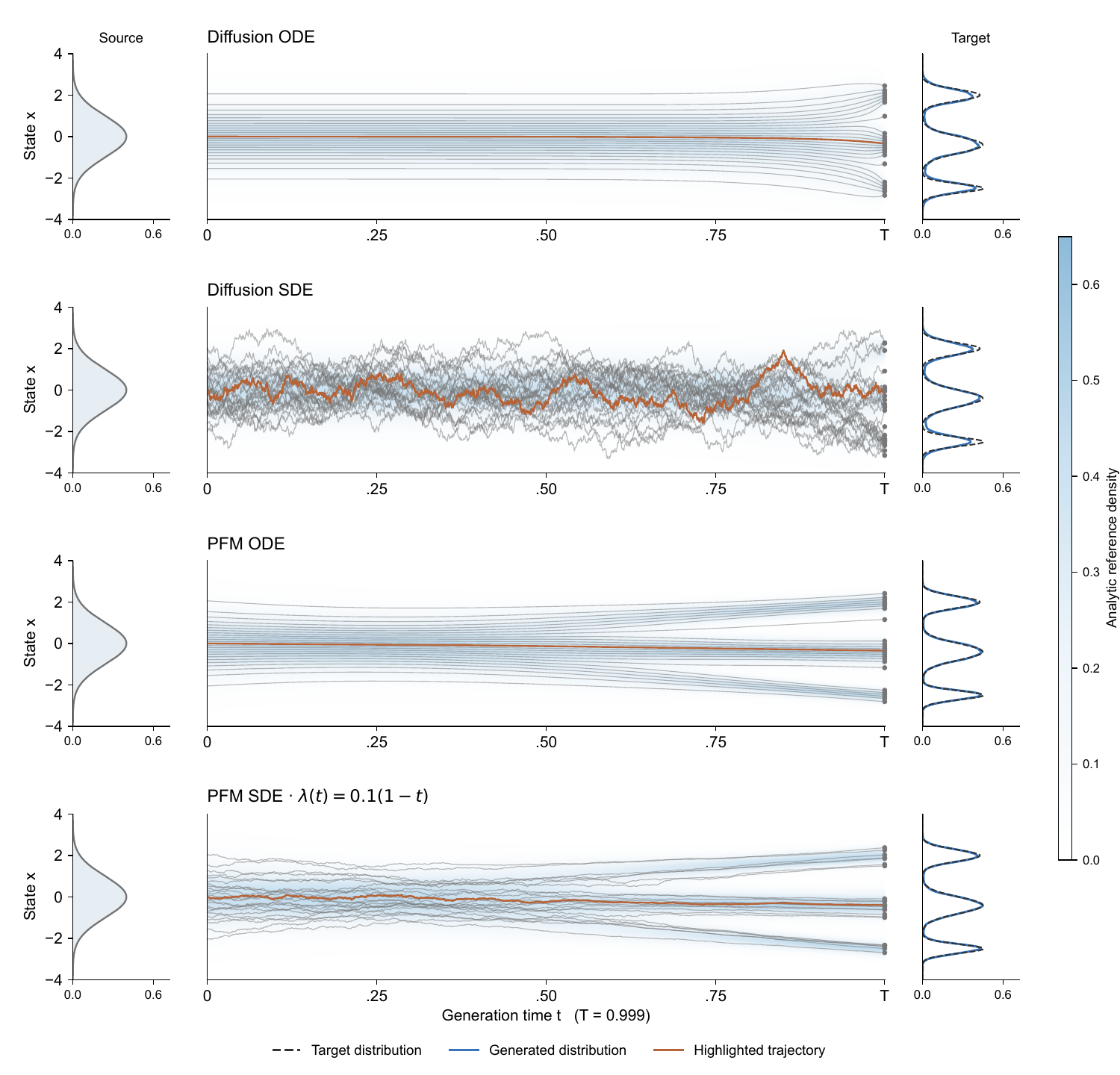}
\par\vspace{-1em}
\caption{Learned diffusion and PFM dynamics on a one-dimensional Gaussian mixture. Gray curves are sampled trajectories, with one highlighted in orange. Backgrounds show each model family's analytic reference density. Right panels compare generated densities (blue) with the target (dashed). The PFM SDE uses $\lambda(t)=0.1(1-t)$. All rows use 2,000 steps to $T=0.999$.}
\label{fig:trajectory-dynamics}
\vspace{-1em}
\end{figure}

To separate the learned noise allocation from its overall strength, Figure~\ref{fig:trajectory-stochasticity} holds the PFM checkpoint fixed and varies $\lambda_0$ in $\lambda(t)=\lambda_0(1-t)$. In Eq.~\ref{eq:posterior-diffusion}, $\pfmDiffusion{\Sigma_\theta}(x,t)$ supplies the local covariance and $\lambda(t)$ scales it. All settings share initial particles and paired Brownian increments on an 8,000-step grid. The six scales connect deterministic transport at $\lambda_0=0$ to increasingly stochastic motion under the same learned covariance field, evaluated along each sampler's trajectory. This gives a trajectory view of the adjustment studied in Section~\ref{sec:ablation-studies}.
% Source: pfm-trajectory-visualization/runs/update_060k/lambda_sweep/figures/
% comparison_pfm_lambda_sweep.pdf (same exact update-60k PFM EMA, seed 42).
\begin{figure}[p]
\centering
\includegraphics[width=\linewidth]{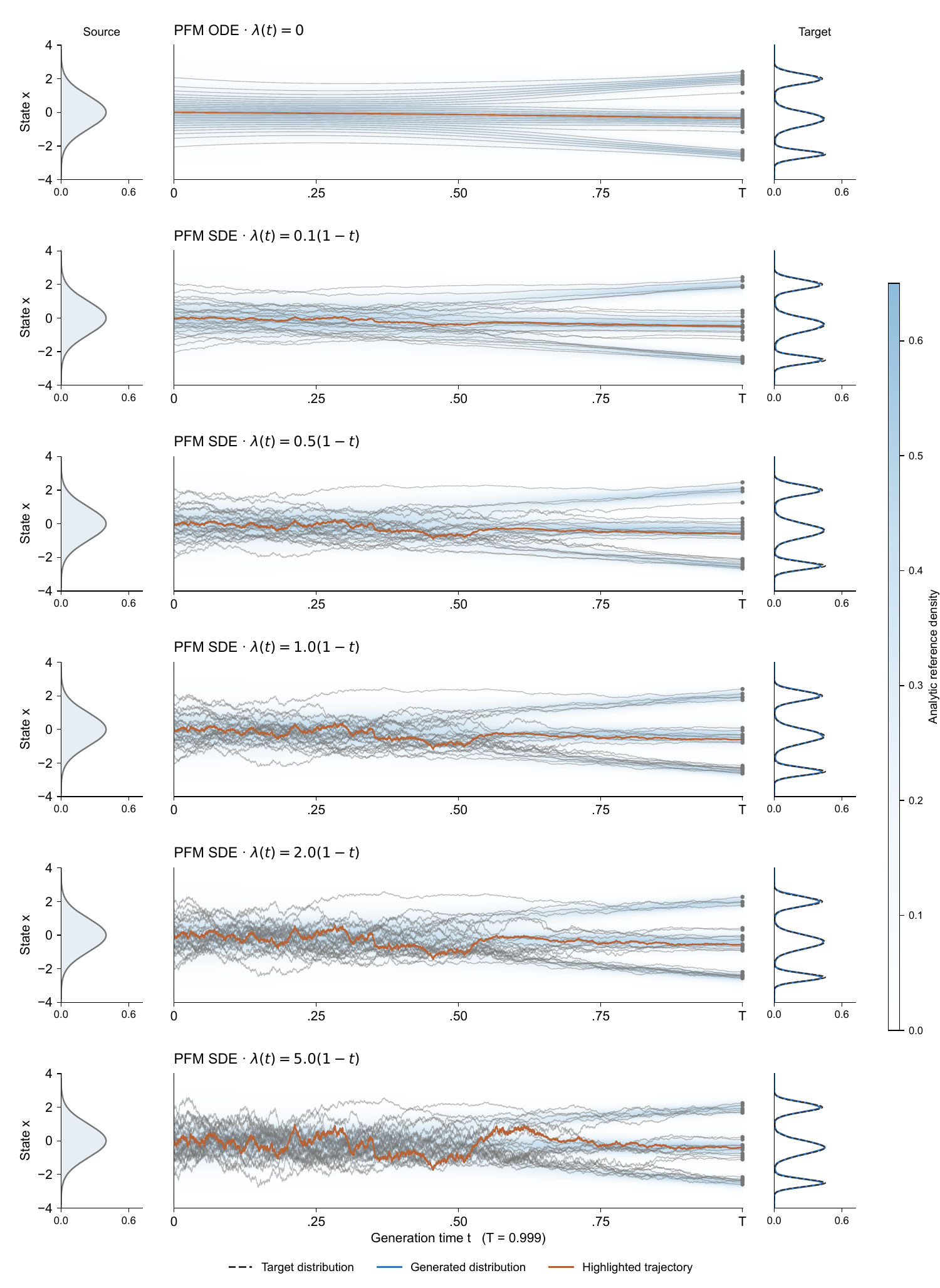}
\par\vspace{-1em}
\caption{PFM trajectories for $\lambda(t)=\lambda_0(1-t)$, using one checkpoint, shared initial particles, and paired Brownian increments. Rows use $\lambda_0\in\{0,0.1,0.5,1,2,5\}$ with 8,000 steps to $T=0.999$. The zero-noise row is the PFM ODE. Colors and density curves follow Figure~\ref{fig:trajectory-dynamics}.}
\label{fig:trajectory-stochasticity}
\vspace{-1em}
\end{figure}

\clearpage

The trajectory views lead to a distributional question: how does an initial density perturbation evolve under the compatible dynamics? To connect this question to Proposition~\ref{prop:kl}, we now use the mixture's analytic conditional velocity moments and score on the same PFM reference path, with $\lambda(t)=1-t$. For each sampler, $p_t$ and $\widetilde p_t$ share these coefficients and start from $p_0=\mathcal N(0,1)$ and $\widetilde p_0=\mathcal N(0.5,1)$, giving initial KL $0.125$. The shifted initialization makes the density discrepancy visible while keeping the dynamics fixed. We evaluate the SDE density evolution through Eq.~\ref{eq:perturbed-density} up to $T=0.999$ and use KL conservation for the ODE.

Figure~\ref{fig:trajectory-kl-dissipation}(a,b) visualizes these dynamics from the same initial samples drawn from $\widetilde p_0$, against the reference density $p_t$. The ODE transports each sample deterministically, while the PFM SDE adds fluctuations scaled by the posterior variance. Panel (c) quantifies the resulting distributional difference. The ODE preserves KL at $0.125$ even as the distributions change shape. Under the PFM SDE, KL decreases to $0.02128$, showing contraction of the initial density perturbation. Panel (d) connects this decrease to the integrated posterior-weighted dissipation in Eq.~\ref{eq:kl-cumulative}: both reach approximately $0.10373$ at $T$. Thus the trajectory and density views link stochastic motion along individual paths to the relative-entropy dissipation established in Proposition~\ref{prop:kl}.
% Sources: pfm-trajectory-visualization/runs/analytic_kl_20260921/ and its transport_review/ subdirectory.
\begin{figure}[!htbp]
\centering
\includegraphics[width=\linewidth]{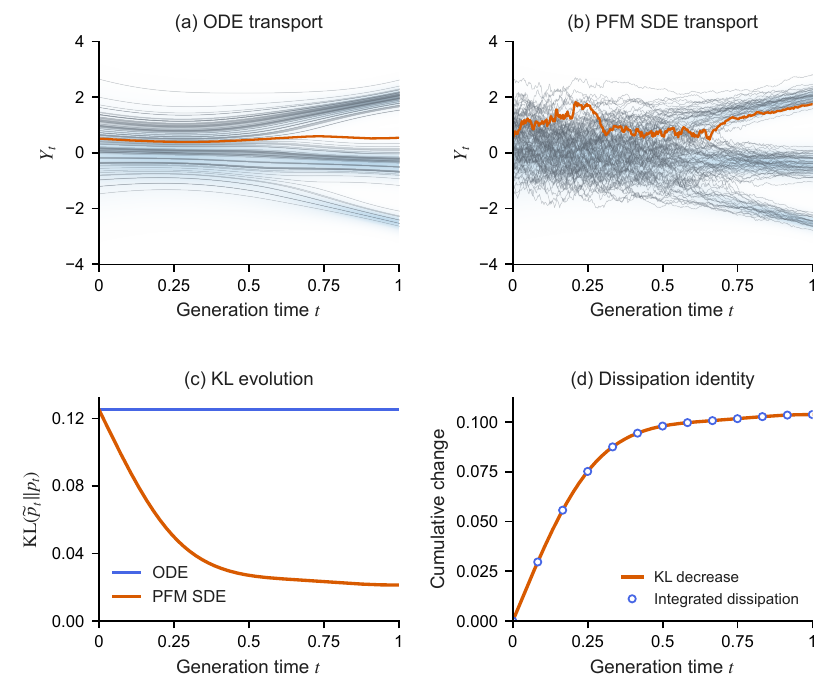}
\par\vspace{-1em}
\caption{Density perturbation under analytic PFM dynamics on the Gaussian-mixture reference path of Figures~\ref{fig:trajectory-dynamics}--\ref{fig:trajectory-stochasticity}, with $p_0=\mathcal N(0,1)$ and $\widetilde p_0=\mathcal N(0.5,1)$. (a,b) ODE and PFM SDE trajectories from shared samples of $\widetilde p_0$, with $\lambda(t)=1-t$ for the SDE. Backgrounds show $p_t$, and orange highlights one trajectory. (c) Density-based KL evolution. (d) SDE KL decrease and integrated dissipation from Eq.~\ref{eq:kl-cumulative}.}
\label{fig:trajectory-kl-dissipation}
\vspace{-1em}
\end{figure}

\clearpage
\subsection{Ablation Studies and Sampling Cost}
\label{app:sampling-analysis}

This section expands Section~\ref{sec:ablation-studies} around the two factors in $\pfmDiffusion{a_\theta}=\lambda(t)\pfmDiffusion{\Sigma_\theta}$. The covariance field determines local noise allocation, and the schedule controls its overall strength. We first compare covariance choices at a shared mean model, then vary stochasticity within the posterior-driven sampler, and finally compare sampling times.

\subsubsection{Role of posterior covariance}
\label{app:covariance-ablation}

PFM ODE and PFM SDE use the same learned velocity $\pfmVelocity{v_\theta}$, with the SDE additionally using $\pfmDiffusion{\Sigma_\theta}$ to allocate noise. An isotropic alternative allows us to examine this allocation while keeping the mean model fixed. We replace the posterior-driven diffusion tensor by $\pfmDiffusion{a_{\mathrm{iso}}}(t)=\pfmDiffusion{\kappa}(t)I$ with the shared scalar schedule $\pfmDiffusion{\kappa}(t)=\lambda(t)$ specified in Appendix~\ref{app:implementation-details}. Using the same derived score $\widehat s_\theta$ from Eq.~\ref{eq:pfm-ode-score}, the correspondence in Eq.~\ref{eq:correspondence} gives
\begin{equation}
 dY_t=\left[\pfmVelocity{v_\theta}(Y_t,t)
 +\frac{\pfmDiffusion{\kappa}(t)}{2}\widehat s_\theta(Y_t,t)\right]dt
 +\sqrt{\pfmDiffusion{\kappa}(t)}\,dW_t.
 \label{eq:pfm-isotropic-ablation}
\end{equation}
The spatial divergence of this time-only tensor is zero. For molecular conformations, isotropy is defined within the centered subspace $\mathcal H$ of Appendix~\ref{app:experiment-molecule}: the ambient tensor is $\pfmDiffusion{\kappa}(t)P$ and the noise is $\sqrt{\pfmDiffusion{\kappa}(t)}P\,dW_t$. All three samplers use the same trained PFM model, with conditions fixed as in Appendix~\ref{app:conditional}. The isotropic variant thus changes how the PFM model is sampled, rather than using the separately trained diffusion model.

Table~\ref{tab:posterior-covariance-ablation} organizes this comparison using the task metrics from Section~\ref{sec:generation-domains}. PFM ODE provides the deterministic reference, while the two SDEs compare prescribed isotropic noise with posterior-dependent noise and their corresponding drift corrections. The PFM ODE and PFM SDE entries reproduce the main generation results. Relative to isotropic sampling, PFM SDE reduces ImageNet FID from $24.31$ to $22.28$ and raises the PushT score from $0.882$ to $0.914$ at 50 steps. The isotropic alternative gives smaller gains on ImageNet, Chair, and PushT, but lowers molecular coverage and worsens Airplane MMD and 1-NNA relative to PFM ODE. With the trained PFM model fixed, posterior-driven SDE sampling improves generation quality over isotropic SDE sampling across all four tasks.
\begin{table}[!htbp]
\centering
\caption{Covariance choice for sampling a shared PFM model. PFM ODE and PFM SDE entries reproduce Tables~\ref{tab:imagenet-table}, \ref{tab:pointcloud-table}, \ref{tab:molecule-table}, and~\ref{tab:pusht-table}. PFM + isotropic SDE uses Eq.~\ref{eq:pfm-isotropic-ablation}. Metrics follow the corresponding main tables. Best values are bold.}
\label{tab:posterior-covariance-ablation}
\small
\setlength{\tabcolsep}{4pt}
\begin{minipage}[t]{0.48\linewidth}
\centering
\textbf{ImageNet-256}\par\smallskip
\begin{tabular}{@{}lrr@{}}
\toprule
Sampler & FID $\downarrow$ & IS $\uparrow$ \\
\midrule
PFM ODE & 24.45 & 65.44 \\
PFM + isotropic SDE & 24.31 & 65.53 \\
PFM SDE & \textbf{22.28} & \textbf{68.06} \\
\bottomrule
\end{tabular}
\end{minipage}\hfill
\begin{minipage}[t]{0.48\linewidth}
\centering
\textbf{PushT}\par\smallskip
\begin{tabular}{@{}lrr@{}}
\toprule
Sampler & 20 steps $\uparrow$ & 50 steps $\uparrow$ \\
\midrule
PFM ODE & 0.827 & 0.878 \\
PFM + isotropic SDE & 0.833 & 0.882 \\
PFM SDE & \textbf{0.848} & \textbf{0.914} \\
\bottomrule
\end{tabular}
\end{minipage}

\medskip
\textbf{ShapeNet (CD). Ideal 1-NNA: $50\%$.}\par\smallskip
{\scriptsize
\setlength{\tabcolsep}{2.5pt}
\begin{tabular*}{\linewidth}{@{\extracolsep{\fill}}lrrrrrr@{}}
\toprule
& \multicolumn{3}{c}{Chair} & \multicolumn{3}{c}{Airplane} \\
Sampler & MMD$\times10^3$ $\downarrow$ & COV (\%) $\uparrow$ & 1-NNA (\%) & MMD$\times10^3$ $\downarrow$ & COV (\%) $\uparrow$ & 1-NNA (\%) \\
\midrule
PFM ODE & 11.57 & 50.98 & 59.95 & 3.23 & 50.13 & 62.70 \\
PFM + isotropic SDE & 11.46 & 51.34 & 59.68 & 3.32 & 50.32 & 62.78 \\
PFM SDE & \textbf{10.50} & \textbf{51.58} & \textbf{58.23} & \textbf{3.06} & \textbf{51.80} & \textbf{60.23} \\
\bottomrule
\end{tabular*}\par}

\medskip
\textbf{GEOM-QM9}\par\smallskip
\begin{tabular}{@{}lrrrr@{}}
\toprule
Sampler & COV-R (\%) $\uparrow$ & COV-P (\%) $\uparrow$ & MAT-R (\AA{}) $\downarrow$ & MAT-P (\AA{}) $\downarrow$ \\
\midrule
PFM ODE & 96.44 & 93.99 & 0.075 & 0.108 \\
PFM + isotropic SDE & 96.14 & 93.58 & 0.078 & 0.123 \\
PFM SDE & \textbf{98.37} & \textbf{96.42} & \textbf{0.063} & \textbf{0.081} \\
\bottomrule
\end{tabular}
\end{table}

\subsubsection{Inference-time stochasticity}
\label{app:stochasticity-ablation}

The stochasticity adjustment in Section~\ref{sec:ablation-studies} varies $\lambda_0$ in the schedule of Appendix~\ref{app:implementation-details} while keeping the learned posterior covariance field fixed. Since $\lambda$ depends only on time, the drift correction in Eq.~\ref{eq:pfm-sde} scales linearly with $\lambda_0$, while the Brownian coefficient scales with $\sqrt{\lambda_0}$. The adjustment therefore changes the full compatible dynamics, with covariance retaining its role in local noise allocation.

The quality trends in Figure~\ref{fig:exp-lambda} distinguish this adjustable strength from the learned covariance: increasing the scale does not give a uniform improvement, and the two PushT integration budgets favor different ranges. Thus the scale provides a sampling choice within one trained model whose effect depends on the task and integration budget. Appendix~\ref{app:trajectory-dynamics} gives the corresponding trajectory view: its six settings show how changing this scale alters individual paths under the same learned covariance field.

\subsubsection{Sampling time}
\label{app:sampling-time}

Table~\ref{tab:sampling-time} expands the cost comparison in Section~\ref{sec:ablation-studies} across the four generation tasks. FM ODE and Diffusion SDE pair each reference model with its named sampler, while PFM ODE and PFM SDE use the same trained posterior model. Relative to PFM ODE, PFM SDE takes $1.16\times$ as long on ImageNet-256, $1.31\times$ on ShapeNet, $1.33\times$ on GEOM-QM9, and $1.17\times$ on PushT. PFM ODE and FM ODE times differ by at most $0.05$ ms per sample. Together with the generation results in Section~\ref{sec:generation-domains}, these measurements describe the quality and time tradeoff between the two PFM samplers.

A PFM forward evaluation jointly returns the mean and variance. PFM ODE uses the mean for transport, while the full PFM SDE also evaluates the variance-field derivatives required by its drift correction. Appendix~\ref{app:derivatives} describes their exact and randomized computation.
\begin{table}[!htbp]
\centering
\caption{Sampling time in milliseconds per generated sample on NVIDIA B200. Columns distinguish each model and its sampler. PushT timings use 50 integration steps.}
\label{tab:sampling-time}
\small
\setlength{\tabcolsep}{8pt}
% Split task rows by category or integration budget when measurements require it.
\begin{tabular*}{\linewidth}{@{\extracolsep{\fill}}lrrrr@{}}
\toprule
Task & FM ODE & Diffusion SDE & PFM ODE & PFM SDE \\
\midrule
ImageNet-256 & 3.43 & 3.78 & 3.42 & 3.98 \\
ShapeNet & 1.87 & 2.11 & 1.88 & 2.47 \\
GEOM-QM9 & 2.18 & 2.42 & 2.17 & 2.89 \\
PushT & 4.25 & 4.49 & 4.30 & 5.04 \\
\bottomrule
\end{tabular*}
\end{table}

\end{document}